\documentclass{article}

\usepackage[preprint]{neurips_2021}

\usepackage[utf8]{inputenc} 
\usepackage[T1]{fontenc}    
\usepackage{hyperref}       
\usepackage{url}            
\usepackage{booktabs}       
\usepackage{amsfonts}       
\usepackage{nicefrac}       
\usepackage{microtype}      

\usepackage{amsmath}
\usepackage{amsfonts}
\usepackage{graphicx}
\graphicspath{ {./images/} }
\usepackage{xcolor}
\newcommand{\tp}[1]{\textcolor{red}{#1}} 
\usepackage{wrapfig}

\newcommand{\x}{\mathbf{x}}
\newcommand{\z}{\mathbf{z}}
\newcommand{\y}{\mathbf{y}}
\newcommand{\w}{\mathbf{w}}
\usepackage{amssymb}
\usepackage{mathtools}
\usepackage{listings}

\usepackage{bbm} 

\usepackage{amsthm}
\usepackage[utf8]{inputenc}
\newtheorem{theorem}{Theorem}
\newtheorem{proposition}{Proposition}
\newtheorem{definition}{Definition}
\newtheorem{lemma}{Lemma}[theorem]
\newtheorem{corollary}{Corollary}[theorem]
\newtheorem*{remark}{Remark}

\usepackage{enumitem}

\definecolor{tpgreen}{RGB}{1,254,1}

\definecolor{codegreen}{rgb}{0,0.6,0}
\definecolor{codegray}{rgb}{0.5,0.5,0.5}
\definecolor{codepurple}{rgb}{0.58,0,0.82}
\definecolor{backcolour}{rgb}{0.96,0.96,0.96}
\lstdefinestyle{mystyle}{
    backgroundcolor=\color{backcolour},   
    commentstyle=\color{codegreen},
    keywordstyle=\color{magenta},
    numberstyle=\tiny\color{codegray},
    stringstyle=\color{codepurple},
    basicstyle=\ttfamily\scriptsize,
    breakatwhitespace=false,         
    breaklines=true,                 
    captionpos=t,                    
    keepspaces=true,                 
    numbers=none,                    
    numbersep=5pt,                  
    showspaces=false,                
    showstringspaces=false,
    showtabs=false,                  
    tabsize=2
}
\title{Understanding Softmax Confidence and Uncertainty}

\author{%
  Tim Pearce \\ 
  TSAIL\\
  Tsinghua University
  \And
  Alexandra Brintrup \\
  MAG \\
  University of Cambridge
  \And
  Jun Zhu \\
  TSAIL\\
  Tsinghua University
}

\begin{document}

\maketitle

\begin{abstract}
%

It is often remarked that neural networks fail to increase their uncertainty when predicting on data far from the training distribution. Yet naively using softmax confidence as a proxy for uncertainty achieves modest success in tasks exclusively testing for this, e.g., out-of-distribution (OOD) detection. This paper investigates this contradiction, identifying two implicit biases that do encourage softmax confidence to correlate with epistemic uncertainty: 1) Approximately optimal decision boundary structure, and 2) Filtering effects of deep networks. It describes why low-dimensional intuitions about softmax confidence are misleading. Diagnostic experiments quantify reasons softmax confidence can fail, finding that extrapolations are less to blame than overlap between training and OOD data in final-layer representations. Pre-trained/fine-tuned networks reduce this overlap.




\end{abstract}

\section{Introduction}
\label{sec_intro}

Many papers studying uncertainty and robustness in deep neural networks observe that unmodified networks fail to be uncertain when predicting on data far from the training distribution, or that they are unable to capture epistemic uncertainty (example critiques in section \ref{sec_app_softmax_critiques}). 
Intuition for this deficiency can be communicated through visualisations of low-dimensional input spaces where a network's softmax nonsensically extrapolates with increasing confidence, as in Fig. \ref{fig_2d}a. 
This is one motivation for research into methods such as ensembling and Bayesian neural networks.
When evaluating the quality of these methods empirically, softmax confidence is invariably used as a baseline -- what's surprising is that despite the rhetoric, and without any obvious mechanism enabling capture of epistemic uncertainty, softmax confidence nevertheless performs moderately well. 


For example, out-of-distribution (OOD) detection requires predicting if a data point is from the training distribution or not (fig. \ref{fig_2d} contains example distributions), where success depends on good epistemic uncertainty estimates. A system incapable of capturing epistemic uncertainty should do no better than random guessing -- an AUROC of 50\%. Yet directly interpreting softmax confidence as uncertainty can score from 75\% to 99\% across datasets and architectures (section \ref{sec_app_baseline_table}).
Meanwhile, adding modifications purposefully designed to capture uncertainty typically improves AUROC by just 1\% to 5\%. 
This suggests that softmax confidence may be more useful as an indicator of epistemic uncertainty than widely thought.

\textbf{Contributions:}
The goal of this paper is to understand the value of softmax confidence as a proxy for epistemic uncertainty in common vision benchmarks for detection of non-adversarial OOD data. This is valuable in improving our understanding of when, why, and how much we can trust deep learning's most accessible uncertainty estimates. 
We defer mathematical results (section \ref{sec_app_proofs2}), extended related work (section \ref{sec_app_relatedwork}), and full experimental details (section \ref{sec_app_experiment_details}) to the appendix, using the main paper to intuitively communicate results. 
Concrete contributions are as follows:
\begin{itemize}[leftmargin=15pt]
\item \textbf{Uncertain regions of the softmax (section \ref{sec_geom_softmax}).} 
We analytically study the softmax layer, defining regions on the final-layer that will be correctly labelled as OOD. 
\item \textbf{Decision boundary structure (section \ref{sec_final_layer}).}
Good structuring of a softmax's decision boundaries is the first important factor in its OOD detection ability, since it enlarges the region that an OOD input is allowed to fall within. We derive a theoretically optimal structuring scheme, finding it closely matches the structure found in trained networks.
\item \textbf{Deep networks filter for task-specific features (section \ref{sec_misleading_intuition}).} 
Deep networks extract features useful for a specific training task -- this serves as an implicit filter of OOD features. We demonstrate why our low-dimensional intuitions mislead us, and propose a more representative mental model.
\item \textbf{Diagnostic experiments (section \ref{sec_failure_measure}).} We summarise reasons softmax confidence can fail and measure the importance of each in common benchmark OOD tasks. Further experiments find that pre-trained/fine-tuned networks remedy two of these causes.
\end{itemize}

\begin{figure}[t!]
\begin{center}
\vspace{-0.02in}
\includegraphics[width=0.2\columnwidth]{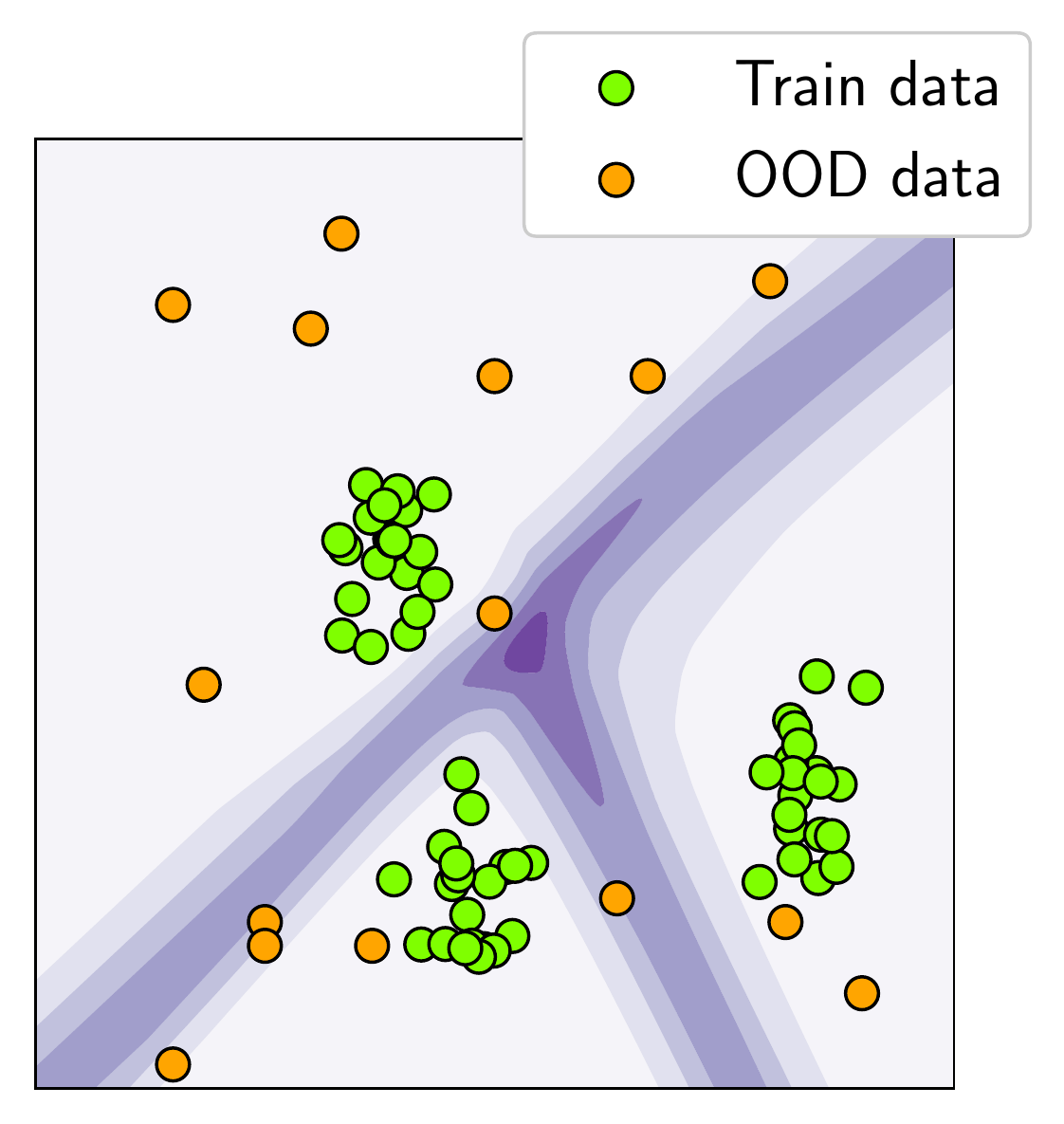}
\put(-59,-7){\small Input space}
\put(-70,83){\small a) Low-dimensional input, shallow network}
\includegraphics[width=0.03\columnwidth]{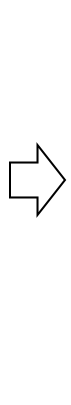}
\hspace{-0.05in}
\includegraphics[width=0.2\columnwidth]{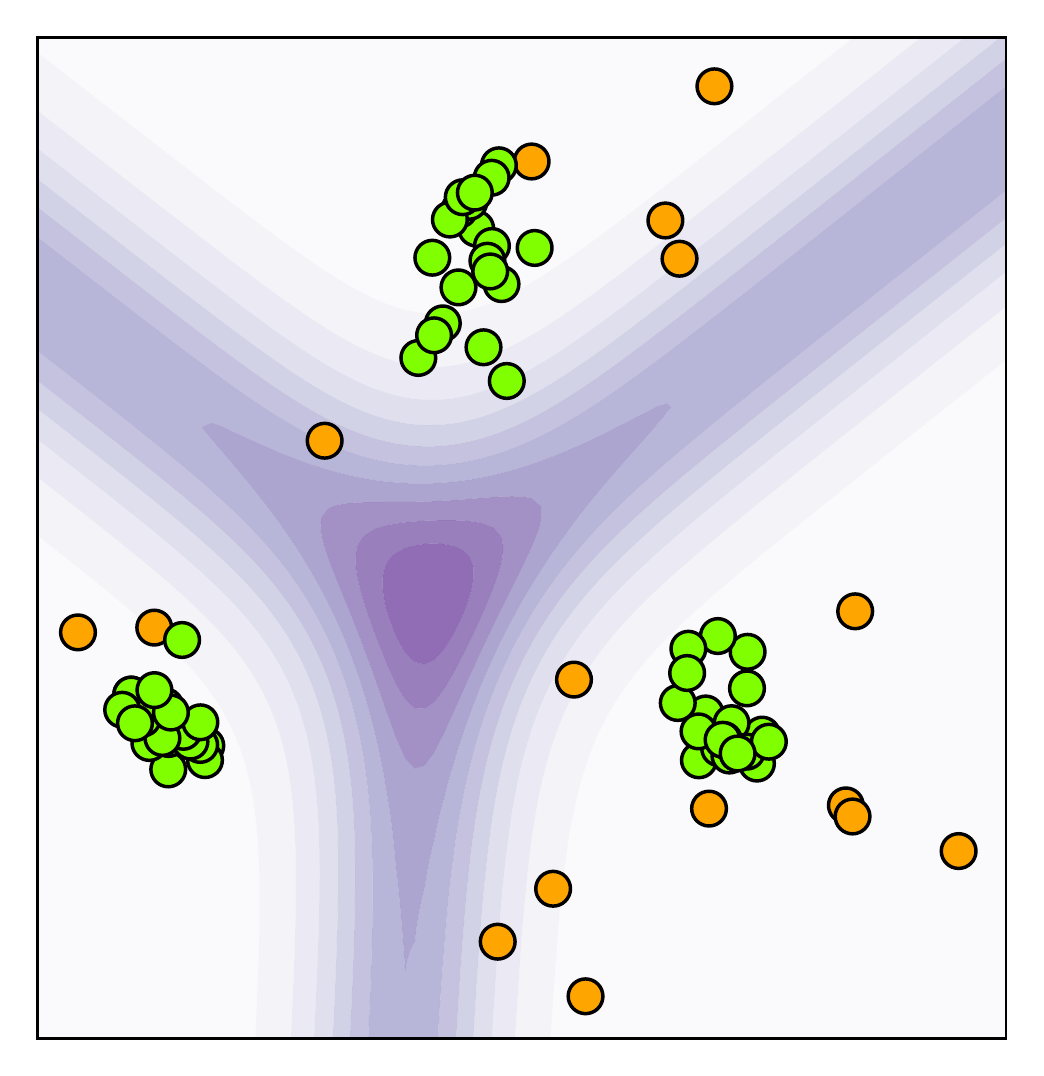}
\put(-92,-7){\small Final-layer PCA visualisation}
\includegraphics[width=0.021\columnwidth]{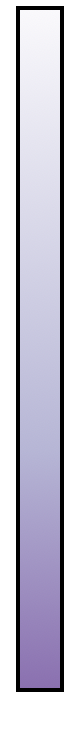}
\put(0,70){\tiny High softmax}
\put(0,63){\tiny confidence}
\put(0,13){\tiny Low softmax}
\put(0,6){\tiny confidence}
\hspace{0.5in}
\includegraphics[width=0.2\columnwidth]{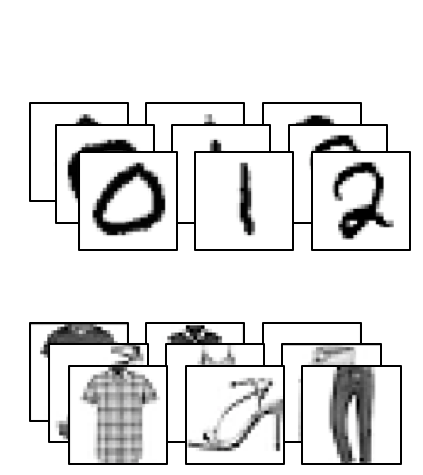} 
\put(-60,72){\small \textbf{\textcolor{tpgreen}{Train data}}}
\put(-60,31){\small \textbf{\textcolor{orange}{OOD data}}}
\put(-65,-7){\small Example inputs}
\includegraphics[width=0.03\columnwidth]{images/10_intro/arrowtp}
\hspace{-0.05in}
\includegraphics[width=0.2\columnwidth]{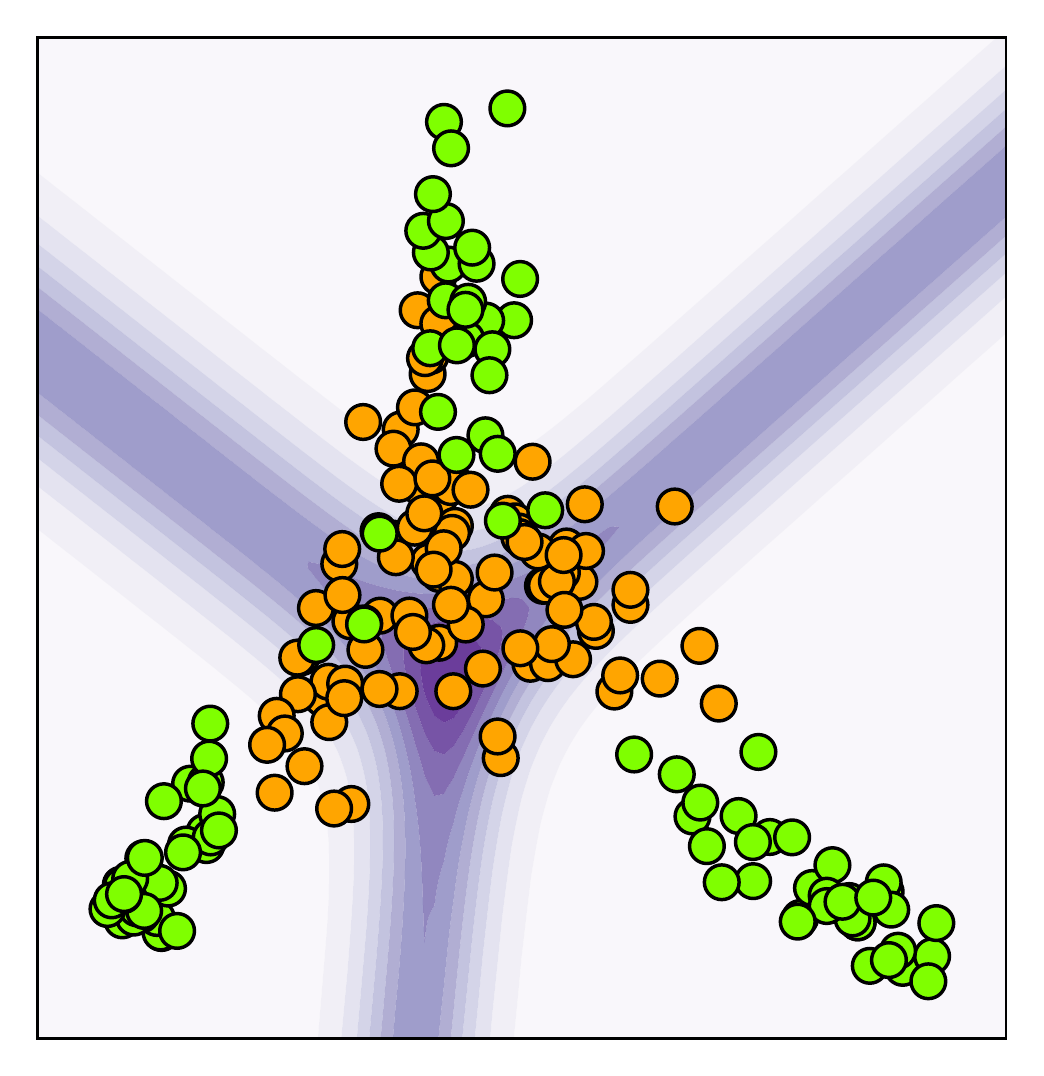}
\put(-92,-7){\small Final-layer PCA visualisation}
\put(-160,83){\small b) High-dimensional input, deep network }
\caption{a) In low-dimensions, softmax confidence is unreliable as a measure of uncertainty -- OOD data far from the training distribution falls into high confidence extrapolation regions. However, with more complex tasks and deeper networks, implicit biases encourage higher uncertainty for data outside the training distribution. b) For a LeNet trained on three classes of MNIST, OOD data from Fashion MNIST is reliably mapped to the low confidence region of the softmax layer.}
\label{fig_2d}
\end{center}
\vskip -0.2in
\end{figure}

\section{Background}\label{sec_bg}


Consider a neural network taking an input $\x \in \mathbb{R}^D$, producing final-layer activations $\z = \psi(\x) \in \mathbb{R}^H$, which are passed through a softmax function to give predictions, $\hat{\y} =\sigma(\z) \in [0,1]^K$,
\begin{equation}
\sigma(\z)_i = \frac{\exp{\w_i^\intercal \z}}{\sum^K_{j=1} \exp{\w_j^\intercal \z}}, \label{eq_softmax} 
\end{equation}
where the final-layer weight $\mathbf{W} \in \mathbb{R}^{H \times K}$ can be indexed so $\w_i$ represents column $i$. One can think of appending a one element to $\z$ in place of a bias. Targets for $K$ classes use a one-hot-encoding. 

The training distribution is written, $\x \sim \mathcal{D}_\text{in}$ or $p_\text{in}(\x)$. A distribution over the final-layer activations is induced, $\z \sim \mathcal{D}_\text{in}$, or, $p_\text{in}(\z)$. The outlier, or OOD distribution is given by $\x \sim \mathcal{D}_\text{out}$ or $p_\text{out}(\x)$, and similarly for $\z$. A dataset assumes $N$ data points drawn i.i.d..

Two common measures of uncertainty derived from softmax confidence are max predicted probability for any class, and entropy. We define two estimators based on these, ensuring higher is more uncertain. In this paper `softmax confidence' refers to either estimator. 
\begin{align}
U_\text{max}(\z) &\coloneqq - \max_i \sigma(\z)_i   \;\;\;\;\;\;\;
U_\text{entropy}(\z) \coloneqq - \sum_{i=1}^K \sigma(\z)_i \log \sigma(\z)_i  \label{eq_soft_ent}
\end{align}
We further consider an uncertainty metric operating on the final-hidden layer activations, but not derived from the softmax, similar to Lee et al. \citeyearpar{Lee2018a}. A probability density is estimated, $\hat{q}(\z) \approx p_\text{in}(\z)$, and the negative log likelihood is used as an uncertainty score (again higher is more uncertain),
\begin{align}
U_\text{density}(\z) &\coloneqq -\log \hat{q}(\z).
\end{align}
In this work we use a Gaussian mixture model (GMM) with $K$ (equal to the number of classes) components to estimate this density (though not restricted to one component per class, and with no restrictions on the covariance structure). Marginal probabilities are given by, $\pi_i \coloneqq p(\y=i)$. A code snippet in section \ref{sec_app_codesnippet} summarises the implementation.
\begin{align}
    \hat{q}(\z) = \sum_{i=1}^K \pi_i \mathcal{N}(\z; \pmb{\mu}_i,\pmb{\Sigma}_i).
\end{align}
\textbf{Epistemic vs Aleatoric:}
Whilst in everyday usage `uncertainty' describes the single concept of being unsure, 
for modelling purposes it is useful to decompose it into
 two types; aleatoric and epistemic \citep{DerKiureghian2008, Gal2015}. \textbf{Aleatoric} refers to uncertainty due to classes that overlap in input space,  e.g. a network trained on MNIST digits should have aleatoric uncertainty if asked to classify an input appearing between `1' and `7'. \textbf{Epistemic} refers to uncertainty about the model or its parameters; if the MNIST network is asked to classify an image of clothing, it should be uncertain due to lack of knowledge about how to handle this type of input. 
A softmax layer in isolation can learn to output a probability in between 0 and 1 to catch overlapping classes (aleatoric), but fails to decrease its confidence if queried far from the training data (epistemic).





\textbf{Related Work:}
Neural networks have long combined a softmax output and cross-entropy loss for classification \citep{Bridle1990}. Recent work has shown that interpreting softmax confidence as predictive probabilities can have several pitfalls; they are poorly calibrated (generally overconfident) \citep{Guo2017}, and can be easily manipulated by adversarial examples \citep{Nguyen2015}. 
It's also been claimed there is no reason to trust them outside of the training distribution -- \textit{"[the softmax output] is often erroneously interpreted as model confidence."} \citep{Gal2015}.

This paper investigates the last point. Whilst most work on uncertainty and neural networks criticises softmax uncertainties only informally, several works more rigorously demonstrate weaknesses.
Hein et al. \citeyearpar{Hein} prove that any input can be magnified, $\tilde{\x} = \alpha \x$, with large $\alpha>1$, to produce an arbitrarily high confidence prediction by ReLU networks. Whilst they define $\alpha \x$ as `far from the training data', our work uses `far' to mean unfamiliar image datasets with bounded pixel intensity (something trivial to check). 
%
Mukhoti et al. \citeyearpar{Mukhoti2021} released work concurrently with our own. Whilst they make strong claims about the relationship between softmax confidence and uncertainty, they focus on tasks with high aleatoric uncertainty -- we focus on standard datasets which have low aleatoric uncertainty. Section \ref{sec_app_comprehensive} provides more discussion and a comprehensive literature review.

\section{Uncertain Regions of the Softmax}\label{sec_geom_softmax}

This section analytically studies the softmax layer in isolation, defining regions (`valid OOD region') on the softmax layer which an OOD input must fall into to be correctly labelled as OOD. 
\begin{definition}{}
\label{def_valid_ood_region}
`Valid OOD region' specifies a region in $\mathbb{R}^H$, such that if an OOD point, $\z' = \psi(\x')$, lies within it, it will be more uncertain than at least $(1-\epsilon)$\% of the training data (true positive rate, one axis of the ROC). For, $\epsilon \in [0,1]$, some uncertainty estimator $U(\cdot)$, and indicator function $\mathbbm{1}(\cdot)$,
\begin{align}
    \mathcal{R} :\{\z' \in \mathbb{R}^H | 
    \mathbb{E}_{\z \sim \mathcal{D}_\text{in}} \left[\mathbb{I}(U(\z') > U(\z)) \right] > 1-\epsilon
    \}.
\end{align}
\end{definition}
Section \ref{sec_vec_fields} provides further insight into softmax confidence, deriving uncertainty vector fields for each of our estimators, formalising common intuition about the danger of softmax extrapolations.



\begin{figure}[b!]
\begin{center}
a) $K=2, H=2, \epsilon=0.05$ \hspace{1.2in} b) $K=3, H=2, \epsilon=0.05$
\includegraphics[width=0.25\columnwidth]{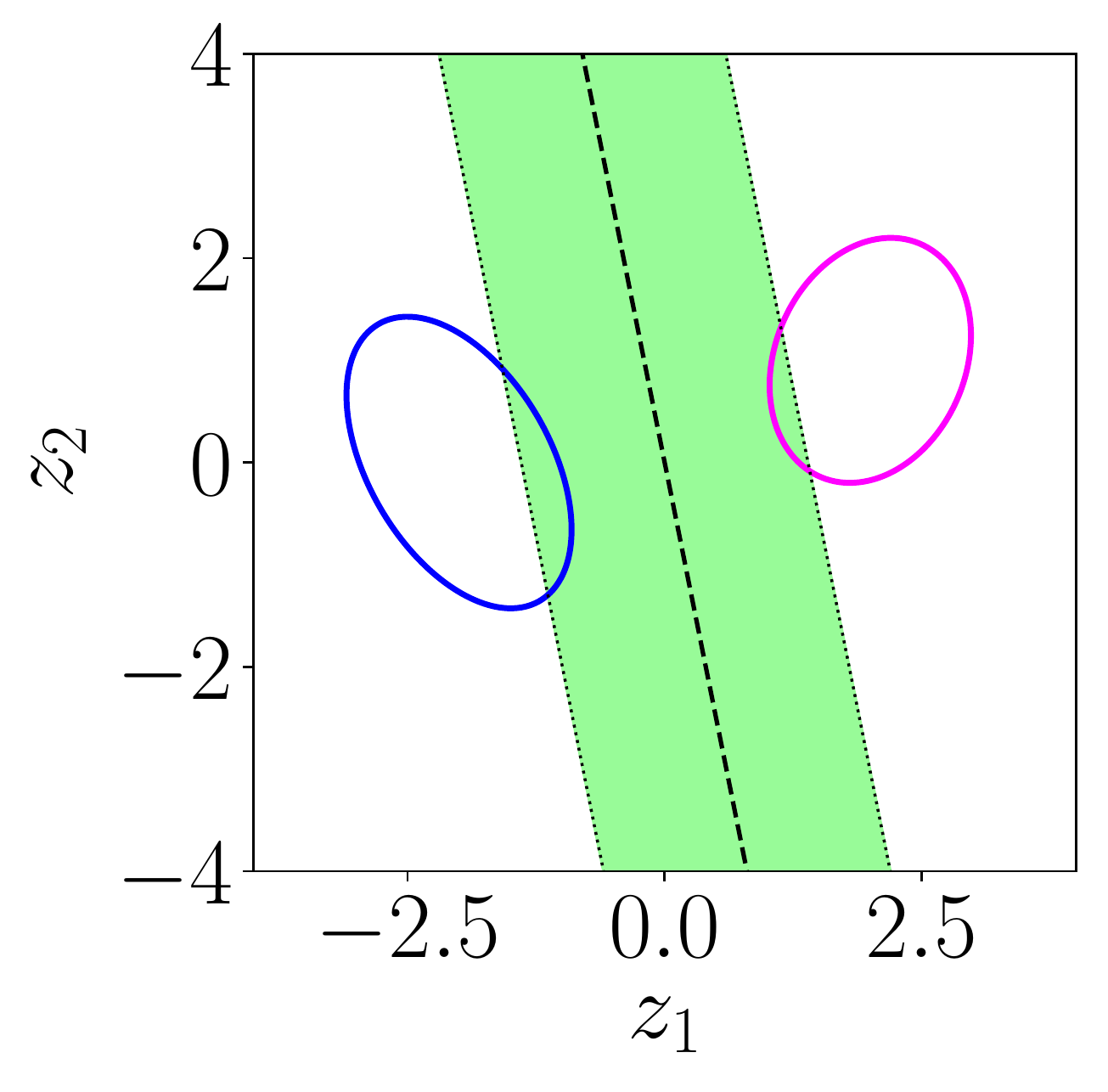}
\put(-65,-6){\small $U_\text{max}$ \& $U_\text{entropy}$}
\put(-75,72){\small \textcolor{blue}{$p(\z|y=1)$}}
\put(-40,81){\small \textcolor{magenta}{$p(\z|y=2)$}}
\includegraphics[width=0.25\columnwidth]{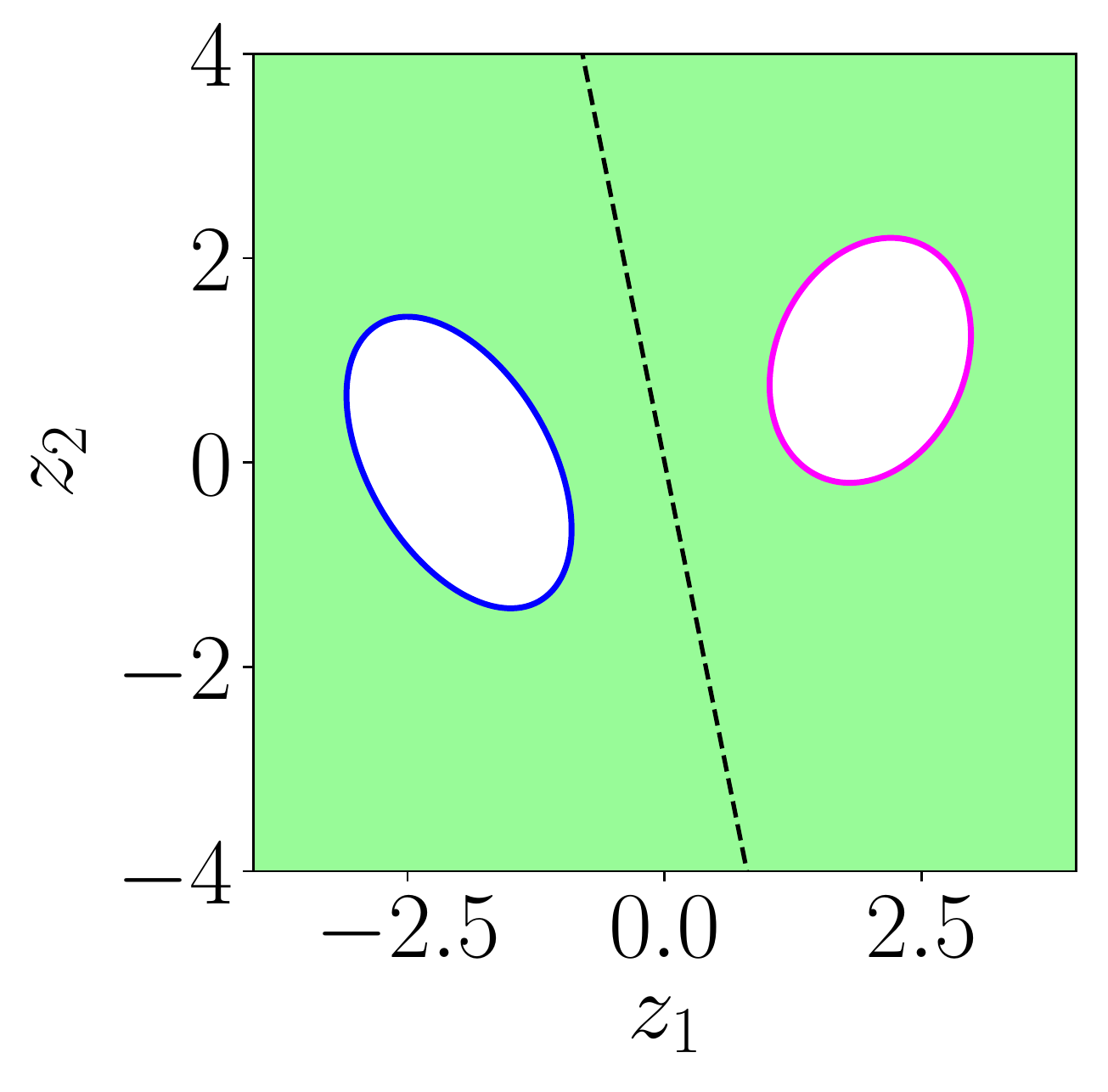}
\put(-52,-6){\small $U_\text{density}$}
\includegraphics[width=0.25\columnwidth]{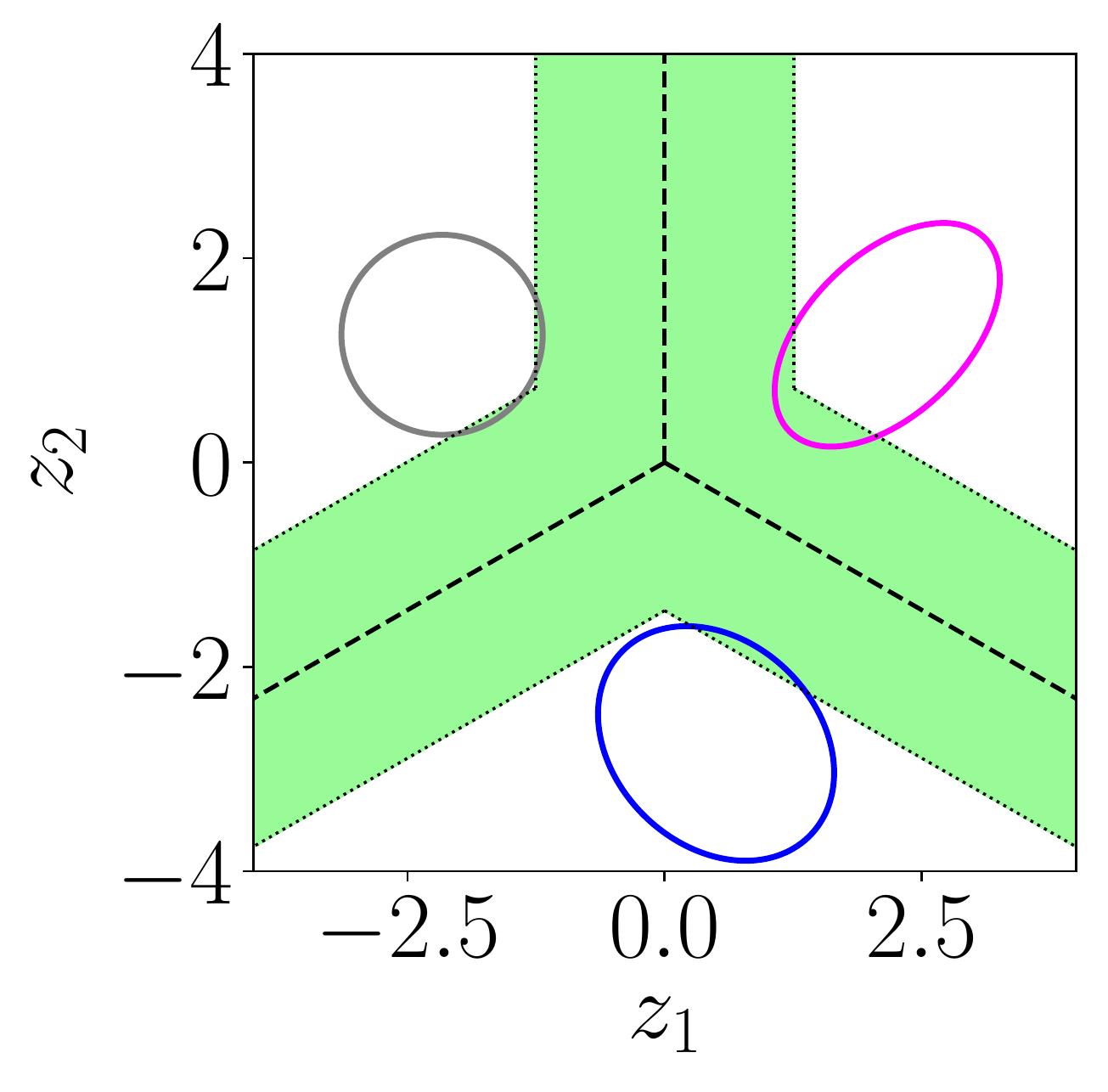}
\put(-80,-6){\small $U_\text{max}$, linear approximation}
\includegraphics[width=0.25\columnwidth]{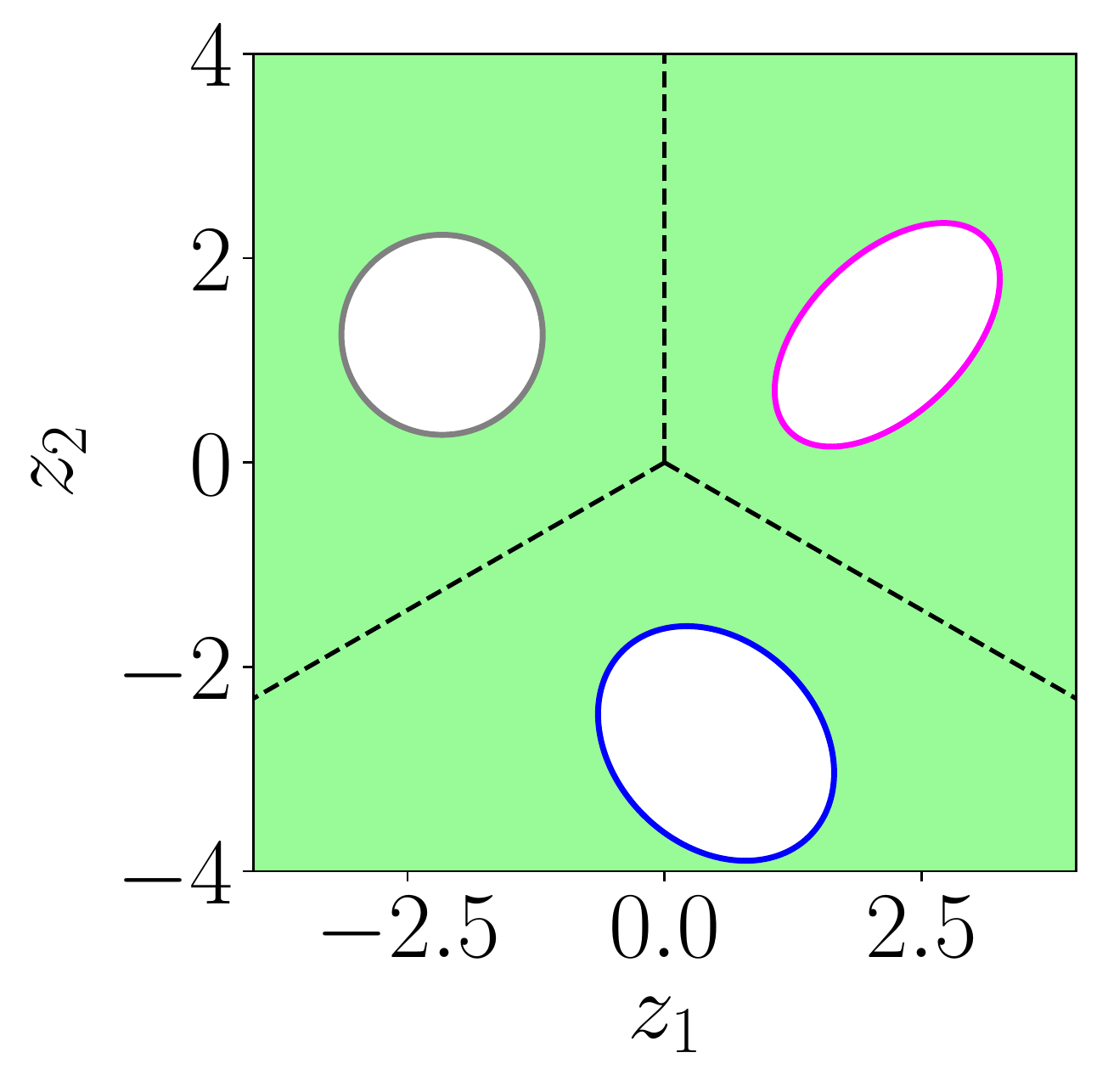}
\put(-52,-6){\small $U_\text{density}$}
\caption{The valid OOD region is shown in green for each uncertainty estimator. Gaussian ellipses capture 95\% of clusters. For softmax confidence, OOD data must fall closer to a decision boundary than 95\% of training data, whilst the density estimator only demands it fall outside of training clusters.}
\label{fig_valid_ood_region}
\end{center}
\end{figure}

\subsection{Valid Out-of-Distribution Region Definition}\label{sec_region_defn}


Fig. \ref{fig_valid_ood_region} illustrates valid OOD regions for all estimators. Roughly speaking, for $U_\text{max}$ and $U_\text{entropy}$, OOD data must fall \textit{closer to a decision boundary} than $(1-\epsilon)$\% of the training distribution to be in the valid OOD region. For $U_\text{density}$, it is required only that OOD data does not overlap with the training distribution. Observe that if a sample is labelled OOD by $U_\text{max}$ and $U_\text{entropy}$, it will usually also be labelled OOD by $U_\text{density}$, but not vice versa -- we exploit this property in section \ref{sec_failure_measure}'s  measurements.

These regions are formally defined in theorem \ref{theorem_valid_ood_region_2class} for $U_\text{max}$ and $U_\text{entropy}$ for two classes. 
Making the assumption that final-layer features in the training data follow a mixture of Gaussians allows analytical integration of class clusters in corollary \ref{sec_corollary_2class_gaussian}.
Specification of the regions for $U_\text{max}$ and $U_\text{entropy}$ becomes more difficult for higher numbers of classes, since intersections between decision boundary hyperplanes create curved valid OOD regions.
However, we define an approximation of the $U_\text{max}$ valid OOD region for general numbers of classes using pairs of linear hyperplanes, offset from the decision boundary by a distance $\alpha(\w_i - \w_j)$  (def. \ref{def_linearapprox_region}).
This linear approximation forms a subset of the exact valid OOD region (corollary \ref{corollary_linear_region_is_subset_of_exact}), matching the exact valid OOD region for $K=2$ (proposition \ref{prop_linearapprox_equalfor_k2}) and also matching the exact valid OOD region at nearly any point at large magnitudes (theorem \ref{theorem_exact_linear_convergence}).
Specifying the valid OOD region for $U_\text{density}$ is more straightforward for general classes, $K \geq 2$, as per proposition \ref{prop_valid_ood_region_density}.


\section{Implicit Bias 1: Approximately Optimal Decision Boundary Structure}

\label{sec_final_layer}

The previous section defined the valid OOD region presuming the distribution of training features, $p_\text{in}(\z)$, was given. In fact, a neural network learns the transformation, $\z = \psi(\x)$, and while $p_\text{in}(\x)$ is fixed, the shape and position of $p_\text{in}(\z)$ can itself be optimised -- this section asks what decision boundary structure emerges when $p_\text{in}(\z)$ and $\w_i$'s are jointly optimised.
Section \ref{sec_optimal_decision_boundary} studies this in an idealised setting, while section \ref{sec_decision_analysis} analyses it empirically in trained networks. Section \ref{sec_optimal_boundary_importance} considers how the valid OOD region is affected by structure, and demonstrates experimentally that sub-optimal softmax boundaries can degrade OOD detection.



\subsection{Optimal Decision Boundary Structure}
\label{sec_optimal_decision_boundary}
\begin{definition}
\label{def_optimal_structure}
`Optimal decision boundary structure' (also `optimal structure') refers to softmax weight vectors that have the following properties. 1) All weight magnitudes are equal, $||\w_i||=c_1 \; \forall i$. 2) Bias values (absorbed into $\w_i$) are zero. 3) $\w_i$'s are `evenly distributed'; if $\theta_{i,j}$ represents the angle between $\w_i$ \& $\w_j$, then, $\cos \theta_{i,j} = \frac{-1}{K-1} \; \forall i \neq j$. 
Where relevant, it further means the training distribution features, $p_\text{in}(\z)$, have the following properties. Let $\pmb{\mu}_i \coloneqq \mathbb{E}_{\z \sim p(\z|\y=i)}[\z]$.
1) Mean vectors for each class are of constant magnitude, $|| \pmb{\mu}_i ||=c_2 \; \forall i$, and as large as possible subject to regularisation.
2) Weight and mean vectors of the same class point in the same direction, $ \pmb{\mu}_i= c_3 \w_i$.
3) The variance of each cluster is small, $\mathbb{V}\text{ar}_{\z \sim p(\z|\y=i)}[\z] \approx \mathbf{0}$.
\end{definition}

\begin{wrapfigure}{r}{0.45\columnwidth}
\vskip -0.4in
\begin{center}
\includegraphics[width=0.45\columnwidth]{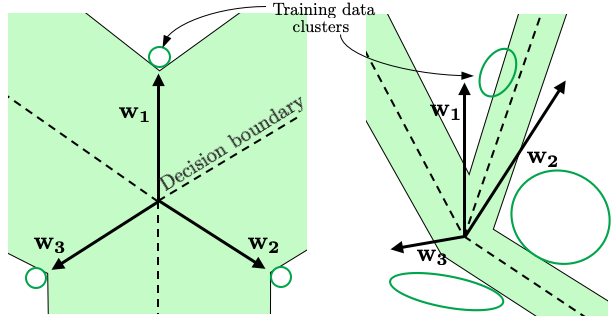}
\put(-165,-5){\small Optimal structure} 
\put(-75,-5){\small General structure} 
\caption{Optimal decision boundary structure for $H=2, K=3$. The valid OOD region is overlaid in green.}
\label{fig_optimal}
\end{center}
\vskip -0.2in
\end{wrapfigure}
Fig. \ref{fig_optimal} illustrates an optimal structure, compared to a general (non-optimal) structure.
Under the assumption that class weights are balanced, $\pi_i=1/K$, and mean magnitudes are fixed, $||\pmb{\mu}_i||=c_3$, theorem \ref{theorem_hyperplane_structure_optimal} shows that the optimal structure minimises a regularised cross-entropy loss. This is rather intuitive -- training classes should be positioned in small tight clusters as far as possible from each other to minimise this loss. There is inherent symmetry, meaning weights must be evenly distributed. The angle, $\cos \theta_{i,j} = \frac{-1}{K-1}$, arises in spherical coding problems.
We will later empirically evidence that optimal structures produce lower cross-entropy losses (fig. \ref{fig_counterfact_structs}).

\subsection{Empirical Measurement of Decision Boundary Structures in Trained Networks}
\label{sec_decision_analysis}

Analysing properties of final-layer weights in trained networks reveals three key properties of decision boundary structure. Fig. \ref{fig_weight_analysis} plots histograms of $||\w||$, $\mathbf{b}$, and $\cos \theta_{i,j} \forall i \neq j$ for various architectures and datasets (details in section \ref{sec_app_exp_traineddecisionboundary}).
1) Bias values tend to be small. 2) Weight vectors all have similar magnitude. 3) Weight vectors are approximately evenly distributed, with $\cos{\theta_{i,j}} \approx \frac{-1}{K-1}$.


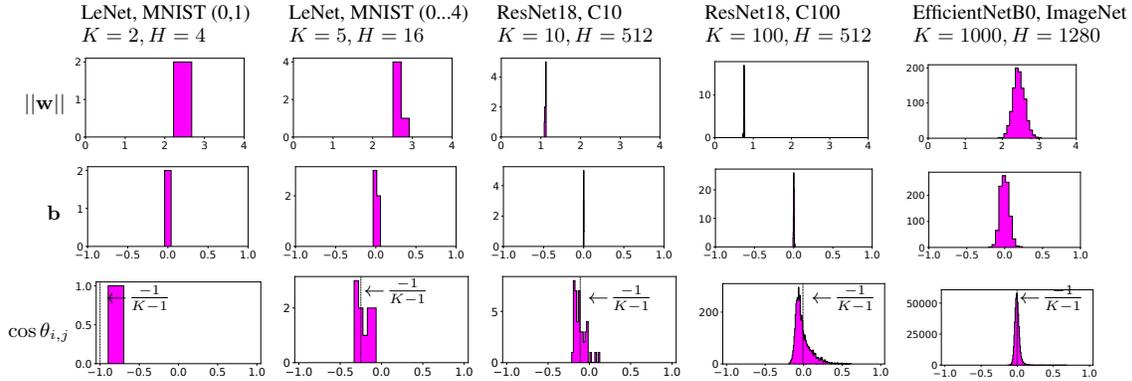
\begin{figure}[h!]
\resizebox{\textwidth}{!}{
\begin{tabular}{p{0.2\textwidth} p{0.2\textwidth} p{0.2\textwidth} p{0.2\textwidth} p{0.2\textwidth}}

\includegraphics[width=0.2\columnwidth]{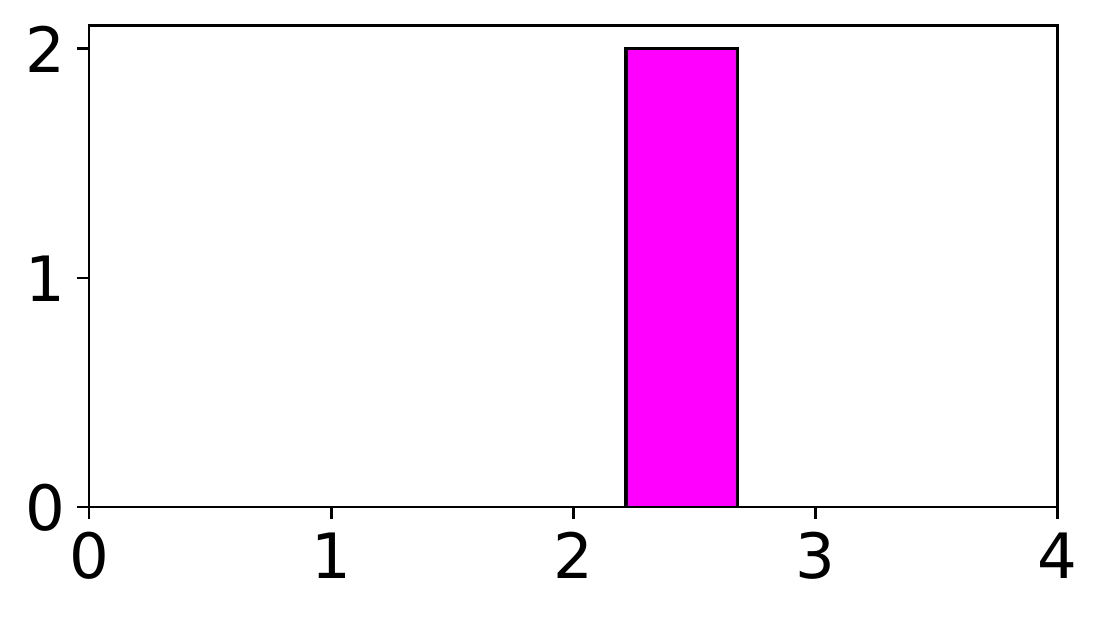}
\put(-75,60){\small LeNet, MNIST (0,1)} 
\put(-75,50){\small $K=2, H=4$} 
\put(-100,20){\small $||\w||$} 
  &
  \includegraphics[width=0.2\columnwidth]{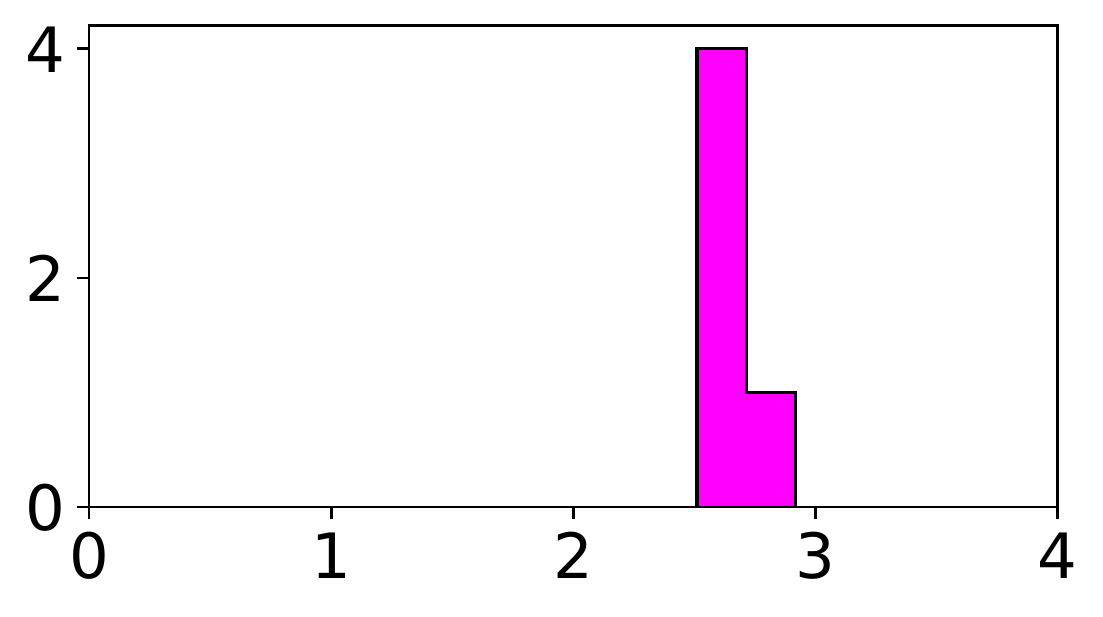}
  \put(-75,60){\small LeNet, MNIST (0...4)} 
\put(-75,50){\small $K=5, H=16$} 
&
\includegraphics[width=0.2\columnwidth]{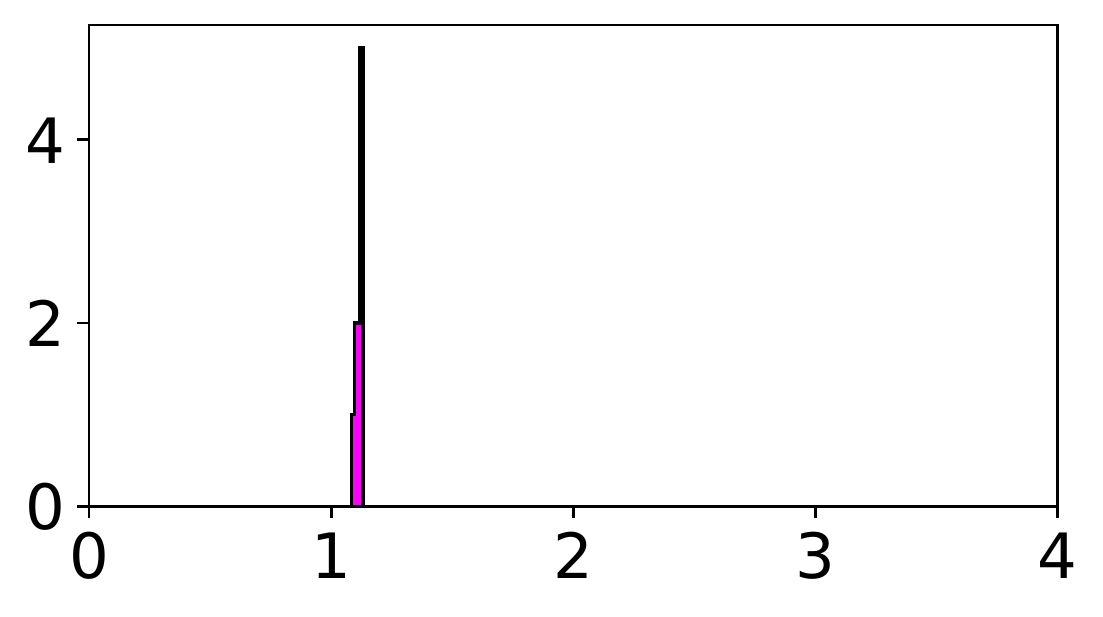}
\put(-75,60){\small ResNet18, C10} 
\put(-75,50){\small $K=10, H=512$} 
&
\includegraphics[width=0.2\columnwidth]{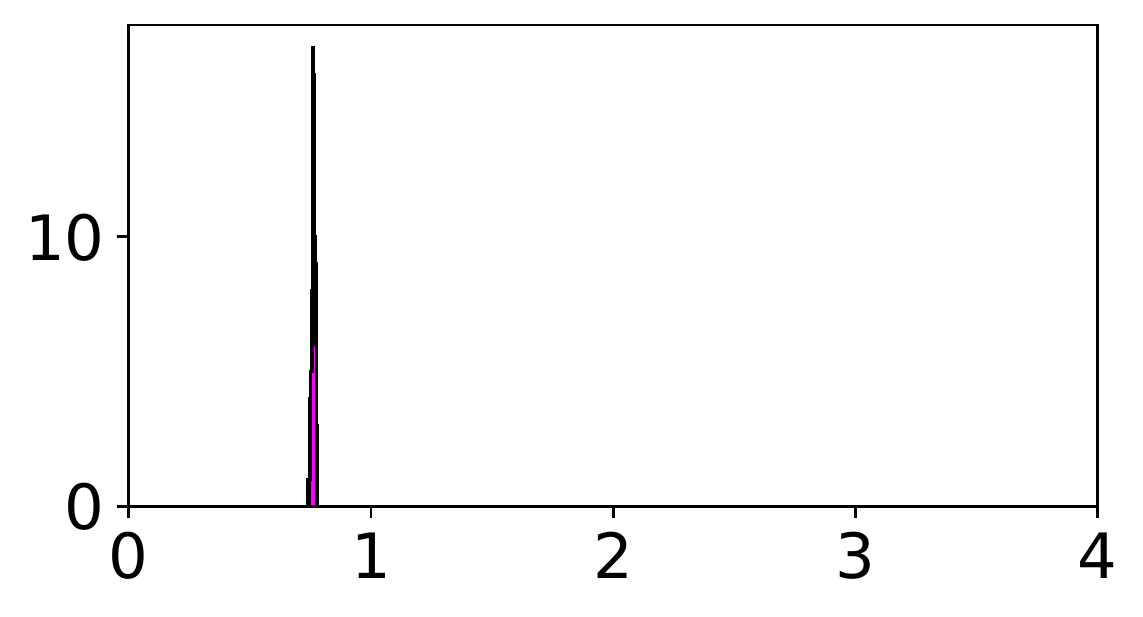}
\put(-75,60){\small ResNet18, C100} 
\put(-75,50){\small $K=100, H=512$} 
&
\includegraphics[width=0.2\columnwidth]{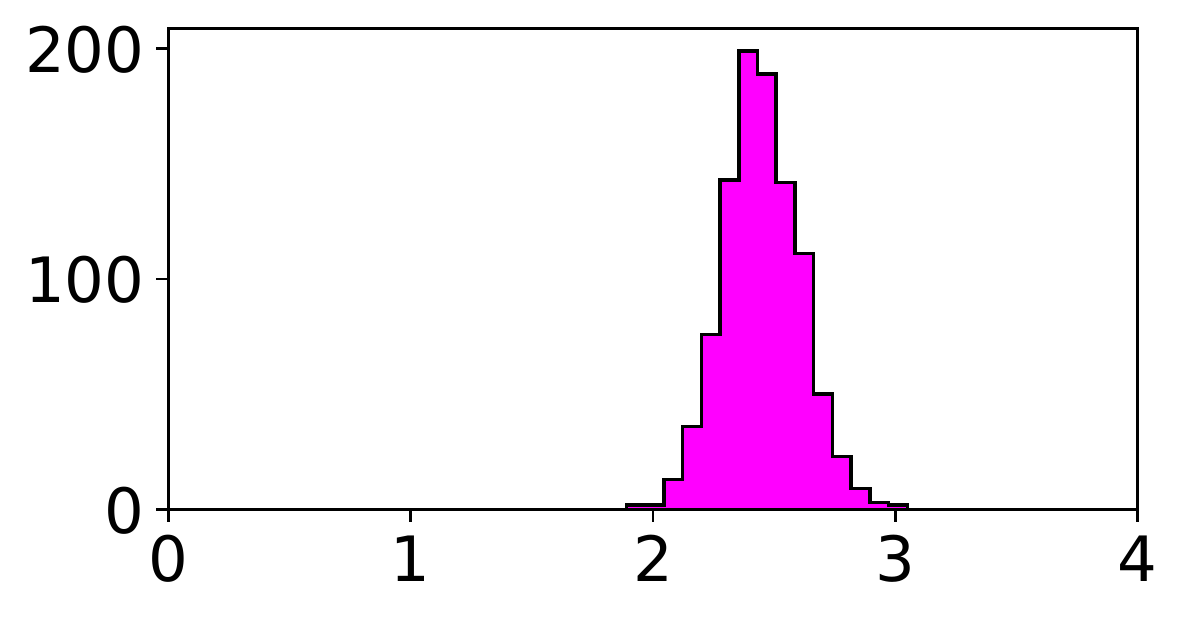} 
\put(-75,60){\small EfficientNetB0, ImageNet} 
\put(-75,50){\small $K=1000, H=1280$} 
\\

\includegraphics[width=0.21\columnwidth]{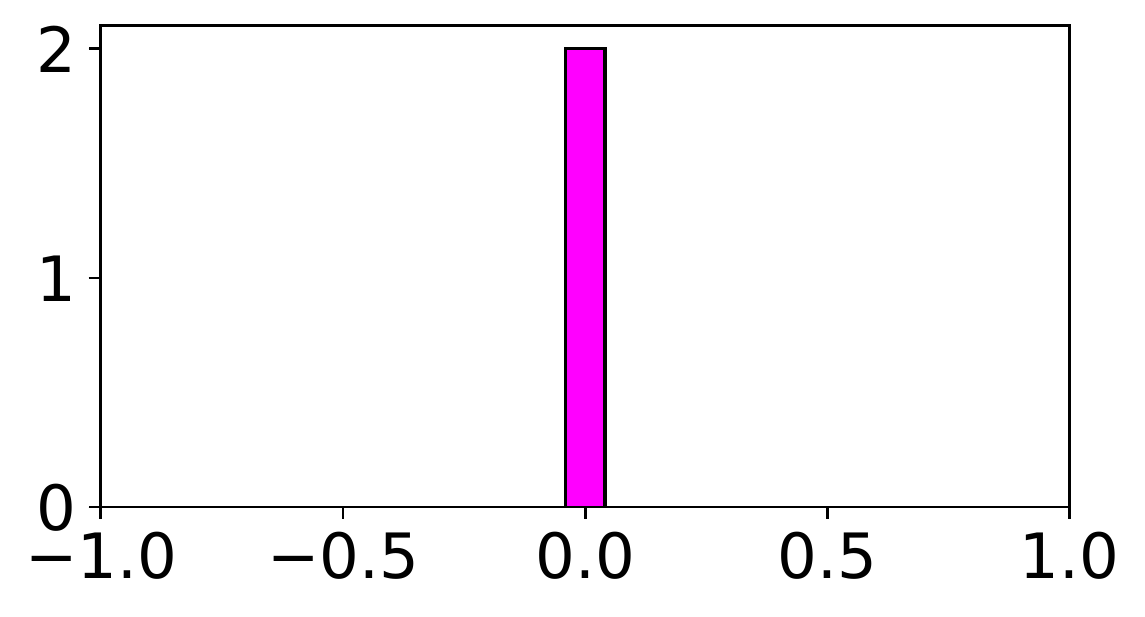}
\put(-94,20){\small $\mathbf{b}$} 
&
\includegraphics[width=0.21\columnwidth]{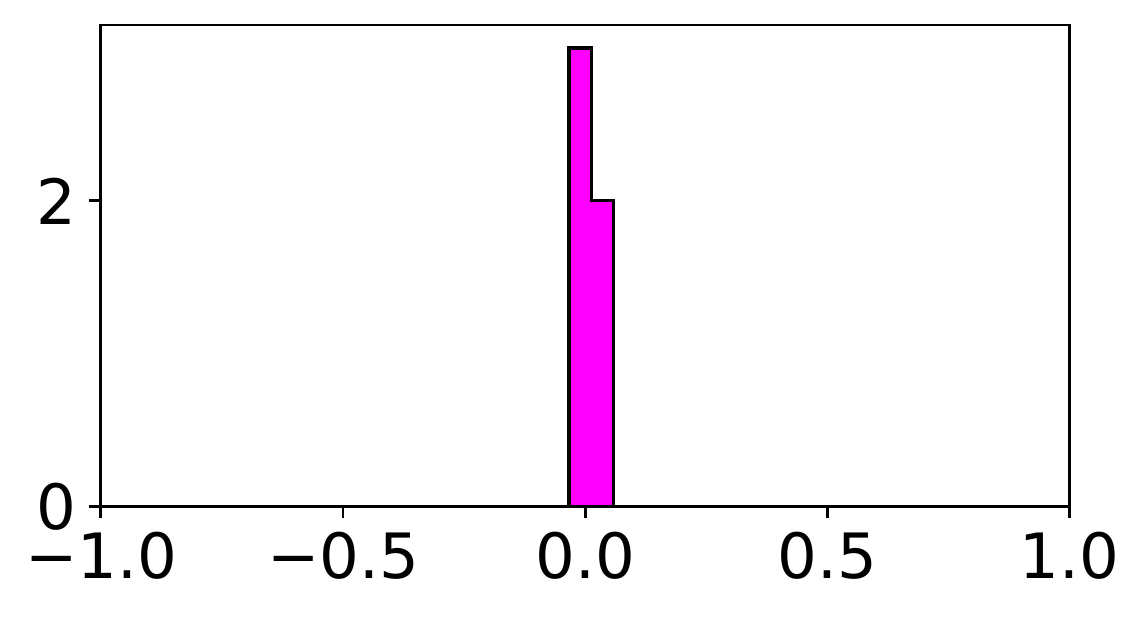}
&
\includegraphics[width=0.21\columnwidth]{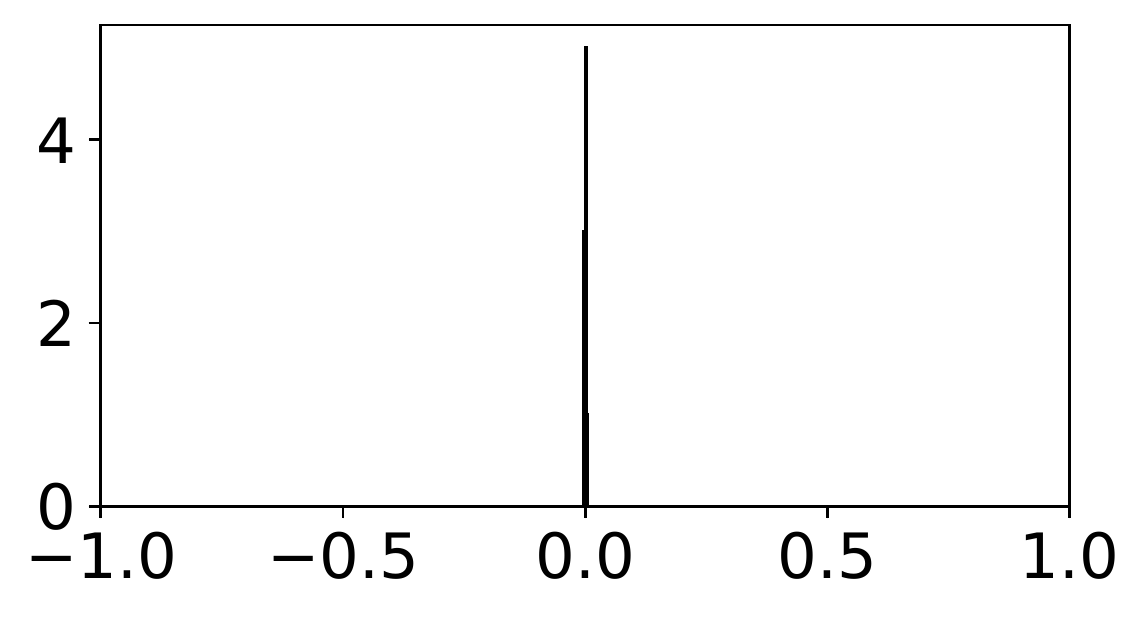}
&
\includegraphics[width=0.21\columnwidth]{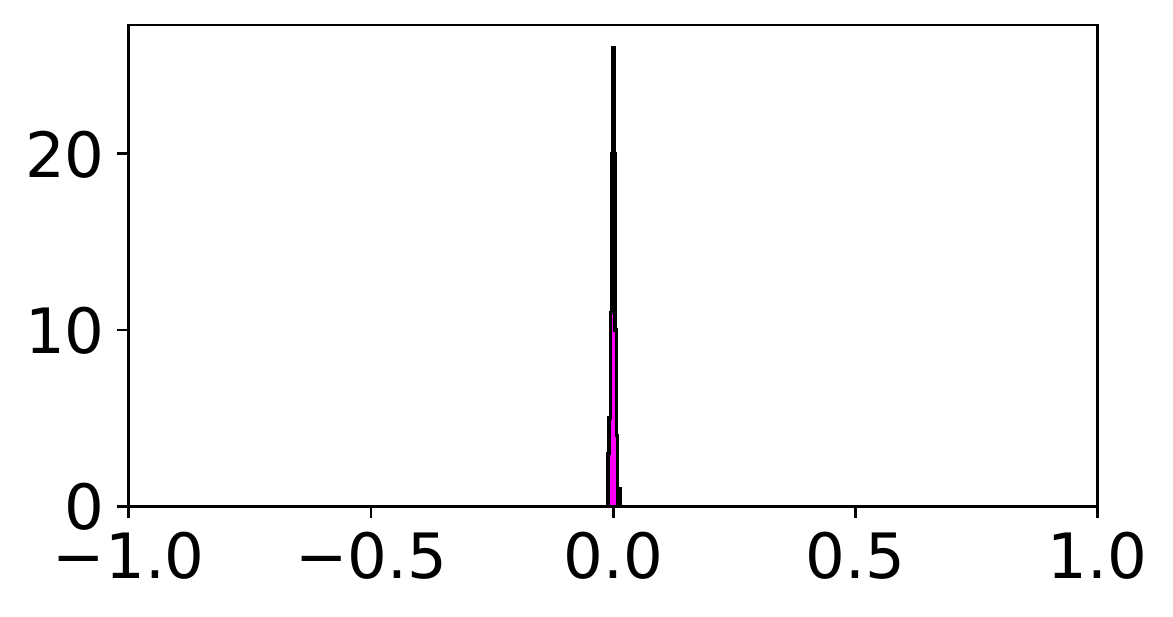}
&
\includegraphics[width=0.21\columnwidth]{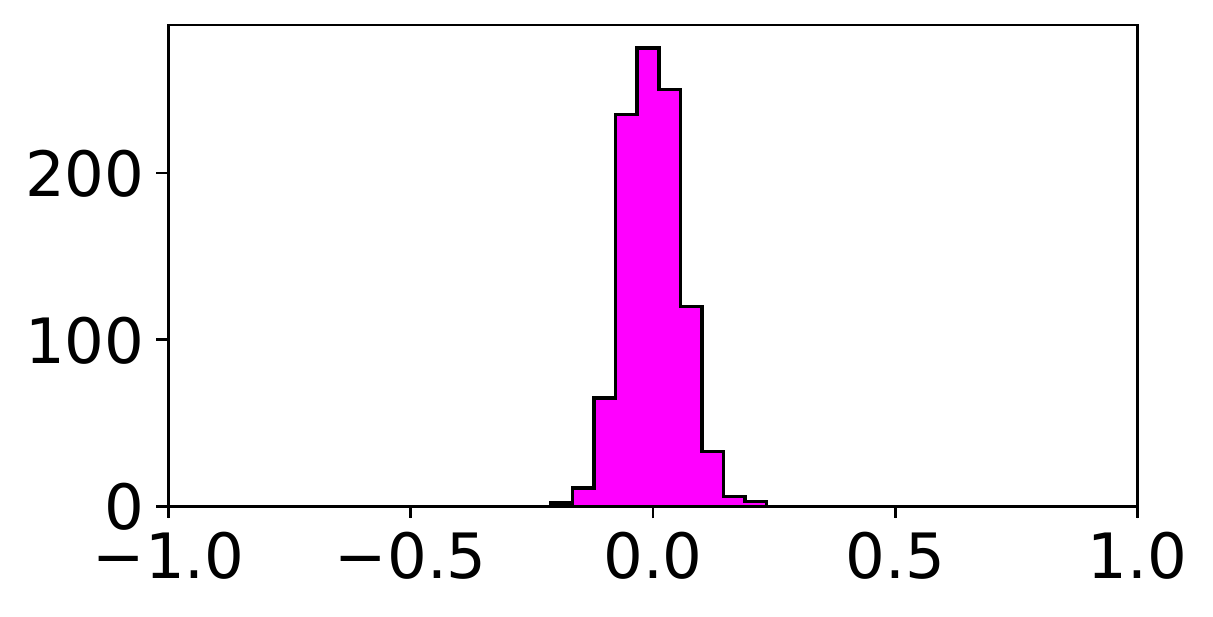} 
\\

\includegraphics[width=0.22\columnwidth]{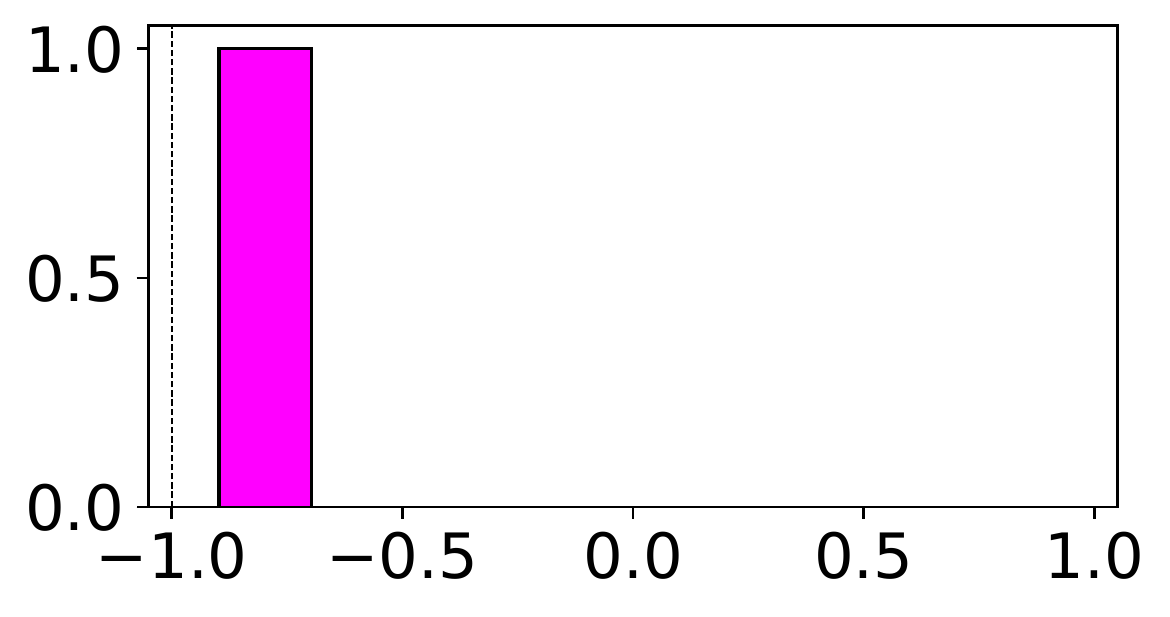}
\put(-115,20){\small $\cos \theta_{i,j}$} 
\put(-62,35){\small $\frac{-1}{K-1}$} 
\put(-72,35){\small $\leftarrow$} 
&
\includegraphics[width=0.22\columnwidth]{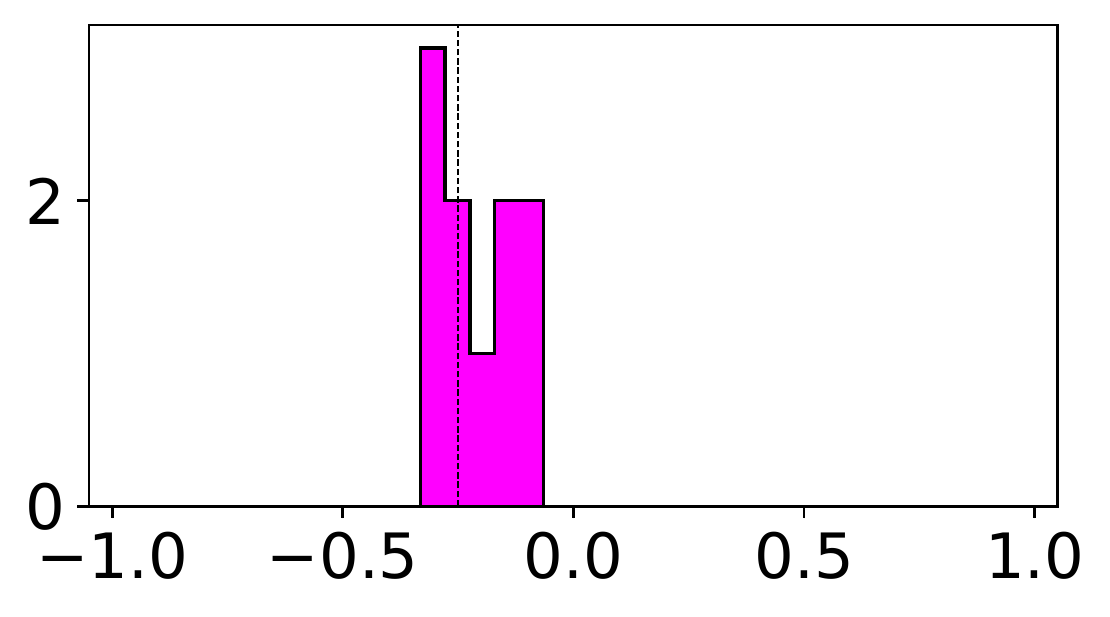}
\put(-40,38){\small $\frac{-1}{K-1}$} 
\put(-50,38){\small $\leftarrow$} 
&
\includegraphics[width=0.22\columnwidth]{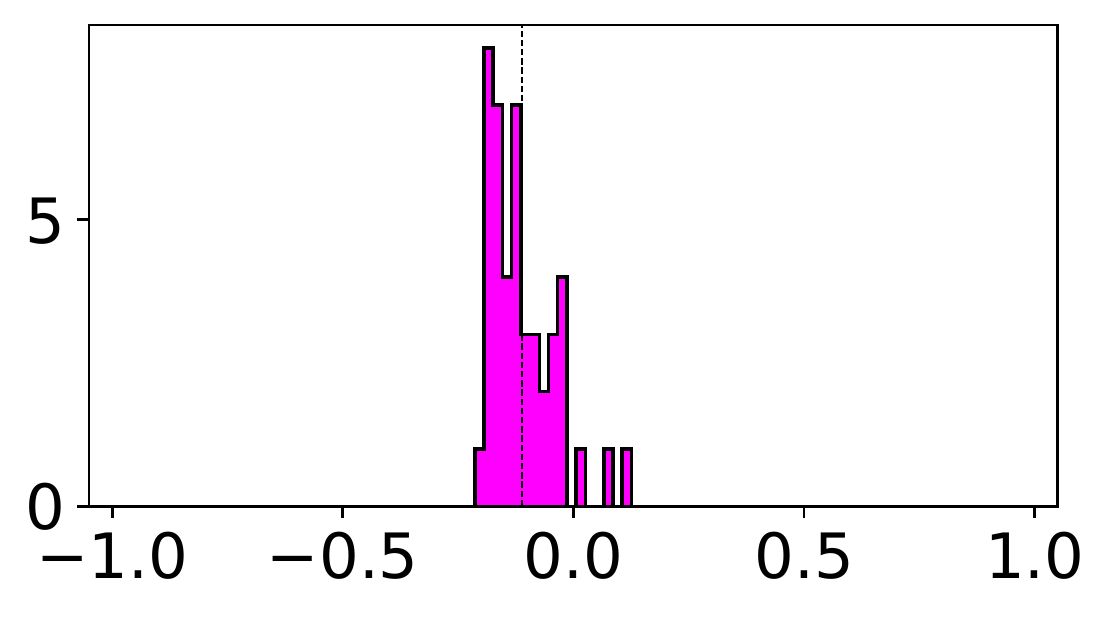}
\put(-33,35){\small $\frac{-1}{K-1}$} 
\put(-43,35){\small $\leftarrow$} 
&
\includegraphics[width=0.22\columnwidth]{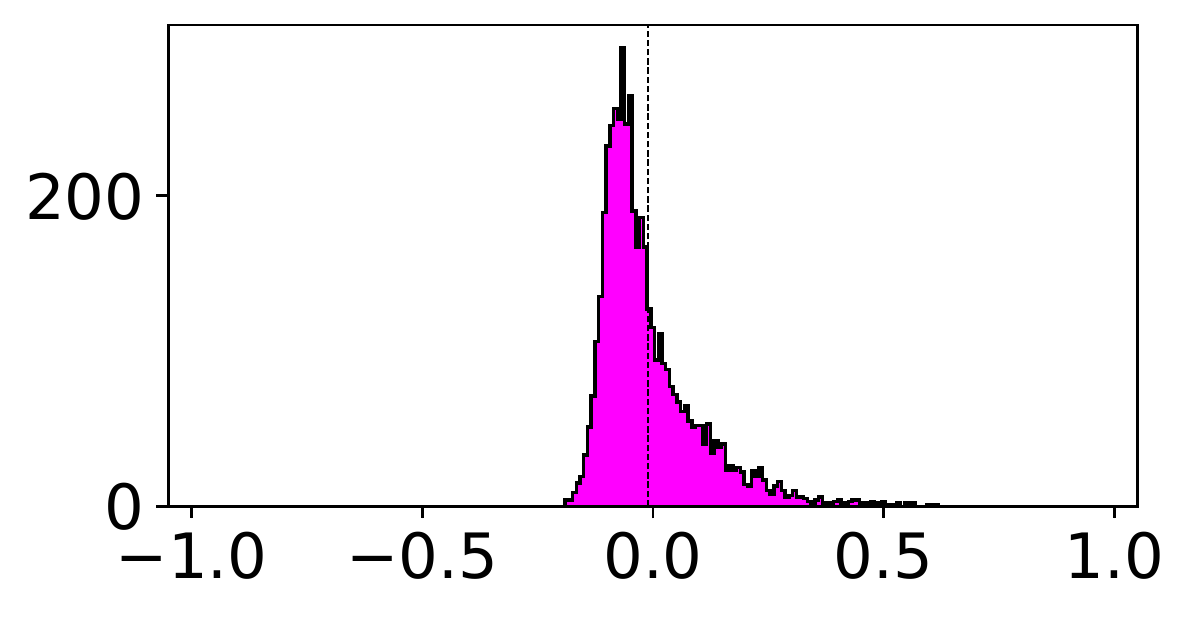}
\put(-26,35){\small $\frac{-1}{K-1}$} 
\put(-36,35){\small $\leftarrow$} 
&
\includegraphics[width=0.22\columnwidth]{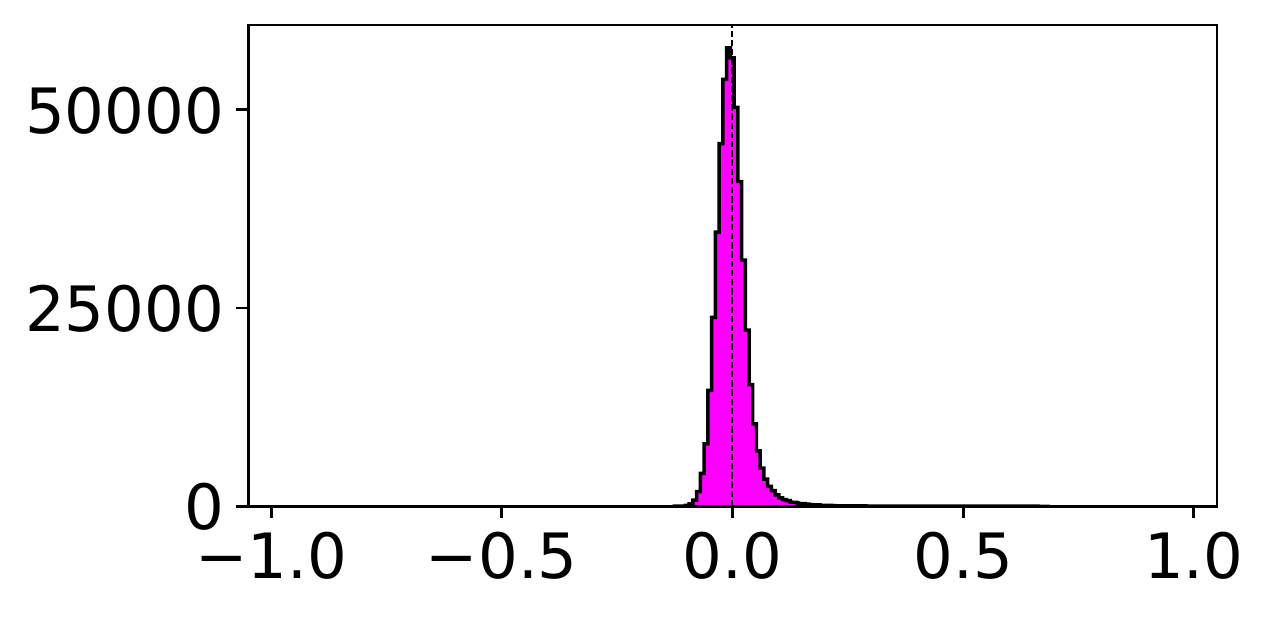}
\put(-26,35){\small $\frac{-1}{K-1}$} 
\put(-36,35){\small $\leftarrow$} 

\end{tabular}
}
\vskip -0.05in
\caption{Histograms of final-layer weight properties for various architectures/datasets show reasonable agreement with the theoretically optimal structure.}
\label{fig_weight_analysis}
\end{figure}

These properties match those of the optimal structure reasonably well, though for larger numbers of classes there tends to be feature sharing -- there are weak positive correlations between related classes, e.g. different dog breeds, which produces a positively skewed histogram.

Further evidence for agreement can be found in other experiments run in this paper. PCA visualisations of final-layer features in fig. \ref{fig_2d} show an evenly-spaced decision boundary, with well-separated data clusters (though these have not collapsed to a point). Fig \ref{fig_cifar_etc} shows that $\max_i \cos \theta_{\z,i} \approx 1$ for most of the training data, verifying that weight vectors and training data clusters are closely aligned, $\pmb{\mu}_i \approx c_3 \w_i $.

\subsection{The Effect of Decision Boundary Structure on OOD Detection}
\label{sec_optimal_boundary_importance}

Having described the decision boundary structure in neural networks, we show why this is a key factor in improving softmax confidence's OOD detection ability -- first by quantifying it's effect on the size of the valid OOD region, secondly by running OOD detection experiments with softmax weights frozen in sub-optimal structures.

\textbf{Effect on valid OOD region.} Fig. \ref{fig_optimal} depicts valid OOD regions for optimal and non-optimal structures. The volume of these regions is larger for the optimal structure (corollary \ref{corollary_optimal_maximises} verifies). This is useful for OOD detection since it creates an increased opportunity for OOD data to fall into these valid regions. 





\textbf{Counterfactual experiments.}
To empirically verify that good structuring is important for good OOD detection, we present a short set of experiments. We take a LeNet architecture ($H=16$) and freeze the softmax weight vectors in various configurations, then train the rest of the network on a subset of three MNIST classes ($K=3$) ($\mathcal{D}_\text{in}$), measuring OOD detection using Fashion MNIST as OOD ($\mathcal{D}_\text{out}$). Trainable refers to not freezing the final layer weights. Further details in section \ref{sec_app_exp_counterfactual_structs}.

Fig. \ref{fig_counterfact_structs} shows results along with illustrations of the structures tested. Whilst structure has limited impact on test accuracy -- all structures achieve around 99.5\% -- there are large differences in OOD detection ability, with only the trainable and optimal structures achieving high AUROC's. Cross-entropy loss is larger for non-optimal structures, supporting our theoretical result that the optimal structure minimises this.

\begin{figure}[t]
\begin{center}
\includegraphics[width=1.\columnwidth]{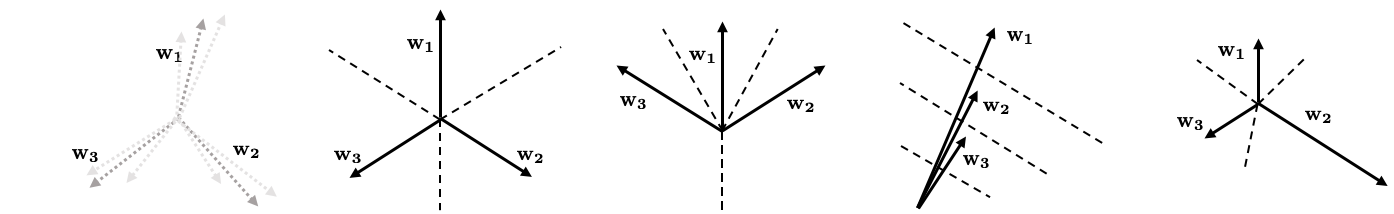}
	\put(-370,72){\small \textcolor{black}{Not frozen}}
	\put(-300, 72){\small \textcolor{black}{\rule{1.35in}{0.5pt} Last-layer weights frozen \rule{1.35in}{0.5pt}}}
    \put(-410,-10){\small Accuracy:}
    \put(-410,-20){\small Cross-ent.:}
    \put(-410,-30){\small AUROC:}
    \put(-380,60){ 1) Trainable}
    \put(-370,-10){\small 99.5\% $\pm$ 0.0 }
    \put(-370,-20){\small 0.011 $\pm$ 0.002}
    \put(-370,-30){\small \textcolor{blue}{98.6\% $\pm$ 0.4}}
    \put(-360,-40){\small \textcolor{blue}{More reliable uncertainties}}
    \put(-300,60){ 2) Optimal}
    \put(-300,-10){\small 99.6\% $\pm$ 0.0}
    \put(-300,-20){\small 0.014 $\pm$ 0.004}
    \put(-300,-30){\small \textcolor{blue}{98.9\% $\pm$ 0.1}}
    \put(-220,60){ 3) Sandwich}
    \put(-220,-10){\small 99.5\% $\pm$ 0.0}
    \put(-220,-20){\small 0.360 $\pm$ 0.002}
    \put(-220,-30){\small \textcolor{red}{71.6\% $\pm$ 0.8}}
    \put(-140,60){ 4) Stack}
    \put(-140,-10){\small 99.3\% $\pm$ 0.0}
    \put(-140,-20){\small 0.276 $\pm$ 0.020}
    \put(-140,-30){\small \textcolor{red}{69.5\% $\pm$ 1.3}}
    \put(-60,60){ 5) Lopsided}
    \put(-60,-10){\small 99.2\% $\pm$ 0.1}
    \put(-60,-20){\small 0.079 $\pm$ 0.008}
    \put(-60,-30){\small \textcolor{red}{73.6\% $\pm$ 6.0}}
    \put(-162,-40){\small \textcolor{red}{Less reliable uncertainties}}
\vspace{-0.05in}
\caption{Illustration of various softmax structures, along with results when $\mathcal{D}_\text{in}$ is $K=3$ classes of MNIST, and $\mathcal{D}_\text{out}$ is Fashion MNIST. Whilst all structures achieve similar test accuracy, AUROC is largest when the softmax weight vectors are either trainable, or fixed in an optimal configuration. Fig. \ref{fig_weird_structures_pca} provides PCA visualisations of the trained networks. Mean $\pm$ 1 std. err. over three runs.}
\label{fig_counterfact_structs}
\end{center}
\end{figure}
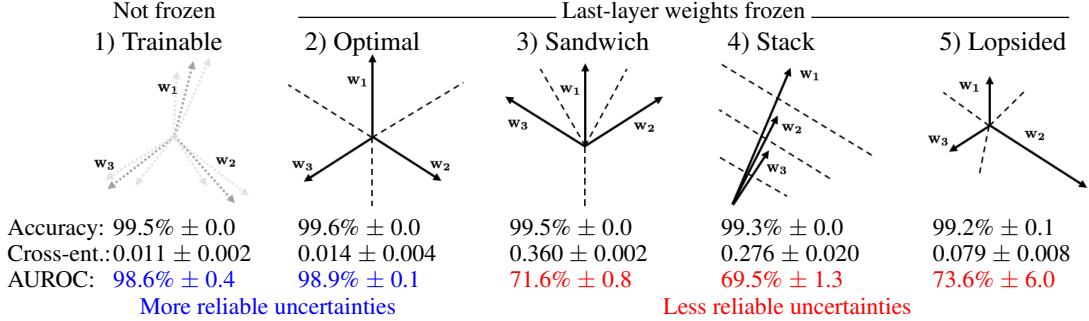

\section{Implicit Bias 2: Deep Networks Filter for Task-Specific Features}
\label{sec_misleading_intuition}

Whilst the decision boundary structure typically learnt by neural networks is helpful for OOD detection, it is not sufficient by itself. This is illustrated by fig. \ref{fig_2d}; in both tasks shown, the final-layer visualisations suggest good decision boundary structure has been learnt, yet in fig. \ref{fig_2d}b, the network maps OOD data to the low confidence region of the softmax, whilst in fig. \ref{fig_2d}a it does not. 
This section studies a second implicit bias that emerges in deep networks also important in making softmax correlate with epistemic uncertainty.


\label{sec_feats_what_feats}

\textbf{What final-layer features are learnt?}
Early-layer features of a deep network are known to be generic, while later features evolve in specialised roles to detect and separate the classes present in the training distribution \citep{Zeiler2014} -- the goal of final-layer features is to create clusters that can be linearly separated by a softmax layer. 
Information about the input that does not directly help with the classification task is thought to be compressed out \citep{Tishby2015, saxe2018}. We therefore view a deep network as a filter, optimised over a training distribution to filter out everything except distinguishing features of the classes.





\textbf{How do final-layer features respond to OOD data?} 
Convolutional filters operate on patches of activations from prior layers. The output of the convolutional filter is maximised when the pattern of patch activations matches the pattern of weights in the filter. 
Since OOD data is unlikely to contain the distinguishing features that the convolutional filters are trained to extract, OOD activations tend to be of \textit{lower magnitude}. Where some of these distinguishing features are present, it will rarely be in the same combinations seen in training -- this results in \textit{unusual patterns} of final-layer activations.


To illustrate this, we train a LeNet with $H=64$ final-layer neurons. 
Fig. \ref{fig_activation_mnsit_fashion} plots final-layer activations for examples of each training class and OOD inputs, with $\mathcal{D}_\text{in}$ and $\mathcal{D}_\text{out}$ as denoted.
Inputs of the same training class tend to produce consistent patterns of final-layer activations, while OOD inputs produce activation patterns that are of lower magnitude, and/or in unusual combinations.
Final-layer activation magnitude is quantified by, $||\z||$, and the activation `familiarity' by, $\max_i \cos \theta_{i,\z}$.

More comprehensively, fig. \ref{fig_cifar_etc} plots histograms of $||\z||$ and $\max_i \cos \theta_{i, \z}$ for inputs from the training distribution vs. OOD inputs, for the ResNet18's trained on various datasets as described later in section \ref{sec_failure_measure}. These plots show $||\z||$ and $\max_i \cos \theta_{i, \z}$ tend to reduce for OOD data.

\textbf{What do these OOD final-layer activations do to softmax confidence?} We can write the softmax using, $\w_i^\intercal \z = ||\z|| \, ||\w_i|| \cos \theta_{i,\z}$, to gain insight into the effect of $||\z||$ \& $\cos \theta_{i,\z}$ on softmax confidence,
\begin{align}
\sigma(\z)_i &= \frac{\exp{\lvert \lvert \w_i \rvert \rvert \, \lvert \lvert \z \rvert \rvert \cos \theta_{i,\z}}}{\sum_j \exp{\lvert \lvert \w_j \rvert \rvert \, \lvert \lvert \z \rvert \rvert \cos \theta_{j,\z}}}. \label{eq_softmax_mag_angle} 
\end{align}
Here, $||\z||$ plays an identical role to the temperature parameter in Platt/temperature scaling, where logits are warmed, $\sigma(\z/T)_i$ for $T>1$, to produce softmax distributions with lower $U_\text{max}$ and $U_\text{entropy}$,  \citep{Guo2017} (formally shown in proposition \ref{prop_optimal_zmag} for any decision boundary structure). The role of $\max_i \cos \theta_{i,\z}$ is less straightforward and for general softmax weight structures there is no guarantee decreasing this will result in decreased softmax confidence (proposition \ref{prop_optimal_notalways} provides an example). But crucially, if combined with an optimal decision boundary structure, lower $\max_i \cos \theta_{i,\z}$ does reduce softmax confidence -- as proved in theorem \ref{theorem_optimal_costheta_k2} for $K=2$ \& $K=3$, and theorem \ref{theorem_optimal_costheta} for large $K$. 

\begin{figure}[t!]
\begin{center}
\hspace{-0.3in}
\includegraphics[width=0.5\columnwidth]{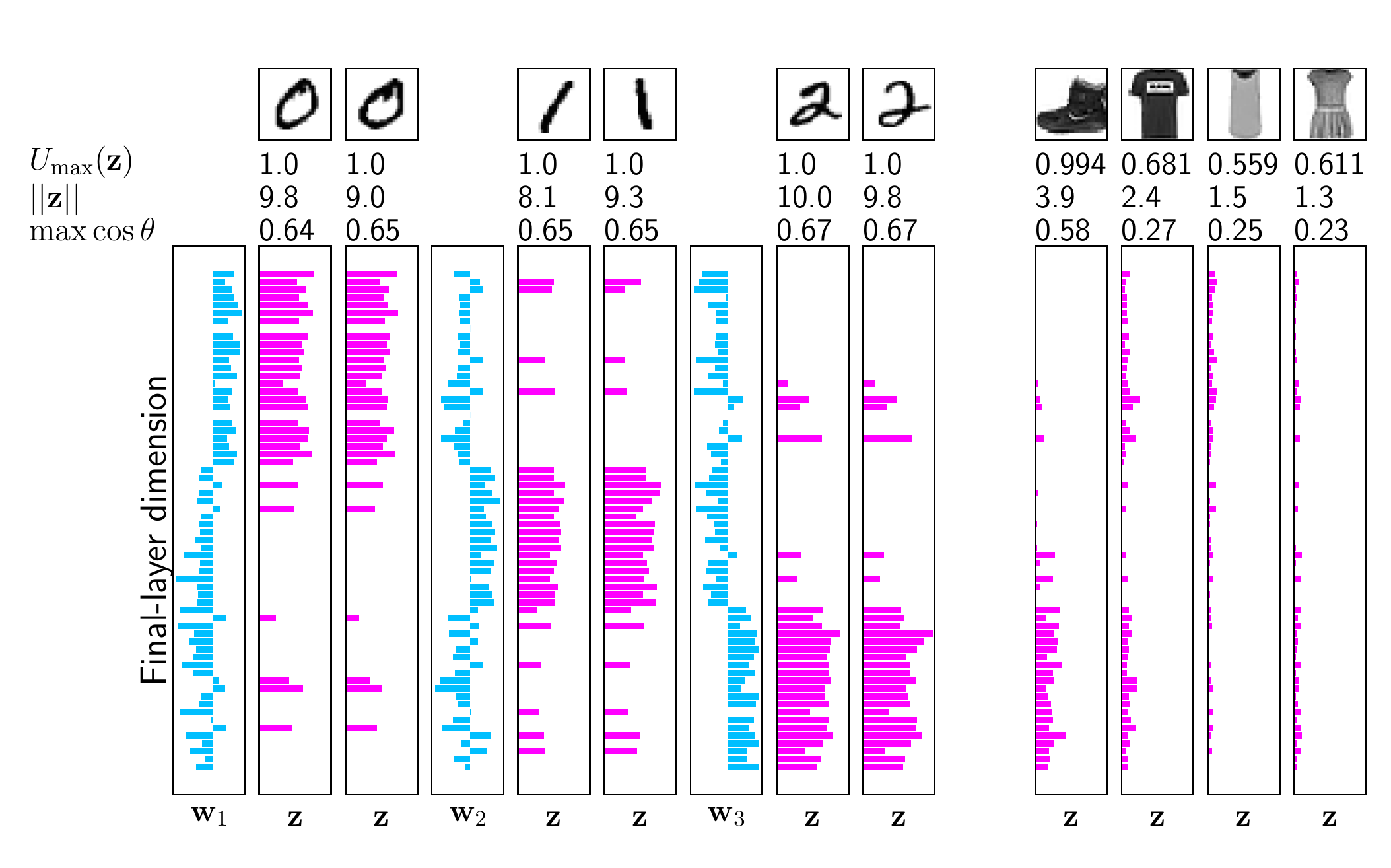}
\put(-164,120){\small $\mathcal{D}_\text{in}=$MNIST (3 classes)}
\put(-64,120){\small $\mathcal{D}_\text{out}=$ F. MNIST}
\hspace{0.2in}
\includegraphics[width=0.5\columnwidth]{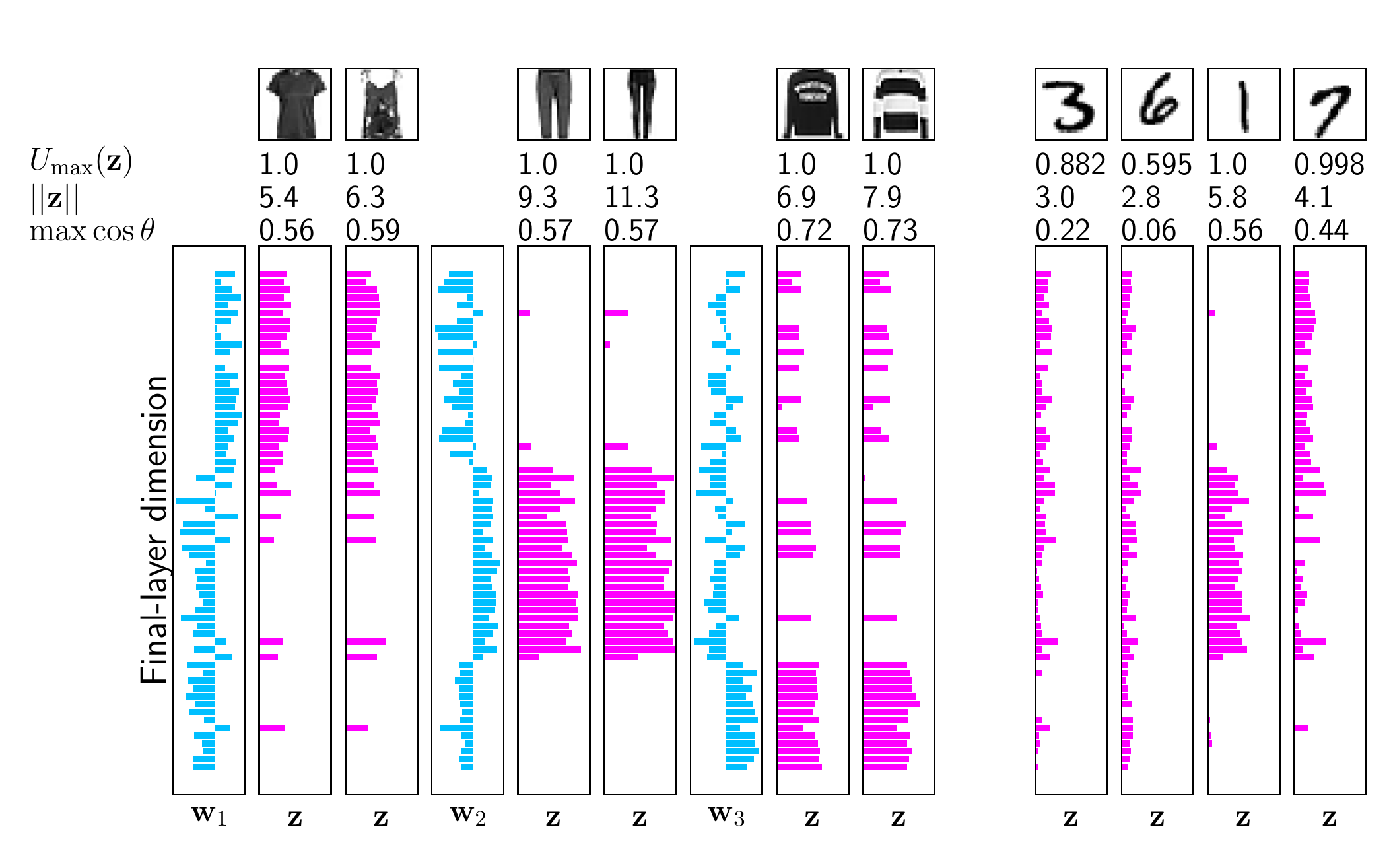}
\put(-164,120){\small $\mathcal{D}_\text{in}=$F. MNIST (3 classes)}
\put(-55,120){\small $\mathcal{D}_\text{out}=$MNIST}

\caption{Inputs from the same training class produce consistent patterns of final-layer activations, roughly aligning with their weight vectors. OOD inputs tend to produce final-layer activations of lower magnitude and/or in unusual combinations. When combined with good decision boundary structure, this often causes OOD data to fall into the valid OOD region.}
\label{fig_activation_mnsit_fashion}
\end{center}
\vskip -0.1in
\end{figure}

\subsection{Empirical Demonstration of the Filtering Effect}
\label{sec_filter_empirical}



Consider a literal interpretation of a neural network as a filter; when bombarded with \textit{any} allowable input, softmax confidence should be highest on inputs similar to that which the network has been trained on. 
Typical OOD benchmarks test this over a tiny fraction of the input space covered by some unrelated dataset. We now consider a more exhaustive test; defining a small, bounded training distribution, and in brute-force style generating OOD data sampled uniformly from the input space and examining which inputs are assigned lowest uncertainty by the softmax.



We use MNIST digits, downsampled to 9$\times$9 pixels (if smaller the digits become unrecognisable) and binarised, as a training distribution. The entire input space consists of $2^{81}$ possible patterns. We proceed to sample $2^{27} \approx 100$ million inputs uniformly from this space, and collect the most confident examples from each class. Fig. \ref{fig_noise} shows that, although noisy, these examples all contain features reminiscent of the training data, evidencing the presence of the filtering effect.
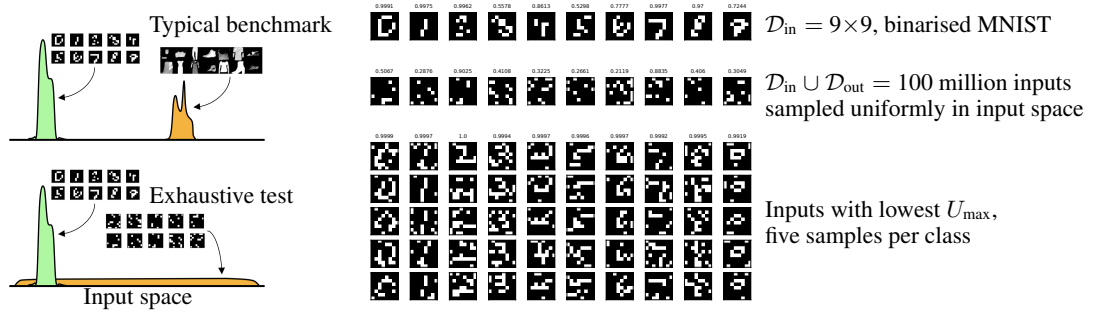
\begin{figure}[b!]
\begin{center}
\hskip -0.8in
\includegraphics[width=0.25\columnwidth]{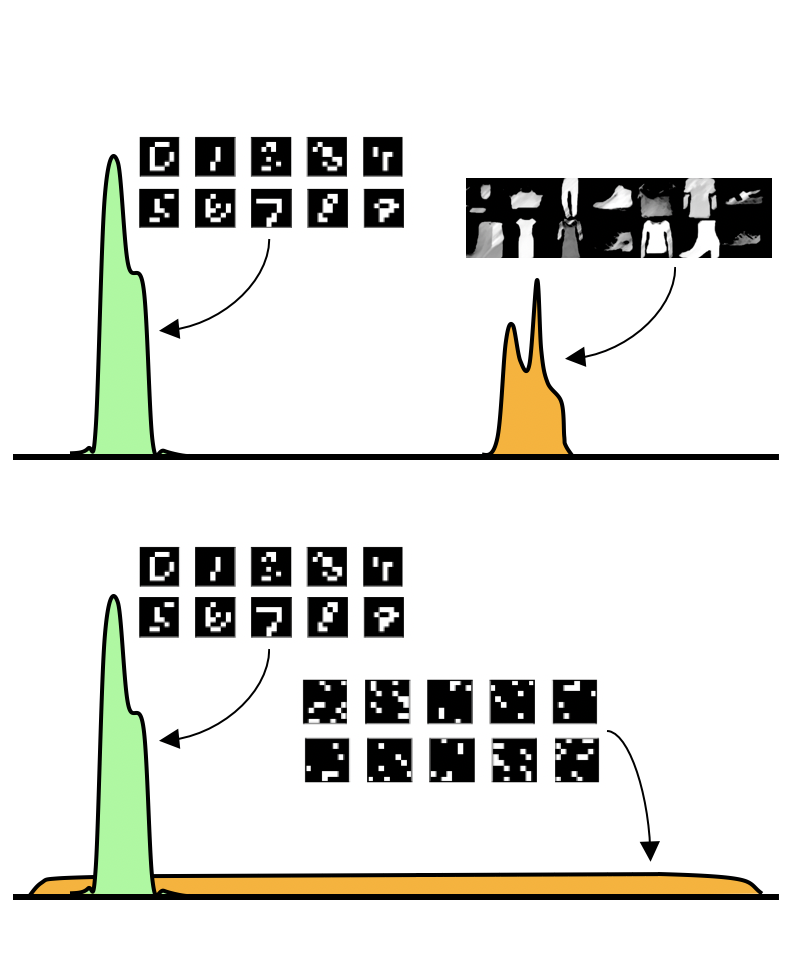}
\put(-70,2){\footnotesize Input space}
\put(-45,40){\footnotesize Exhaustive test}
\put(-45,105){\footnotesize Typical benchmark}
\hskip 0.4in
\includegraphics[width=0.4\columnwidth]{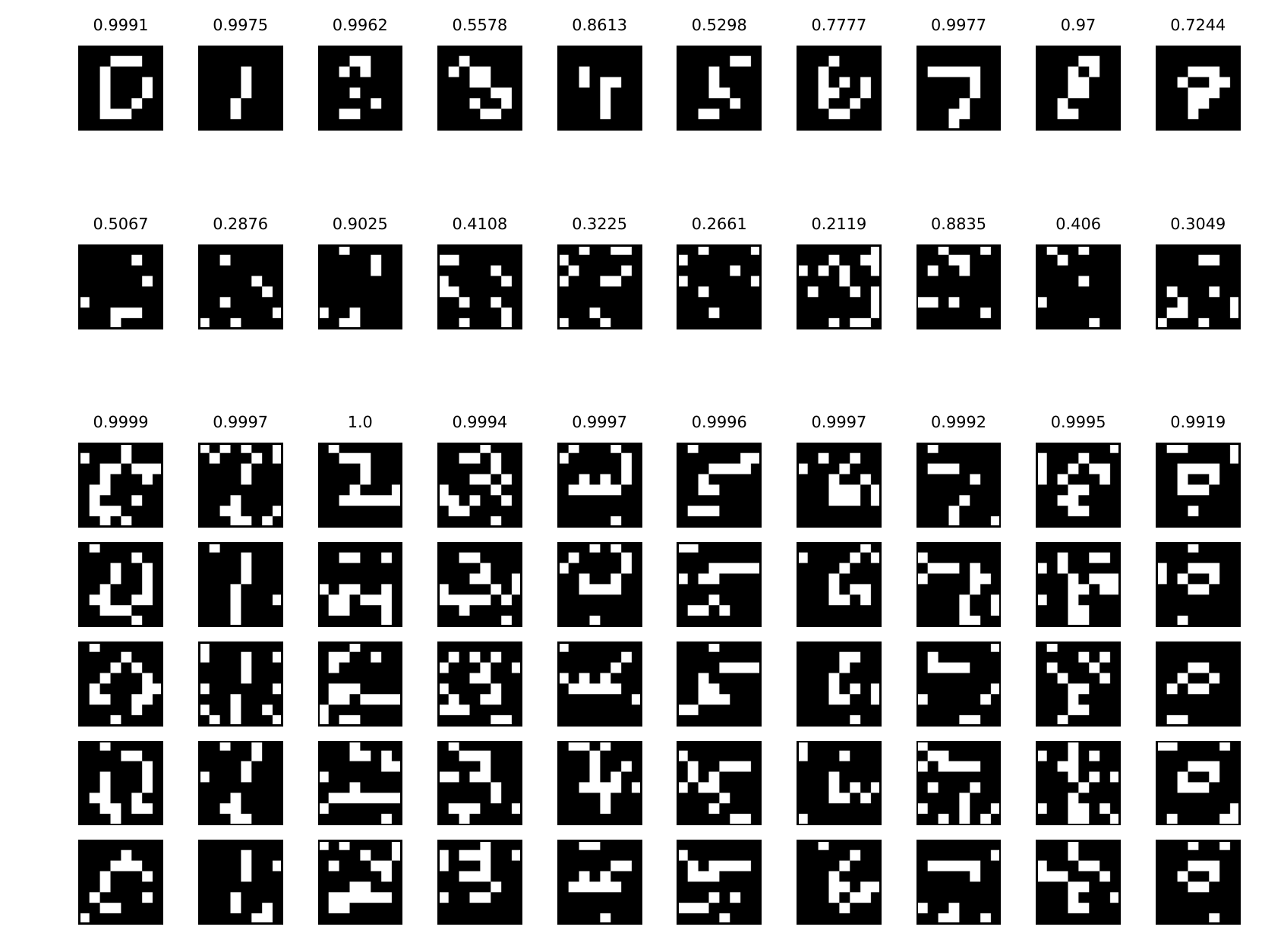}
\put(0,105){\footnotesize $\mathcal{D}_\text{in}=$ 9$\times$9, binarised MNIST}
\put(0,83){\footnotesize $\mathcal{D}_\text{in} \cup \mathcal{D}_\text{out}=$ 100 million inputs}
\put(0,73){\footnotesize sampled uniformly in input space}
\put(0,35){\footnotesize Inputs with lowest $U_\text{max}$,}
\put(0,25){\footnotesize five samples per class}
\caption{A convolutional network is trained on MNIST, then bombarded with 100 million samples drawn uniformly from the input space. Shown are the top 5 most confident samples for each class. This demonstrates that deep neural networks can filter the entire input space for task-specific features.}
\label{fig_noise}
\end{center}
\vskip -0.1in
\end{figure}

In section \ref{sec_filter_depth_exp} we find empirically that depth is important in strengthening this filtering effect.


\subsection{Misleading Intuition}
\label{sec_feats_intuition}
Consider again the low-dimensional example in fig. \ref{fig_2d}a. Such visualisations reinforce the intuition that softmax confidence cannot be trusted `far from the training data'.
Should we be concerned that softmax confidence fails in such a simple setting? 
We argue that it fails \textit{because} it's so simple.

Firstly, for the filtering effect to emerge, there must be an opportunity to do feature extraction, but in low-dimensions there is no such opportunity -- final-layer features are trivial transformations of the input space.
Secondly, many of the OOD inputs in fig. \ref{fig_2d}a can be viewed as magnified training examples, so for any training data point, $\x_i$, an OOD input can be created, $\alpha \x_i$ for $\alpha \gg 1$. Due to ReLU being a homogenous function, magnified inputs produce magnified final-layer features, $\alpha \z \approx \psi(\alpha \x)$, leading to increased softmax confidence (as observed by Hein et al. \citeyearpar{Hein}). In low-dimensions, magnification is \textit{the most obvious way} to create outliers. But this type of OOD input does not translate to high dimensions (should MNIST digits with magnified pixel intensity be labelled OOD?) nor does it obviously pose a problem in typical deep learning domains such as images, text, and audio which have bounded input ranges.

\textbf{A new mental model.}
If low-dimensional examples do not accurately capture the behaviour of softmax confidence in deep networks on complex tasks, how \textit{should} we think about it? We now propose a single interpretable equation summarising this paper's findings. 


We take the softmax expressed in terms of magnitudes and angles from eq. \ref{eq_softmax_mag_angle} and make two strong assumptions, supported by this paper's findings, to give $U_\text{max mental}$. 1) The network's decision boundary structure is optimal (section \ref{sec_optimal_decision_boundary}). 2) $\cos\theta_{i,\z} = \frac{-1}{K-1}$ for all angles except the maximum class. 
Proposition \ref{prop_mentalmodel} provides derivation and discussion on assumptions. Note we trade off accuracy for interpretability with these assumptions. 
\begin{align}
U_\text{max mental}(\z) &\coloneqq
\frac{-1}
{1 + (K-1)\exp{- ||\z||( \frac{1}{K-1} + \max \cos \theta_{i,\z} ) }}.
\end{align}
Here, $||\z|| \geq 0 $ represents the \textit{strength} of distinguishing training features, and $\max \cos \theta_{i,\z} \in [-1,1]$ represents the \textit{familiarity} of the combination of final-layer features relative to the training data. This model captures that uncertainty increases with a decrease in either of these quantities. 

\section{Failure Causes and their Prevalence}
\label{sec_failure_measure}

We have so far described two implicit biases in deep networks that encourage OOD data to fall into the softmax's valid OOD region. Softmax confidence is fallible, and we turn our attention to understanding why this mechanism can fail. 
Figure \ref{fig_ankle_boots} illustrates a concrete example of failure when trained on specific classes of Fashion MNIST -- the network confuses ankle boots for digits.
We now define several causes for softmax failure.  Some of these issues have been previously noted \citep{Liu2020, VanAmersfoort2020, Mukhoti2021}.

\textbf{Cause 1) Softmax saturation:} 
$\max_i \sigma(\z)_i$ is typically very high (e.g. fig. \ref{fig_activation_mnsit_fashion}). Sometimes, the softmax saturates for a proportion of both training and OOD data and $U_\text{max}$ is exactly -1.00000; this leads to an inability to create an ordering between samples, and OOD detection deteriorates. Simple fixes include rescaling logits; when problematic we use, $U_\text{cool}(\z) \coloneqq U_\text{entropy}(0.1\z)$.

\textbf{Cause 2a) Softmax extrapolations:} OOD data may be mapped to regions of the softmax of higher confidence than the training distribution -- the extrapolation regions. 

\textbf{Cause 2b) Conflation with aleatoric uncertainty:} Where class distributions within the training data overlap, a network should learn to explicitly map inputs to the appropriate low confidence region of the softmax. This leads to a proportion of the training data legitimately having low softmax confidence.
Note that the effect of this is to shrink the valid OOD region, equivalently enlarging the softmax extrapolation region, hence we combine this with cause 2a.

\textbf{Cause 3) Feature overlap:} Neural networks are not bijective and there is a risk that a training input and OOD input both map to the same point in the final-layer. Any uncertainty estimator accessing only these final activations will be unable to distinguish these types of input.



%
\textbf{Which of these causes are primarily responsible for overconfident predictions on OOD data?}
We first estimate this on typical benchmark tasks 
-- a ResNet18 is trained on one of five datasets (MNIST, Fashion MNIST, SVHN, Cifar10, Cifar100) ($\mathcal{D}_\text{in}$) with OOD data taken from the other four datasets ($\mathcal{D}_\text{out}$). (Note we convert MNIST and Fashion MNIST to 32$\times$32 RGB format so all datasets are exchangeable.) We secondly test the same protocol, but using an EfficientNetB1 pre-trained on ImageNet, and fine-tuned on five datasets, in each case with the other four used as $\mathcal{D}_\text{out}$. Full details are given in section \ref{sec_app_exp_standard_bench} \& \ref{sec_app_exp_pretrain_finetune}.

A perfect uncertainty estimator should achieve 100\% AUROC, and 
we assign responsibility for not reaching this score in the following way. 
Cause 1 is remedied by using $U_\text{cool}$ rather than $U_\text{entropy}$. Hence, AUROC using $U_\text{cool}$ minus AUROC using $U_\text{entropy}$ gives the responsibility of cause 1. 
Cause 2 is estimated by finding the improvement realised by using $U_\text{density}$ rather than $U_\text{cool}$, visually:

 \hspace{3.3in}
 \includegraphics[width=0.08\columnwidth]{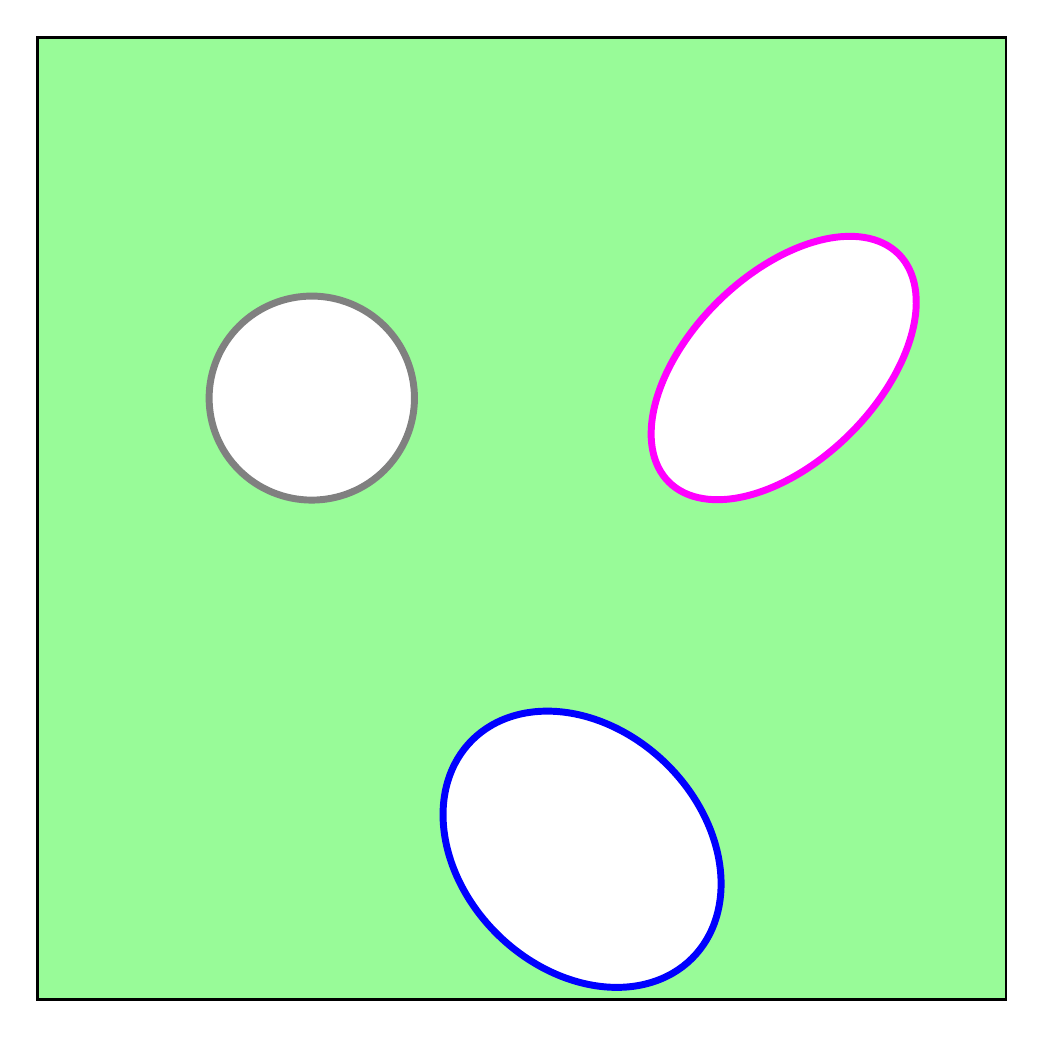}
 \put(-260,15){$\mathcal{R}_\text{density} \bigcap \mathcal{R}_\text{entropy}' = \text{Softmax extrapolation region}$}
 \put(3,12){$-$}
 \hspace{0.15in}
\includegraphics[width=0.08\columnwidth]{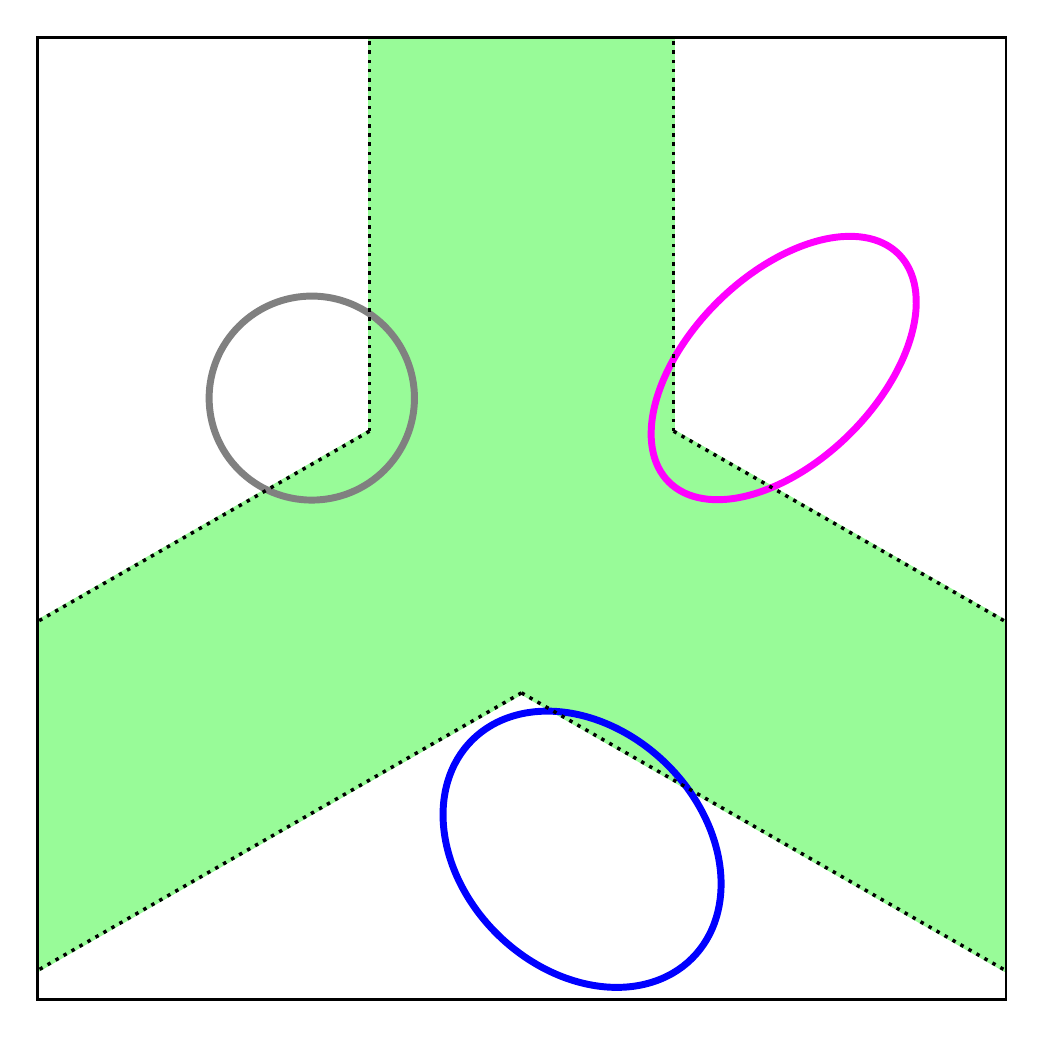}
  \put(5,12){$=$}
 \hspace{0.2in}
\includegraphics[width=0.08\columnwidth]{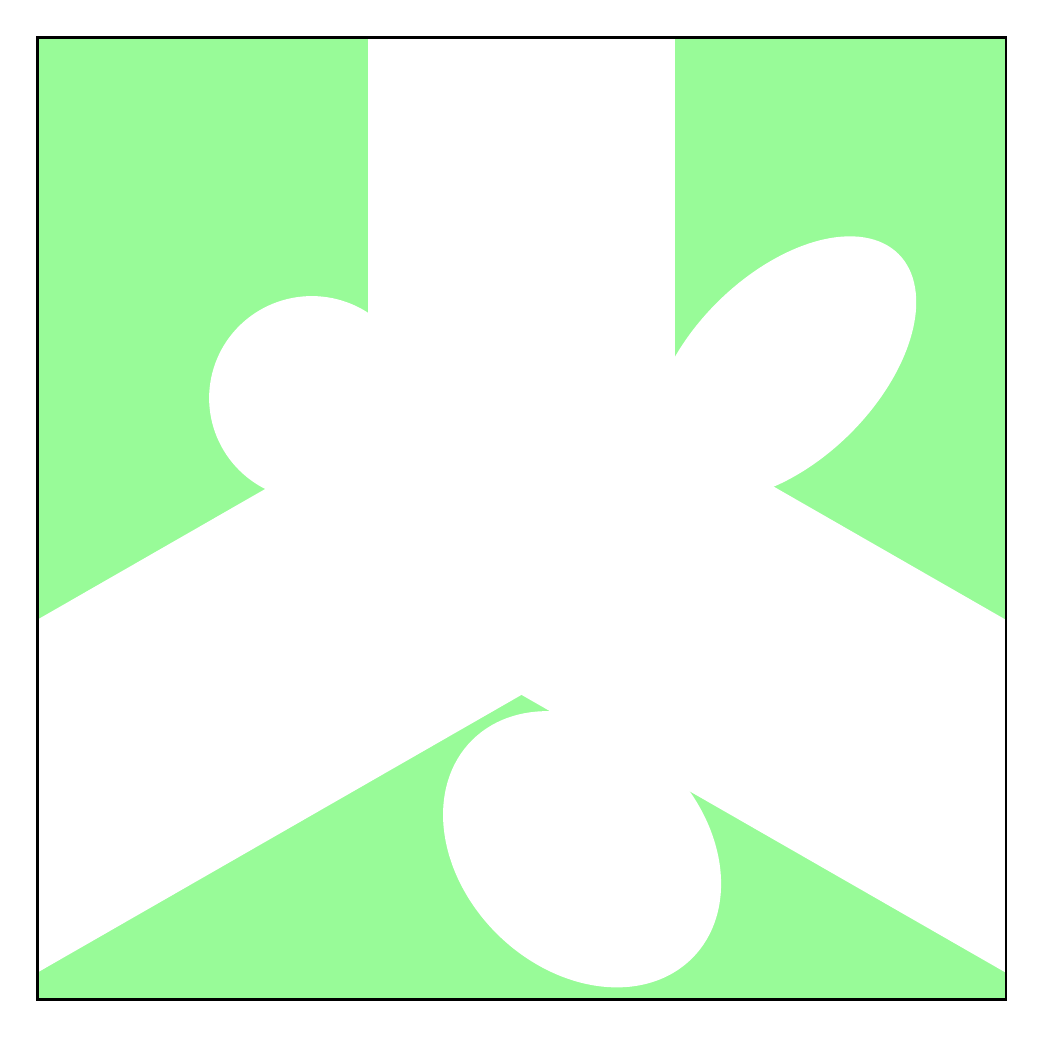}

Any inaccuracy remaining using $U_\text{density}$ is attributed to cause 3 -- this assumes the density estimator is a good approximation of the true density, $\hat{q}(\z) \approx p_\text{in}(\z)$. We spent considerable effort tuning the GMM as far as possible, finding it superior to KDE (details in section \ref{sec_app_exp_standard_bench}).

Table \ref{table_AUROC_bench} displays results. Our numbers in standard training are in line with prior work (table \ref{table_reported_results}). 
$U_\text{entropy}$ is consistently preferable to $U_\text{max}$. Saturation is avoided for ResNet18, and only minor for EfficientNet. Importantly, we attribute 4.2\% of AUROC drop to softmax extrapolations, and a 7.4\% drop to feature overlap -- we believe this goes against common intuition in the field.
AUROC is high for all estimators in the pre-trained case, and interestingly, $U_\text{density}$ provides a close to perfect score across all datasets, largely eliminating feature overlap. 

\begin{table}[h!]
  \caption{AUROC scores for OOD detection, mean $\pm$ 1 standard error. For each $\mathcal{D}_\text{in}$ dataset, $\mathcal{D}_\text{out}$ comprises the other four datasets. Further breakdown in section \ref{sec_app_exp_standard_bench} \& \ref{sec_app_exp_pretrain_finetune}. Note that cause 1 is the dominant cause of failure in standard training, while pre-trained networks largely avoid all causes.}
  \vspace{-0.1in}
  \label{table_AUROC_bench}
  \begin{center}
  \resizebox{\textwidth}{!}{
  \begin{tabular}{l c ccc cc c c c}
    \toprule
    & & & & & & \multicolumn{3}{c}{ Estimated responsibility of each cause } \\
    & & \multicolumn{1}{c}{$A$} & \multicolumn{1}{c}{$B$} & \multicolumn{1}{c}{$C$} & \multicolumn{1}{c}{$D$} & \multicolumn{1}{c}{$C - B$} & \multicolumn{1}{c}{$D-C$} & \multicolumn{1}{c}{$1.0 - D$}\\
    $\mathcal{D}_\text{in}$ & Accuracy & $U_\text{max}$ & $U_\text{entropy}$  & $U_\text{cool}$ & $U_\text{density}$ & Cause 1 & Cause 2a \& 2b & Cause 3 \\
    \midrule
	\multicolumn{3}{l}{\textbf{Standard training}, ResNet18, five runs} \\

mnist  &  99.6\% $\pm$ 0.0  &  96.3\% $\pm$ 1.1  &  96.3\% $\pm$ 1.1  &  96.3\% $\pm$ 1.1  &  99.5\% $\pm$ 0.1  &  0.0\%  &  3.2\%  &  0.5\%  &  \\
f.mnist  &  94.5\% $\pm$ 0.1  &  87.4\% $\pm$ 1.1  &  88.7\% $\pm$ 1.1  &  88.7\% $\pm$ 1.1  &  98.5\% $\pm$ 0.3  &  0.0\%  &  9.7\%  &  1.5\%  &  \\
svhn  &  96.4\% $\pm$ 0.1  &  90.7\% $\pm$ 0.6  &  90.9\% $\pm$ 0.6  &  90.9\% $\pm$ 0.6  &  93.2\% $\pm$ 0.2  &  0.0\%  &  2.3\%  &  6.8\%  &  \\
c10  &  93.5\% $\pm$ 0.1  &  85.4\% $\pm$ 1.1  &  85.7\% $\pm$ 1.1  &  85.7\% $\pm$ 1.1  &  89.3\% $\pm$ 0.4  &  0.0\%  &  3.6\%  &  0.7\%  &  \\
c100  &  72.0\% $\pm$ 0.4  &  78.6\% $\pm$ 0.5  &  80.6\% $\pm$ 0.5  &  80.6\% $\pm$ 0.5  &  82.5\% $\pm$ 0.6  &  0.0\%  &  2.0\%  &  17.5\%  &  \\
\textbf{Mean} & - & 87.7\% & 88.4\% & 88.4\% & 92.6\% & 0.0\% & \textbf{\textcolor{blue}{4.2\%}} & \textbf{\textcolor{blue}{7.4\%} }\\


	\midrule
	\multicolumn{3}{l}{\textbf{Pre-trained/fine-tuned}, EfficientNetB1, three runs} \\
	

Satellite & 97.7\% $\pm$0.1 & 99.7\% $\pm$0.0 & 99.9\% $\pm$0.0 & 99.9\% $\pm$0.0 & 99.9\% $\pm$0.0 & 0.0\% & 0.0\% & 0.1\% & \\
Cancer & 95.2\% $\pm$0.1 & 98.4\% $\pm$0.0 & 99.3\% $\pm$0.0 & 99.8\% $\pm$0.0 & 100.0\% $\pm$0.0 & 0.5\% & 0.2\% & 0.0\% & \\
Pets & 92.5\% $\pm$0.1 & 100.0\% $\pm$0.0 & 100.0\% $\pm$0.0 & 100.0\% $\pm$0.0 & 100.0\% $\pm$0.0 & 0.0\% & 0.0\% & 0.0\%& \\
Flowers & 94.7\%$\pm$0.0 & 99.9\%$\pm$0.0 & 99.9\%$\pm$0.0 & 99.9\%$\pm$0.0 & 100.0\%$\pm$0.0 & 0.0\% & 0.1\% & 0.0\% & \\
Beans & 94.1\%$\pm$0.2 & 96.8\%$\pm$0.1 & 98.2\%$\pm$0.1 & 99.0\%$\pm$0.0 & 100.0\%$\pm$0.0 & 0.8\% & 1.0\% & 0.0\% & \\
\textbf{Mean} & - & 99.0\% & 99.5\% & 99.7\% & 100.0\% & 0.3\% & \textbf{\textcolor{blue}{0.2\%}} & \textbf{\textcolor{blue}{0.0\%}} \\


    \bottomrule
  \end{tabular}
  }
\end{center}

\end{table}






\section{Discussion \& Conclusion}
\label{sec_discussion_conclusion}






\textbf{Summary.} This paper sought to understand the value of softmax confidence as a proxy for epistemic uncertainty in the detection of OOD data. It uncovered two implicit biases baked into unmodified networks which together produce a rudimentary mechanism for uncertainty estimation. We argued that common intuition around softmax confidence is misleading.
Carefully designed diagnostic experiments on standard benchmark tasks suggested final-layer feature overlap is more responsible for failures than softmax extrapolations.


\textbf{Limitations \& impact.} This paper has contributed to an understanding of deep learning's most accessible uncertainty estimates; we expect this to be helpful for practitioners and researchers. Whilst our paper has explained why using softmax confidence for uncertainty estimation might have a sounder basis than widely believed, we remind readers of several things. Our results have said nothing about calibration of the softmax probabilities, nor it's effectiveness against adversarial examples, nor about uncertainty in regression tasks.
Softmax confidence remains an imperfect measure of uncertainty, and caution should be applied when used in real-world applications. 
Our study has focused on one specific type of OOD data commonly used in benchmarking OOD detection methods -- $\mathcal{D}_\text{in}$ is the input distribution of the task of interest (with low aleatoric uncertainty), and $\mathcal{D}_\text{out}$ is an arbitrarily selected, unrelated, dataset. Softmax confidence may be less effective on other types of OOD data.

\textbf{Conclusions.} We close with three thoughts. 
1) The ability to `capture epistemic uncertainty' is often presented as a binary property, whilst our work described a mechanism providing partial capture -- perhaps this property should be interpreted in a softer way.
2) It may be revealing to study ensembling, Bayesian neural networks, and MC Dropout, in the context of strengthening  the implicit biases this paper has described, rather than as adding fundamental properties by themselves. 
3) Addressing final-layer feature overlap should be a priority for the field. There are at least two approaches;  modifying networks to encourage bijective behaviour, e.g. \citep{Liu2020, VanAmersfoort2020}, or learning more diverse representations, e.g. \citep{Hendrycks2019a, Hendrycks2019, Hendrycks2020}. Given the trend across machine learning for large-scale pre-training, we see the second route as particularly promising.





\bibliography{library.bib}

\appendix

\newpage
\onecolumn
\section{Proofs}
\label{sec_app_proofs2}

This section contains all mathematical results and proofs, organised as follows.
\begin{itemize}
\item \textbf{Section \ref{sec_vec_fields}.} Derives uncertainty vector fields for all uncertainty estimators. 
\item \textbf{Section \ref{sec_app_proof_validoodregions}.} Derives valid OOD regions for each estimator under various assumptions:
  \begin{itemize}
    \item Theorem \ref{theorem_valid_ood_region_2class}: $U_\text{max}$ and $U_\text{entropy}$ for $K=2$ (assume weight vectors, $\w_1 = -\w_2$).
    \item Corollary \ref{sec_corollary_2class_gaussian}: $U_\text{max}$ and $U_\text{entropy}$ for $K=2$ (Gaussian and linearly separable assumption).
    \item Proposition \ref{prop_valid_ood_region_density}: $U_\text{density}$ for $K\geq2$ (Gaussian assumption).
    \item Definition \ref{def_linearapprox_region}:
    Describes a linear approximation to the exact valid OOD region for general $K \geq 2$.
    \item Proposition \ref{prop_linearapprox_equalfor_k2}:
    Shows the linear approximation is equal to the exact valid OOD region for $K=2$ (assume weight vectors, $\w_1 = -\w_2$).
    \item Theorem \ref{theorem_exact_linear_convergence}:
    For $K\geq3$, shows the exact valid OOD region converges to the linear approximation at large magnitudes at nearly all points (except close to intersections between decision hyperplanes).
    \item Corollary \ref{corollary_linear_region_is_subset_of_exact}: 
    For $K\geq3$, shows that the linear approximate valid OOD region is a subset of the exact valid OOD region.
  \end{itemize}
\item \textbf{Section \ref{sec_app_proof_optimalstructure}.} Considers the optimal decision boundary structure and related results.
  \begin{itemize}
    \item Theorem \ref{theorem_hyperplane_structure_optimal}: Shows that the optimal structure minimises a regularised cross-entropy loss. Assumes class clusters follow a Gaussian distribution, $H \geq K-1$, and $||\pmb{\mu}_i|| = c_2 \; \forall i$.
    \item Proposition \ref{prop_optimal_zmag}: For any structure, $U_\text{max}$ increases with a decrease in $||\z||$.
    \item Theorem \ref{theorem_optimal_costheta_k2}: Shows a decrease in $\max_i \cos \theta_{\z,i}$ always results in an increase in $U_\text{max}$ for $K=2$ and $K=3$.
    \item Theorem \ref{theorem_optimal_costheta}: Shows a decrease in $\max_i \cos \theta_{\z,i}$ always results in an increase in $U_\text{max}$ for large $K$ (assumes that $\frac{-1}{K-1} \approx 0$).
    \item Corollary \ref{corollary_optimal_maximises}: Shows that the optimal structure maximises the linear approximate valid OOD region for $U_\text{max}$. Assumes that $U_\text{max}$ increases for a decrease in $\max_i \cos \theta_{\z,i}$ for all $K \geq 2$, and also that $||\pmb{\mu}_i|| = c_2 \; \forall i$.
    \item Proposition \ref{prop_optimal_notalways}: Shows that a decrease in $\max_i \cos \theta_{\z,i}$ does not always produce an increase in $U_\text{max}$ for structures that are not optimal.
  \end{itemize}
\item \textbf{Section \ref{sec_app_proof_mentalmodel}.} Derives the proposed mental model $U_\text{max mental}$ (proposition \ref{prop_mentalmodel}) under two strong assumptions, which are discussed.
\end{itemize}

\newpage
\subsection{Vector Fields}
\label{sec_vec_fields}

\begin{proposition}{}
\label{theorem_vector_fields}
The following equations calculate the direction of maximum uncertainty given any point, $\z$, in the final-hidden layer space for each uncertainty estimator.  
\begin{align}
        \frac{\partial U_\text{max}(\z)}{\partial \z} &=  \sigma(\z)_i \sum_{ j=1}^K \sigma(\z)_j (\w_j - \w_i)  \;\; \text{where, }  i=\text{argmax}_i\sigma(\z)_i \label{eq_vector_soft_max}  \\
        %
        \frac{\partial U_\text{entropy}(\z)}{\partial \z} 
    &= 
    \sum_{i=1}^K ( \log \sigma(\z)_i  +1  ) 
    \sigma(\z)_i \sum_{ j=1}^K \sigma(\z)_j (\w_j - \w_i)
    \label{eq_vector_soft_ent} \\
  \frac{\partial U_\text{density}(\z)}{\partial \z} 
   & \propto 
     \frac{1}{\hat{q}(\z)} \sum_{i=1}^K \pi_i \exp\bigg[-\frac{1}{2}(\z-\pmb{\mu}_i)^\intercal\mathbf{\Sigma}^{-1}_i(\z-\pmb{\mu}_i)\bigg]\mathbf{\Sigma}^{-1}_i(\z-\pmb{\mu}_i) \label{eq_vector_density}
\end{align}
\end{proposition}

\begin{proof}
These are derived through application of the chain rule, applied to the uncertainty estimators defined in the main text.
\end{proof}

\begin{remark}
Fig. \ref{fig_vecfield} visualises the equations from proposition \ref{theorem_vector_fields}.
${\partial U_\text{max}(\z)}/{\partial \z}$ is a linear combination of all weight vectors, which is the root cause of softmax confidence being inexpressive, with $-\w_i$ its primary direction.
${\partial U_\text{entropy}(\z)}/{\partial \z}$ has similar limitations, producing a smoothed version of ${\partial U_\text{max}(\z)}/{\partial \z}$.
Meanwhile, ${\partial U_\text{density}(\z)}/{\partial \z}$, always points away from the data clusters, $\z-\pmb{\mu}_i$, weighted by the distance from each component.

\begin{figure}[h!]
\begin{center}
\includegraphics[width=0.3\columnwidth]{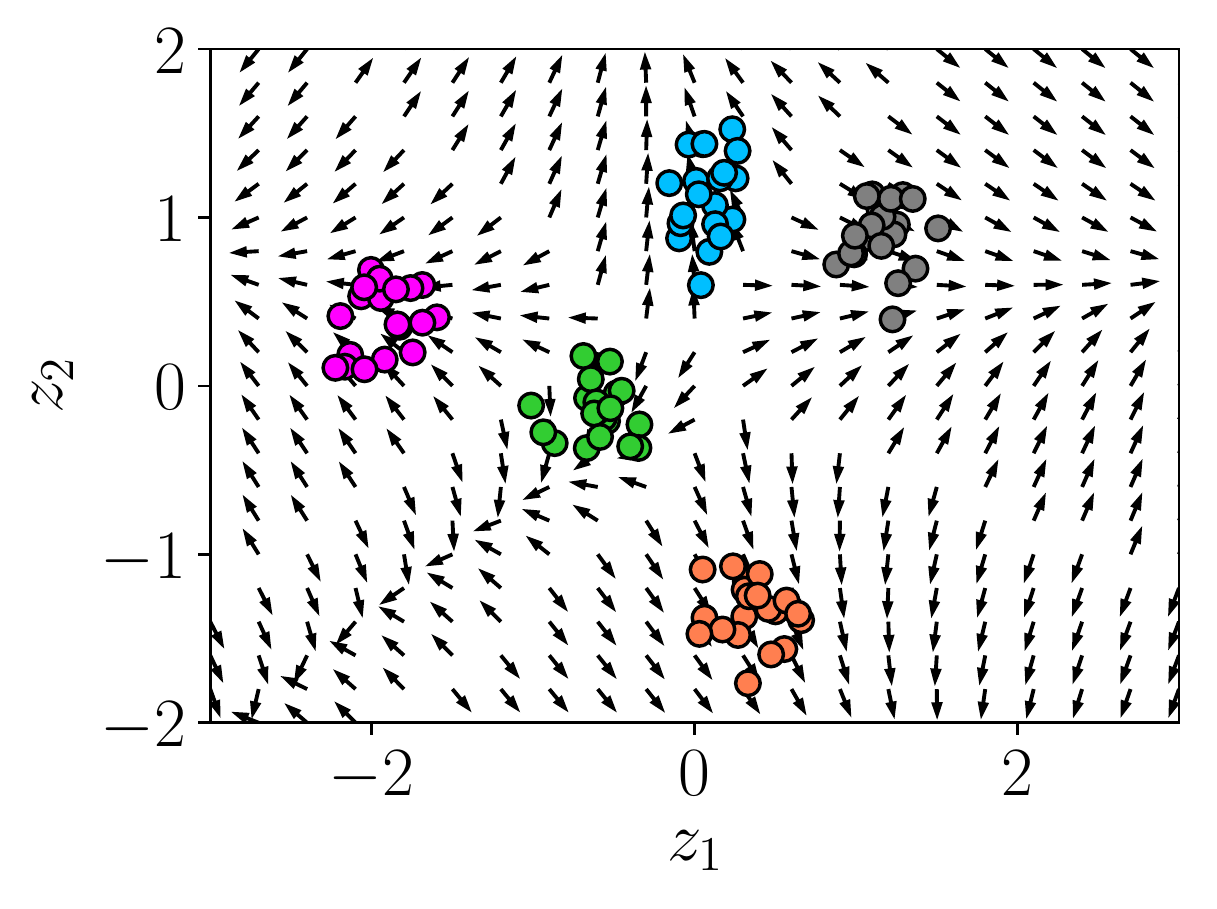}
\put(-70,-5){\small Softmax, $U_\text{max}$}
\includegraphics[width=0.3\columnwidth]{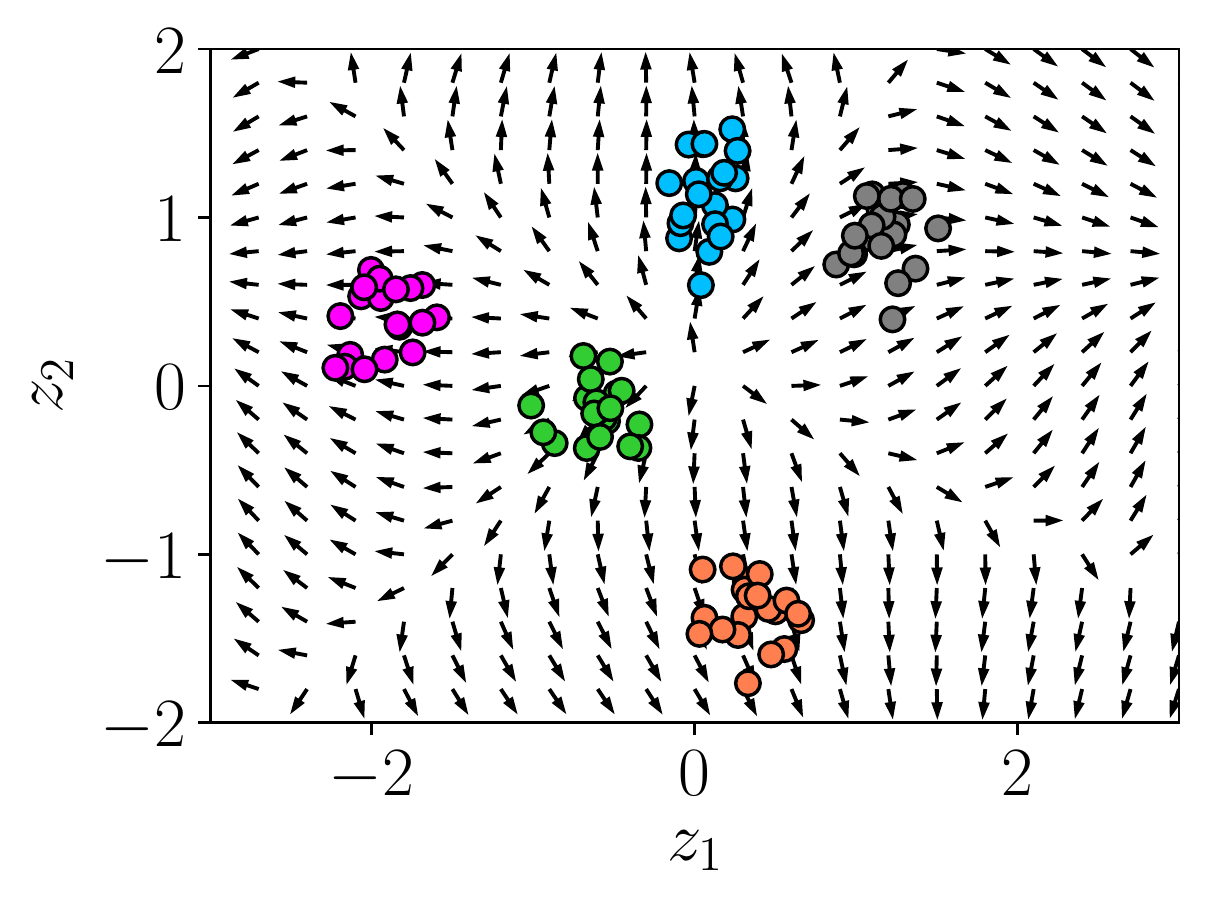}
\put(-70,-5){\small Softmax, $U_\text{entropy}$}
\includegraphics[width=0.3\columnwidth]{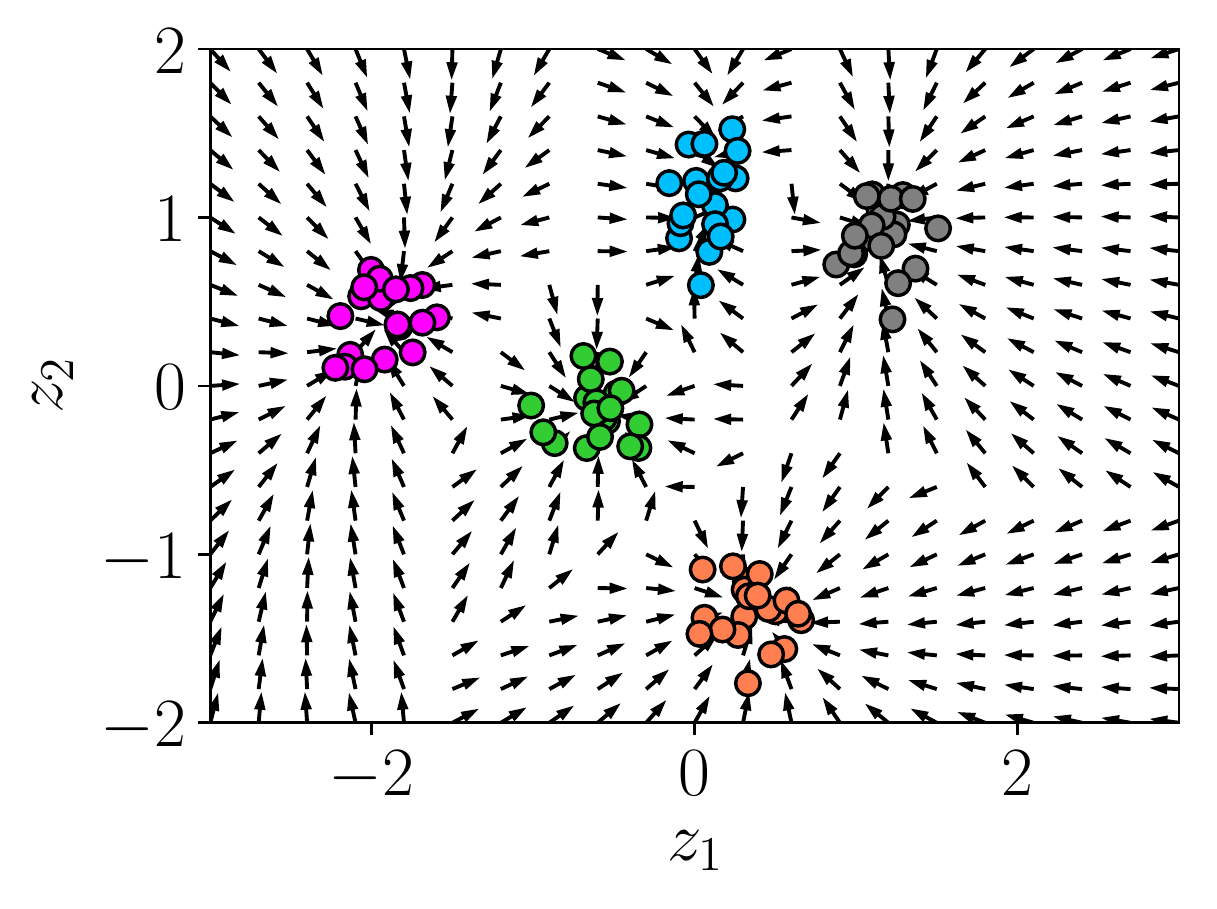}
\put(-70,-5){\small Density, $U_\text{density}$}
\caption{Vector fields showing direction of maximum certainty (magnitude ignored).}
\label{fig_vecfield}
\end{center}
\end{figure}
\end{remark}

\newpage
\subsection{Valid OOD Regions}
\label{sec_app_proof_validoodregions}
\begin{theorem}{}
\label{theorem_valid_ood_region_2class}
Consider $K=2$ classes, 
and softmax weight vectors, $\w_1 = -\w_2$. The valid OOD region (def. \ref{def_valid_ood_region}) for $U_\text{max}$ and $U_\text{entropy}$ is the region between two hyperplanes, each parallel to the decision boundary,
\begin{align}
\mathcal{R}_{K=2}: \{\z \in \mathbb{R^H} | \w_1^\intercal\left(\z_0 -\alpha \w_1\right) < \w_1^\intercal \z  < \w_1^\intercal\left(\z_0 +\alpha \w_1\right)\},
\end{align}
where $\z_0$ is some point on the decision boundary. The offset, $\alpha>0$, is determined by solving,
\begin{align}
    \pi_1 p_\text{in}(\z \notin \mathcal{R}_{K=2} |\y=1) + \pi_2 p_\text{in}(\z \notin \mathcal{R}_{K=2} |\y=2) = 1-\epsilon.
\end{align}

\end{theorem}

\begin{remark}
Note that by using a softmax with two weight vectors for two classes, we have an overparameterised model, and it seems reasonable to assume that the optimised weight vectors will be, $\w_1 \approx -\w_2$ (a softmax with this constraint enforced is equivalent to a scaled sigmoid). Note also that the weight vectors $\w_1$ and $\w_2$ are perpendicular to the decision boundary hyperplane \citep{Bishop2013}.
\end{remark}

\begin{proof}
We first consider $U_\text{max}$.
Begin observing that for the case of two classes, $\sigma(\z)_1 = 1-\sigma(\z)_2$. If, $\w_1=-\w_2$, and if $1=\text{argmax}_i\sigma(\z)_i$, eq. \ref{eq_vector_soft_max} reduces to,
\begin{align}
\frac{\partial U_\text{max}(\z)}{\partial \z}      &= \sigma(\z)_1
    \left[ (1-\sigma(\z)_1)(\w_2-\w_1) \right] \\
    &=2\sigma(\z)_1(1-\sigma(\z)_1)(-\w_1).
\end{align}
This shows that for the region where $\sigma(\z)_1 > \sigma(\z)_2$, uncertainty always increases in the direction of $-\w_1$ (toward the decision boundary). The opposite direction, $-\w_2$, is true for $\sigma(\z)_2 > \sigma(\z)_1$. 

Observe also that $U_\text{max}$ is symmetric about the decision boundary (lemma \ref{lemma_1_symmetric_softmax}). 

Consider two hyperplanes, parallel to the decision surface, and offset by some perpendicular vector $\alpha \w_1$ and $-\alpha \w_1$, defined by the equations,
\begin{align}
    \w_1^\intercal \z &=  \w_1^\intercal (\z_0 +\alpha \w_1) \\
    \w_1^\intercal \z &= \w_1^\intercal (\z_0 -\alpha \w_1) 
\end{align}
where $\z_0$ describes some point on the decision boundary, and $\alpha \in \mathbb{R}_+$.

Define the region enclosed by these two planes as $\mathcal{R}_{K=2}$, 
\begin{align}
\label{eq_hyperplane_offset_k2}
\mathcal{R}_{K=2}: \{\z \in \mathbb{R^H} | \w_1^\intercal\left(\z_0 -\alpha \w_1\right) < \w_1^\intercal \z  < \w_1^\intercal\left(\z_0 +\alpha \w_1\right)\}.
\end{align}
Due to the way we have designed this volume,
\begin{align}
    U_\text{max}(\z_i) > U_\text{max}(\z_j) \;\;\;\; \forall \z_i \in \mathcal{R}_{K=2}, \; \forall \z_j \notin \mathcal{R}_{K=2}.
\end{align}
In words; any point inside $\mathcal{R}_{K=2}$ will be more uncertain than any point outside of it. We can use this observation to define the valid OOD region for $U_\text{max}$, by choosing $\alpha$ so that the proportion of the training distribution outside the region, $p_\text{in}(\z \notin \mathcal{R}_{K=2})$, and hence with lower uncertainty than any point inside the region, is $1-\epsilon$. 

Since $p_\text{in}(\z \notin \mathcal{R}_{K=2})$ will always decrease with increasing $\alpha$, this creates a requirement finding $\alpha$ so that,
\begin{align}
    p_\text{in}(\z \notin \mathcal{R}_{K=2}) = \pi_1 p_\text{in}(\z \notin \mathcal{R}_{K=2} |\y=1) + \pi_2 p_\text{in}(\z \notin \mathcal{R}_{K=2} |\y=2) = 1-\epsilon
    \label{eq_probs_k2}
\end{align}

This proof can be extended to $U_\text{entropy}$ via lemma \ref{lemma_2_soft_ent_valid}.

\end{proof}

\begin{lemma}{}
\label{lemma_1_symmetric_softmax}
Assume two classes, and, $\w_1=-\w_2$. The softmax is symmetric about the decision boundary, so for some point on the decision boundary, $\z_0$, such that $\sigma(\z_0)_1 = \sigma(\z_0)_2 = 0.5$, it holds that $\sigma(\z_0+\alpha \w_1)_1 = \sigma(\z_0-\alpha \w_1)_2$, for any $\alpha \in \mathbb{R}$.
\end{lemma}
\begin{proof} Simply substitute $\z_0+\alpha \w_1$ and $\z_0-\alpha \w_1$ into eq. \ref{eq_softmax} to verify.
\end{proof}

\begin{lemma}{}
\label{lemma_2_soft_ent_valid}
Theorem \ref{theorem_valid_ood_region_2class} also holds for $U_\text{entropy}$.
\end{lemma}
\begin{proof} The property required of $U_\text{max}$ to prove theorem \ref{theorem_valid_ood_region_2class}, was that when $\sigma(\z)_1 > \sigma(\z)_2$, uncertainty always decreases in the direction of $-\w_1$. From eq. \ref{eq_vector_soft_max} and \ref{eq_vector_soft_ent}, it's straightforward to see that provided $\w_1=-\w_2$ (as it is by assumption), this also holds for $U_\text{entropy}$.

\end{proof}

\begin{definition}{}
\label{def_linearly_separable}
We relax the definition of `linearly separable' slightly. Generally two sets of points, $\mathbf{Z}_1$ and $\mathbf{Z}_2$, are linearly separably if at least one hyperplane exists dividing the sets, i.e. $\w$ exists such that, $\w^\intercal \z_i > c \;\; \forall \z_i \in \mathbf{Z}_1$ and $\w^\intercal \z_j < c \;\; \forall \z_j \in \mathbf{Z}_2$, for some constant $c \in \mathbb{R}$. 

Our results assume densities rather than finite sets of data points, so this does not obviously apply, but we relax the concept slightly. 
For two densities, $p(\z|\y=1)$ and $p(\z|\y=2)$ to be linearly separable, there must exist a hyperplane such that the density of each class falling on the opposite side of the hyperplane is negligible, i.e. $\w$ and $c$ exist such that, $p(\w^\intercal \z < c |\y=2) \approx 1$ and $p(\w^\intercal \z > c |\y=2) \approx 0$, while $p(\w^\intercal \z > c |\y=1) \approx 1$ and $p(\w^\intercal \z < c |\y=1) \approx 0$.
\end{definition}

\begin{corollary}{}
\label{sec_corollary_2class_gaussian}

This corollary extends theorem \ref{theorem_valid_ood_region_2class}. Under conditions as before, and further assuming that each class is represented by a single multivariate Gaussian, $p(\z|\y=i) =\mathcal{N}(\pmb{\mu}_i,\pmb{\Sigma}_i)$, and linearly separable (def. \ref{def_linearly_separable}), the value of $\alpha$ required to define the valid OOD region for $U_\text{max}$ and $U_\text{entropy}$ is given by solving the below,
\begin{align}
    \pi_1  \text{\textup{erf}}\left(\frac{\w_1^\intercal \pmb{\mu}_i - \w_1^\intercal(\z_0+\alpha \w_1)}{\sqrt{2 \w_1^\intercal \pmb{\Sigma}_i \w_1 }} \right)
    + \pi_2    \text{\textup{erf}}\left(\frac{\w_1^\intercal \pmb{\mu}_i - \w_1^\intercal(\z_0+\alpha \w_1)}{\sqrt{2 \w_1^\intercal \pmb{\Sigma}_i \w_1 }} \right)
     = 1 - 2 \epsilon .
\end{align}

\end{corollary}

\begin{remark}
Justification for the Gaussian assumption comes from topological studies finding class components are well modelled by single connected components \citep{Naitzat2020}. Secondly, similar assumptions have been popularised in prior work, e.g. \citep{Lee2018a}.
\end{remark}

\begin{proof}

This result requires deriving an analytical expression for eq. \ref{eq_probs_k2}.
By assumption we have that $p_\text{in}(\z|\y=i) = \mathcal{N}(\z; \pmb{\mu}_i, \pmb{\Sigma}_i)$, which allows us to analytically evaluate the proportion of a component falling either side of a hyperplane -- technically, integrating over a half space. Our assumption on the densities being linearly separable (def. \ref{def_linearly_separable}) simplifies this further, since we need only consider the integral over a single Gaussian component on each side of the decision boundary, with negligible impact on the result.


The required computation to integrate over the training data, $\z \sim \mathcal{D}_\text{in}$, for one component is\footnote{\url{https://math.stackexchange.com/questions/556977/gaussian-integrals-over-a-half-space}} ,
\begin{align}
    p_\text{in}(\z \notin \mathcal{R}_{K=2} |\y=1) 
    &= \frac{1}{(2\pi)^{H/2}(\det\pmb{\Sigma}_1)^{1/2}} \int_{\mathbb{R}_\alpha} 
    \exp{\left[ 
    -\frac{1}{2}(\z-\pmb{\mu}_1)^\intercal \pmb{\Sigma}^{-1} (\z-\pmb{\mu}_1) 
    \right]} d\z \\
    &= \frac{1}{2} \left[ 1 -  \text{erf}\left(-\frac{\w_1^\intercal \pmb{\mu}_i - \w_1^\intercal(\z_0+\alpha \w_1)}{\sqrt{2 \w_1^\intercal \pmb{\Sigma}_i \w_1 }} \right)
    \right].
\end{align}
Combining with the second component, $p_\text{in}(\z \notin \mathcal{R}_{K=2} |\y=2)$, and rearranging, produces the main result.

\end{proof}

\begin{proposition}{}
\label{prop_valid_ood_region_density}
Consider $K \geq 2$ classes, each represented by a single multivariate Gaussian. The valid OOD region (def. \ref{def_valid_ood_region}) for $U_\text{density}$ is given by,
\begin{align}
    \mathcal{R}_\text{density}:\{\z \in \mathbb{R}^H | \bigcap_{i}
    (\z-\pmb{\mu}_i)^\intercal \pmb{\Sigma}_i (\z-\pmb{\mu}_i) > c_i
    \}
\end{align}
where each $c_i$ depends on $\epsilon, \pmb{\mu}_i,\pmb{\Sigma}_i$.
\end{proposition}

\begin{proof}
Similar to eq. \ref{eq_probs_k2}, we require the contour defining,
\begin{align}
    p_\text{in}(\z \notin \mathcal{R}_\text{density}) = \sum_i \pi_i p_\text{in}(\z \notin \mathcal{R}_\text{density} |\y=i) = 1-\epsilon.
\end{align}
The Gaussian assumption allows this to be done analytically. For a single Gaussian component, the $1-\epsilon$\% contour is a hyperellipsoid described by $(\z-\pmb{\mu}_i)^\intercal \pmb{\Sigma}_i (\z-\pmb{\mu}_i) > c_i$, where $c_i$ depends on $\epsilon, \pmb{\mu}_i,\pmb{\Sigma}_i$. Our valid OOD region is then just the intersection over all components.
\end{proof}

\begin{definition}
\label{def_linearapprox_region}

We define the `linear approximation of the valid OOD region' (also `linear approximation') for $U_\text{max}$ as an approximation to the exact valid OOD region (def. \ref{def_valid_ood_region}), using only linear hyperplanes with structure as described following. This applies to multiple classes, $K \geq 2$, and general weight structures.



We begin with a description in words: take each decision hyperplane and create a pair of hyperplanes, offset in opposite directions by some perpendicular distance. The union of the regions enclosed by each pair will form the linear approximation of the valid OOD region. This is visualised in fig. \ref{fig_geom_linearapprox}.
\begin{figure}[h!]
\vspace{-0.1in}
\centering
\resizebox{0.5\textwidth}{!}{
  \includegraphics[width=1\textwidth]{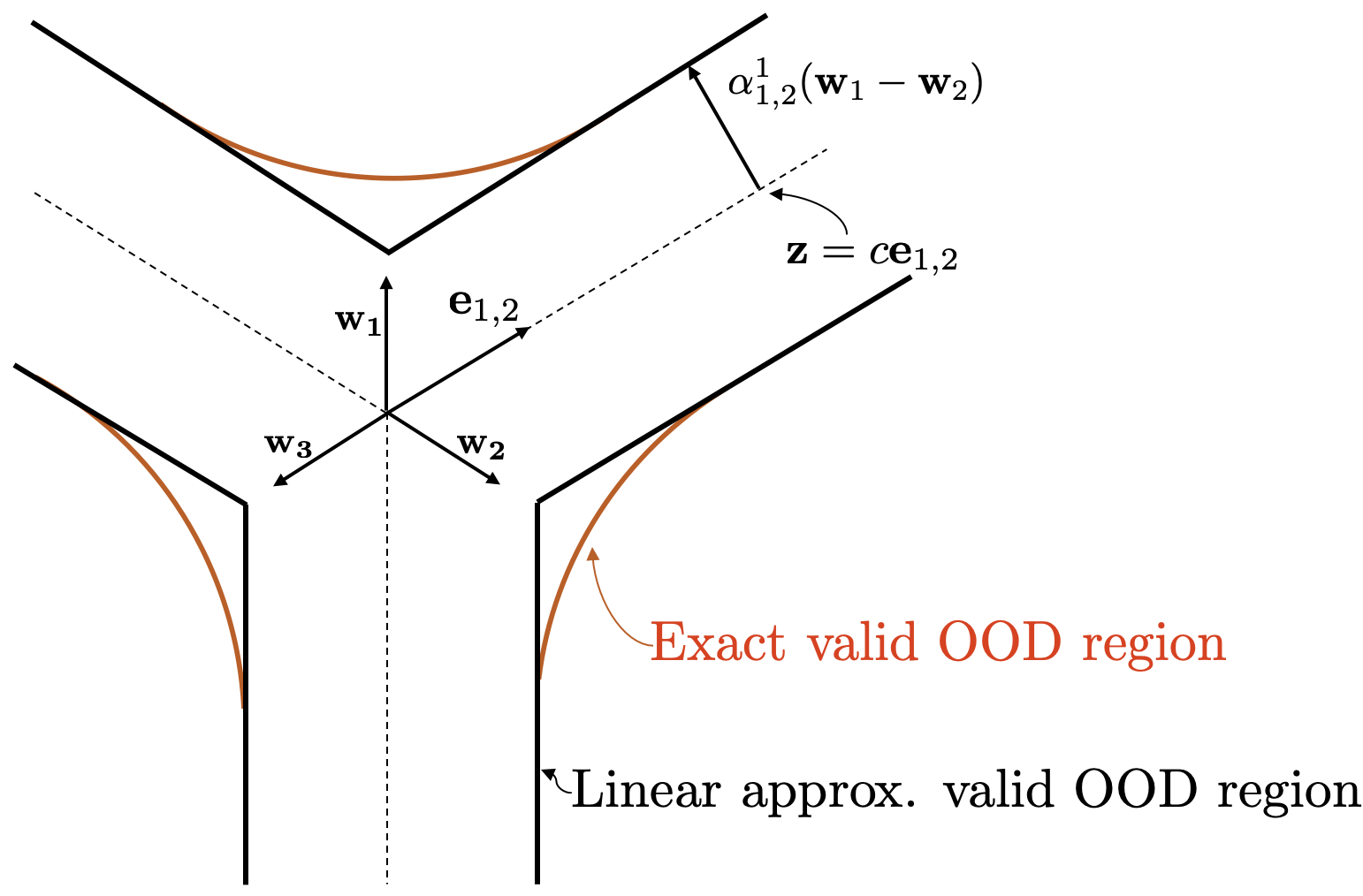} 
}
\vskip -0.05in
\caption{Geometry of the linear approximate region for $H=2$, $K=3$.}
\label{fig_geom_linearapprox}
\end{figure}

The hyperplane separating any two classes requires, $\sigma(\z)_i=\sigma(\z)_j$. For this to be true, we require $\w_i^\intercal \z = \w_j^\intercal \z \implies (\w_i - \w_j)^\intercal \z = 0 $. This produces the hyperplane,
\begin{align}
\mathcal{R}_{H,ij}:& 
\{\z \in \mathbb{R}^H | 
(\w_i-\w_j)^\intercal \z = 0
\} .
\end{align}
This hyperplane will only form part of the decision boundary where either of the classes $i,j$ are the maximum of all classes. We define another region where class $i$ is maximum,
\begin{align}
\mathcal{R}_{\hat{\y}=i}:& 
\{\z \in \mathbb{R}^H | 
i = \text{argmax} \, \sigma(\z)_i
\} ,
\end{align}
and we note for $\mathcal{R}_{H,ij}$ we will be interested in the union of the region for both class $i$ and $j$, $\mathcal{R}_{\hat{\y}=i} \cup \mathcal{R}_{\hat{\y}=j}$.

As per the description, we are interested in a pair of hyperplanes offset in opposite directions, parallel to $\mathcal{R}_{H,ij}$. Similar to the case of $K=2$, this results in a region in between the hyperplane pair given by,
\begin{align}
\label{eq_hyperplane_offset}
\mathcal{R}_{H\text{vol},ij}:& 
\{\z \in \mathbb{R}^H | 
(\w_i-\w_j)^\intercal(\z_0 - \alpha_{i,j}^i (\w_i-\w_j))
< (\w_i-\w_j)^\intercal \z < 
(\w_i-\w_j)^\intercal(\z_0+\alpha_{i,j}^j (\w_i-\w_j) )
\} .
\end{align}
The overall linear approximation of the valid OOD region can be built using a combination of these volumes as follows,
\begin{align}
\label{eq_region_combo_k3}
\mathcal{R}_{\text{linear approx.}} \coloneqq
\bigcup_{\forall i \neq j}
\left( \mathcal{R}_{\hat{\y}=i} \cup \mathcal{R}_{\hat{\y}=j} \right) 
\cap \mathcal{R}_{H\text{vol},ij}
\end{align}

It remains to select $\alpha_{i,j}^i$ and $\alpha_{i,j}^j$, to set the distance separating each hyperplane pair. We choose these $\alpha$'s to match the width of the exact valid OOD region at a specific point as follows. Let $\mathbf{e}_{i,j}$ be the vector following the decision boundary between classes $i$ and $j$, in the plane containing the two weight vectors of interest ($\mathbf{e}_{i,j}$ is perpendicular to $(\w_1-\w_2)$). We choose the point on $\mathbf{e}_{i,j}$ to be far from the origin, so $\z = c \mathbf{e}_{i,j}$, for, $c \to \infty$.


Note that the exact valid OOD region will follow contours of constant softmax confidence. For general (non-optimal) weight structures it will \textbf{not} be true that, $\sigma(\mathbf{e}_{i,j} + \alpha_{i,j}^i(\w_1-\w_2))_i = \sigma(\mathbf{e}_{i,j} - \alpha_{i,j}^j(\w_1-\w_2))_j  $, if, $ \alpha_{i,j}^i = \alpha_{i,j}^j$, hence specifying the linear approximate region requires determining up to $K(K-1)$ $\alpha$ parameters (each class forms a decision boundary with every other class). However, in the special case of optimal weight structure (def. \ref{def_optimal_structure}), due to symmetries, only a single global $\alpha$ parameter is required.	We constrain, $\alpha>0$.


\end{definition}

\begin{proposition}
\label{prop_linearapprox_equalfor_k2}
The linear approximate valid OOD region (def. \ref{def_linearapprox_region}) is equal to the exact valid OOD region (as per theorem \ref{theorem_valid_ood_region_2class}) for $K=2$ if $\w_1 = -\w_2$.
\end{proposition}
\begin{proof}
Observe that in eq. \ref{eq_region_combo_k3}, if we have $K=2$, there is only one combination of $i,j$, and $\mathcal{R}_{\hat{\y}=1} \cup \mathcal{R}_{\hat{\y}=2}$ covers the whole space $\mathbb{R}^H$. Hence, 
\begin{align}
\mathcal{R}_{\text{linear approx.}} &=
 \mathcal{R}_{H\text{vol},1,2} .
\end{align}
It's straightforward to see that eq. \ref{eq_hyperplane_offset} is a scaled version of eq. \ref{eq_hyperplane_offset_k2} since by assumption $\w_1 = -\w_2$. (Recall scaling does not alter a hyperplane.) Also due to symmetry, $\alpha_{i,j}^i = \alpha_{i,j}^j$.

\end{proof}


\begin{theorem}
\label{theorem_exact_linear_convergence}
The linear approximate valid OOD region is an increasingly close approximation of the $U_\text{max}$ exact valid OOD region at large magnitudes, for any $K\geq3$ and general weight structures, for all points except close to the intersection between decision boundary hyperplanes.
\end{theorem}
\begin{proof}
We begin by referring back to fig. \ref{fig_geom_linearapprox}, which illustrates intuitively what we will show. While the exact valid OOD region is curved at the intersection between decision boundary hyperplanes (both close to the origin and where two decision boundary hyperplanes intersect), far away from these intersections, the exact valid OOD region does become linear.

Consider some $\z$ vector that moves along the decision hyperplane between classes 1 and 2 in any direction, $\beta \mathbf{d}_{1,2}$, and perpendicular to the hyperplane, $\alpha (\w_1 - \w_2)$,
\begin{align}
    \z = \beta \mathbf{d}_{1,2} + \alpha (\w_1 - \w_2),
\end{align}
illustrated in fig. \ref{fig_vector_geom}. Note that, $(\w_1-\w_2)^\intercal \mathbf{d}_{1,2} = 0$, and that $\mathbf{e}_{1,2}$ (from def. \ref{def_linearapprox_region}) is a special case of $\mathbf{d}_{1,2}$, where it lives on the plane of $\w_1$ and $\w_2$.
\begin{figure}[h!]
\centering
\resizebox{0.5\textwidth}{!}{
  \vspace{-2pt}
  \includegraphics[width=1\textwidth]{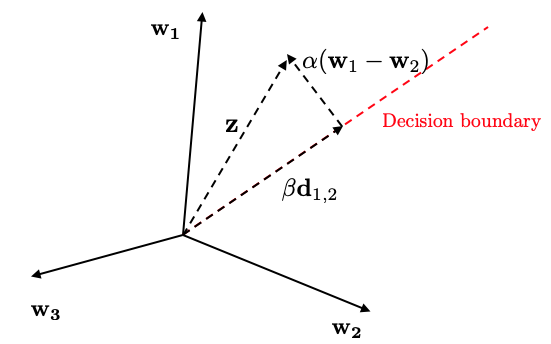} 
}
\vskip -0.05in
\caption{Vector geometry for $H=2$, $K=3$, and general (non-optimal) weight structuring.}
\label{fig_vector_geom}
\end{figure}

Note that by specifying $\mathbf{d}_{1,2}$, $\beta$ and $\alpha$, this $\z$ vector can span the entire $\mathbb{R}^H$ space (though we are only interested in its use over the range $\mathcal{R}_{\hat{\y}=1} \cup \mathcal{R}_{\hat{\y}=2}$). We now push this vector through the softmax and simplify, presuming we are interested in $\beta>0$, so the max class is 1.
\begin{align}
    \sigma(\z)_1 &= 
    \frac{\exp^{\w_1^\intercal(\beta \mathbf{d}_{1,2} + \alpha (\w_1 - \w_2))}}
    {\exp^{\w_1^\intercal(\beta \mathbf{d}_{1,2} + \alpha (\w_1 - \w_2))} + \exp^{\w_2^\intercal(\beta \mathbf{d}_{1,2} + \alpha (\w_1 - \w_2))}+ \sum_{j=3}^K \exp^{\w_j^\intercal(\beta \mathbf{d}_{1,2} + \alpha (\w_1 - \w_2))}} \\
    &= 
    \frac{\exp^{\alpha \w_1^\intercal (\w_1 - \w_2)}}
    {\exp^{\alpha \w_1^\intercal (\w_1 - \w_2)} + 
    \exp^{\w_2^\intercal(\beta \mathbf{d}_{1,2} + \alpha (\w_1 - \w_2) )- \w_1^\intercal \beta \mathbf{d}_{1,2}}+ 
    \sum_{j=3}^K \exp^{\w_j^\intercal(\beta \mathbf{d}_{1,2} + \alpha (\w_1 - \w_2)) - \w_1^\intercal \beta \mathbf{d}_{1,2}}} \\
    &= 
    \frac{\exp^{\alpha \w_1^\intercal (\w_1 - \w_2)}}
    {\exp^{\alpha \w_1^\intercal (\w_1 - \w_2)} + 
    \exp^{\alpha \w_2^\intercal (\w_1 - \w_2)}+ 
    \sum_{j=3}^K \exp^{\beta( \w_j^\intercal\mathbf{d}_{1,2} - \w_1^\intercal \mathbf{d}_{1,2}) +  \alpha \w_j^\intercal(\w_1 - \w_2)}} \label{eq_33} .
\end{align}
Since we know that, $\w_j^\intercal\mathbf{d}_{1,2} < \w_1^\intercal \mathbf{d}_{1,2}$ (the logit for class 1 must be largest), we can say that provided $ \w_1^\intercal \mathbf{d}_{1,2} = \w_2^\intercal \mathbf{d}_{1,2} \neq \w_j^\intercal \mathbf{d}_{1,2} \; \forall j \geq 3$,
\begin{align}
    \lim_{\beta \to \infty}
    \sum_{j=3}^K \exp^{\beta( \w_j^\intercal\mathbf{d}_{1,2} - \w_1^\intercal \mathbf{d}_{1,2}) +  \alpha \w_j^\intercal(\w_1 - \w_2)} = 0.
    \label{eq_34}
\end{align}
Hence, when $\beta$ grows large (at large magnitudes / at points far from the origin), eq. \ref{eq_33} depends on $\alpha$ in a constant way for nearly all vectors $\mathbf{d}_{1,2}$.

Our linear approximate region def. \ref{def_linearapprox_region} is defined to match the width of the exact valid OOD region, far from the origin. Therefore, at increasing magnitudes, the linear approximate valid OOD region is an increasingly close approximation of the exact valid OOD region for most points.

However, note that the rate of this convergence is slower when $\mathbf{d}_{1,2}$ approaches meeting the criteria; $ \w_1^\intercal \mathbf{d}_{1,2} = \w_2^\intercal \mathbf{d}_{1,2} = \w_j^\intercal \mathbf{d}_{1,2}$ for some $j \geq 3$. 
This convergence happens at all orientations except where $\mathbf{d}_{1,2}$ intersects with a second decision boundary hyperplane. For example, if $ \w_1^\intercal \mathbf{d}_{1,2} = \w_2^\intercal \mathbf{d}_{1,2} = \w_3^\intercal \mathbf{d}_{1,2}$, then eq. \ref{eq_34} does not reduce to zero, and an additional term remains in the denominator of $\sigma(\z)_i$, that departs from the linear plane over the rest of the space.


\end{proof}

\begin{corollary}
\label{corollary_linear_region_is_subset_of_exact}
This corollary is an extension of theorem \ref{theorem_exact_linear_convergence}.
Consider the linear approximate valid OOD region (def. \ref{def_linearapprox_region}), $\mathcal{R}_{\text{linear approx.}}$, for any number of classes $K \geq 3$ and general weight structure. 
The linear approximate valid OOD region forms a subset of the exact valid OOD region (def. \ref{def_valid_ood_region}), $\mathcal{R}_{\text{linear approx.}} \subseteq \mathcal{R}_{\text{exact}}$.
\end{corollary}

\begin{proof}
The linear approximate valid OOD region chooses all $\alpha$'s by assuming that $\beta=\infty$ in eq. \ref{eq_33}. 
In general $\beta < \infty$ so we can say,
\begin{align}
    \sum_{j=3}^K \exp^{\beta( \w_j^\intercal\mathbf{d}_{1,2} - \w_1^\intercal \mathbf{d}_{1,2}) +  \alpha \w_j^\intercal(\w_1 - \w_2)} > 0,  \;\; \forall \, 0\leq \beta < \infty, \;\; \forall \, \alpha>0 .
\end{align}
Hence, for some $\z$ with $0\leq \beta < \infty$, $\sigma(\z)_1$ can only be lower than the linear approximate valid OOD region, and $U_\text{max}(\z)$  can only be higher. Therefore the linear region must be a subset of the exact valid OOD region
\end{proof}

\newpage
\subsection{Optimal Structure}
\label{sec_app_proof_optimalstructure}

\begin{theorem}
\label{theorem_hyperplane_structure_optimal}
Consider $K$ equally weighted components in $\mathbb{R}^H$, where $H \geq K-1 $,  each following a Gaussian distribution, $p(\z|\y=i) = \mathcal{N}(\z;\pmb{\mu}_i,\pmb{\Sigma}_i)$. Assume mean vectors have fixed magnitude, $||\pmb{\mu}_i|| = c_2 \; \forall i$.
These clusters are to be separated with a softmax with weights $\w_i$.
In order to minimise a regularised cross-entropy loss,
\begin{align}
\pmb{\mu}_i,\pmb{\Sigma}_i, \w_i &= \text{argmin}_{\pmb{\mu}_i,\pmb{\Sigma}_i, \w_i} - \mathbb{E}_{p(\z)}
 \bigg[ \sum_i^K p(\y=i|\z)   \log \sigma(\z)_i \bigg] + \sum_i^K  \lambda_1 \lvert \lvert \w_i \rvert \rvert^2 , 
\end{align}
the parameters $\pmb{\mu}_i,\pmb{\Sigma}_i, \w_i$, must have an optimal decision boundary structure (def. \ref{def_optimal_structure}).


\end{theorem} 
 
\begin{remark}
The assumption that `mean vectors have fixed magnitude' may not be strictly necessary for this proof, but it reduces the space of possible structures we must consider, making it more straightforward. 

The assumption can be justified in two ways. 1) Empirically, $||\z||$ does not vary greatly within the training distribution with no obvious signs of multi-modal behaviour in fig. \ref{fig_cifar_etc}. 2) Earlier in the network, weight initialisations, batchnormalisation and regularisation might be expected to encourage activations for different classes to be of similar magnitudes.
\end{remark}

\begin{proof}


In order to align with the optimal structure, the following criteria must be fulfilled.
\begin{itemize}[]
    \item Criteria 1. Weights are of constant magnitude, $||\w_i|| = c_1 \; \forall i$, and as large as possible subject to regularisation.
    \item Criteria 2. Bias value are zero.
    \item Criteria 3a. Weight vectors are evenly distributed.
    \item Criteria 3b. $\cos \theta_{i,j} = \frac{-1}{K-1} \; \forall i \neq j$.
    \item Criteria 4. Weights point toward mean vectors, $\pmb{\mu}_i = c_3\w_i$.
    \item Criteria 5. Variance of clusters is small, $\pmb{\Sigma}_i \approx \mathbf{0}$.
\end{itemize}
First observe that, since we have full control over placement of class clusters, we should choose them to be linearly separable (def. \ref{def_linearly_separable}), which simplifies our objective,
\begin{align}
\pmb{\mu}_i,\pmb{\Sigma}_i, \w_i &= \text{argmin}_{\pmb{\mu}_i,\pmb{\Sigma}_i, \w_i} - \sum_i^K \mathbb{E}_{p(\z|y=i)}
 \bigg[  \log \sigma(\z)_i \bigg] + \lambda_1 \lvert \lvert \w_i \rvert \rvert^2 .
\end{align}
We now consider the effect of the covariance matrix on our objective.
We use an analytical approximation for passing a normally distributed random vector through a softmax \citep{Lu2020},
\begin{align}
\int_\z \sigma(\z)_i \mathcal{N}(\z; \pmb{\mu}_i, \pmb{\Sigma}_i) d\z 
\approx
 \frac{1}{1 + \sum_{j \neq i}\exp{-\frac{\w_j^\intercal \z - \w_i^\intercal \z}{\sqrt{ 1+ \frac{\pi}{8} (\w_i^\intercal \Sigma_i \w_i + \w_j^\intercal \Sigma_i \w_j )}} } }.
\end{align}
From this we can see that softmax confidence always decreases with a increasing variance, and hence is maximised when, $\pmb{\Sigma}_i \approx \mathbf{0}$ -- this proves \textbf{criteria 5}. Class clusters are therefore delta functions, and we can rewrite our objective, 
\begin{align}
\max_{\pmb{\mu}_i, ||\w_i||, \cos \theta_{\pmb{\mu}_i,i} \forall i}  \sum_i   \log \left( \frac{\exp ||\w_i|| c_2 \cos \theta_{\pmb{\mu}_i,i} }{\sum_j \exp ||\w_j|| c_2 \cos \theta_{\pmb{\mu}_i,j} }  \right)  - \lambda_1 ||\w_i||^2.
\end{align}
Since $||\pmb{\mu}_i||$ is fixed, there is no option to `stack' boundaries (fig. \ref{fig_counterfact_structs}), and so $\pmb{\mu}_i$'s must be arranged on a hypersphere. There creates inherent symmetry in our problem, and weights should be of constant magnitude. $\sigma(\z)_i$ is increasing in $||\w_i||$ (proposition \ref{prop_optimal_zmag}), which should be as large as possible subject to the regularisation penalty (\textbf{criteria 1}).
\begin{align}
\max_{\pmb{\mu}_i, \cos \theta_{\pmb{\mu}_i,i} \forall i}  \sum_i   \log \left( \frac{\exp c_1 c_2 \cos \theta_{\pmb{\mu}_i,i} }{\sum_j \exp c_1 c_2 \cos \theta_{\pmb{\mu}_i,j} }  \right)  - \lambda_1 ||\w_i||^2
\end{align}
It remains to determine how these weight vectors should be arranged.
One would like to maximise $\cos \theta_{\pmb{\mu}_i,i}$, and minimise $\cos \theta_{\pmb{\mu}_i,j} \; \forall i \neq j$. 
Maximising $\cos \theta_{\pmb{\mu}_i,i}$ is straightforward; by setting $\pmb{\mu}_i = c_3 \w_i$ for $c_3>0$ (\textbf{criteria 4}), we find $\cos \theta_{\pmb{\mu}_i,i}=1$.


To minimise $\cos \theta_{\pmb{\mu}_i,j} \; \forall i \neq j$, we turn to an area know as spherical codes \citep{Coxeter1989, Brauchart2015, Musin2015}.
A general problem considered in this area requires arranging $K$ unit vectors in $H$ dimensions, so that the minimum angle between any pair of vectors, is maximised (see `the kissing number problem' and `Tammes problem'). This matches our own objective, where we'd like to minimise the cosine angle.

For general $H$ and $K$, this problem is unsolved, however when $H \geq K-1$, there are more degrees of freedom allowed by the space than there are vectors, and the problem becomes more straightforward. Note that in neural networks, typically $H \gg K$.



The best thing one can do if $H \geq K-1$, is to evenly distribute weight vectors \citep{Bagchi1997} (\textbf{criteria 3a}), this implies, $\sum_i \w_i = \mathbf{0}$, i.e. directly summing all vectors returns to the origin, and also, $ \cos \theta_{i,j} = \text{const.} \;\; \forall i \neq j$, i.e. the angle between all pairs of vectors is equal.

For any vector, $\mathbf{a} \in \mathbb{R}^H$, we can therefore say,
\begin{align}
\mathbf{a}^\intercal \sum_i \w_i = \sum_i \w_i^\intercal \mathbf{a} = 0.
\label{eq_aw_equals_0}
\end{align}
This allows us to deduce the result in \textbf{criteria 3b},
\begin{align}
0 &= \sum_i \sum_j \w_i^\intercal \w_j \\
   &= \sum_j \w_j^\intercal \w_j + \sum_{i \neq j} \w_i^\intercal \w_j \\
   &= K c_1^2 +  \sum_{i \neq j} c_1^2 \cos \theta_{i,j} \\
   &= K +  2 \sum_{i > j} \cos \theta_{i,j} \\
   &= K +  2  {K \choose 2} \cos \theta_{i,j}  \label{eq_nKchoose2} \\
   &= K +  2  \frac{K !}{2! (K-2)!}\cos \theta_{i,j} \\
   &= K +   K(K-1) \cos \theta_{i,j}\\
 \cos \theta_{i,j} &= \frac{-1}{K-1} .
\end{align}
Finally we think of the bias, absorbed into $\w_i$, as having a corresponding final-layer feature of a constant 1 across classes. Since this can not be modified as part of $\pmb{\mu}_i$, it only introduces a penalty via the regularisation term, $\lambda_1 \lvert \lvert \w_i \rvert \rvert^2$. Hence it should be set to zero (\textbf{criteria 2}).






\vspace{2in}

\end{proof}

\begin{proposition}
\label{prop_optimal_zmag}
Consider a softmax with weight vectors, $\w_i$, and $K$ classes.
For any weight structure, $U_\text{max}(\z)$ increases with decreasing $||\z||$. 
\end{proposition}

\begin{proof}
We have,
\begin{align}
   \sigma(\z)_i &= \frac{\exp{\w_i^\intercal \z}}
   {\sum_j \exp{\w_j^\intercal \z}} \\
   &= \frac{1}
   {1 + \sum_{j \neq i} \exp{\w_j^\intercal \z - \w_i^\intercal \z}} \\
   &= \frac{1}
   {1 + \sum_{j \neq i} \exp{||\w_j|| \, ||\z|| \cos \theta_{\z,j} - ||\w_i|| \, ||\z|| \cos \theta_{\z,i}}} \\
   &= \frac{1}
   {1 + \sum_{j \neq i} \exp{||\z|| \left( ||\w_j|| \cos \theta_{\z,j} - ||\w_i|| \cos \theta_{\z,i} \right)}} \\
\end{align}
which clearly decreases with decreasing $||\z||$. Hence, $U_\text{max}(\z) = \max_i -\sigma(\z)_i$ increases with decreasing $||\z||$.

\end{proof}

\begin{remark}
Ahead of presenting theorem \ref{theorem_optimal_costheta_k2} \& \ref{theorem_optimal_costheta}, we first provide a simple result that intuits why we'd expect their result to be true for any number of classes, $K\geq2$. We then explain why the proof is more involved than might initially be expected, and outline our approach.


We can say that the sum of the changes in cosine angles must be zero,
\begin{align}
    \sum_{i=1}^K \vartriangle \cos \theta_{\z,i} = 0 .
\end{align}
Decreasing $\max_i \cos \theta_{\z,i}$ requires setting $\vartriangle \cos \theta_{\z,1} <0$. This means the cumulative change in all other cosine angles must be positive, $-\vartriangle \cos \theta_{\z,1} = \sum_{j=2}^K \vartriangle \cos \theta_{\z,j}>0$.
Writing this out in a form similar to the softmax but without exponentials clearly decreases with decreasing $\vartriangle \cos \theta_{\z,1}$,
\begin{align}
    \frac{ \vartriangle \cos \theta_{\z,1}}{\vartriangle  \cos \theta_{\z,1} + \sum_{j=2}^K \vartriangle  \cos \theta_{\z,j}}.
\end{align}

Unfortunately, for general $a_i$'s and $b_i$'s,
\begin{align}
    \sum a_i >  \sum b_i \nRightarrow \sum \exp^{a_i} >\exp^{b_i}.
\end{align}
Hence, introducing the exponentials to create a softmax makes the proof challenging, and we must prove there are properties of the $\cos \theta_{\z,i}$'s that make this true.
We have taken two approaches:

1) \textbf{Theorem \ref{theorem_optimal_costheta_k2}.} For low numbers of classes, it is possible to keep track of how angles between $\z$ and every $\w_i$ change with a change in $\z$, and to enumerate all possible rotations of a vector. However, this approach becomes cumbersome beyond $K=3$. 

2) \textbf{Theorem \ref{theorem_optimal_costheta}.} On the other hand, we have, $\cos \theta_{i,j} = \frac{-1}{K-1}$, and for large numbers of classes, $\cos \theta_{i,j}  \approx 0$, i.e. vectors are approximately mutually orthogonal to each other. This simplifies our proof, since we no longer need to keep track of angles between $\z$ and \textit{every} $\w_i$, only those which it is moving away/towards from (the rest can be assumed to remain at $\cos \theta_{i,j} \approx 0$).



\end{remark}

\begin{theorem}
\label{theorem_optimal_costheta_k2}
Consider a softmax with weight vectors, $\w_i$, and $K=2$ or $K=3$  classes.
For the optimal structure (def. \ref{def_optimal_structure}), $U_\text{max}(\z)$ increases with decreasing $\max_i \cos \theta_{\z,i}$ (but fixed $||\z||$).
\end{theorem}

\begin{proof}

We first observe two useful equations.

Similar to eq. \ref{eq_aw_equals_0}, for general classes,
\begin{align}
   0 &= \sum_i \w_i^\intercal \z \\
    &= \sum_i ||\w_i|| \, ||\z|| \cos \theta_{\z,i} \\
    &= \sum_i \cos \theta_{\z,i} .
\end{align}
It will simplify notation to write the softmax with weight vectors relabelled, so $\w_1$ now represents the weight vector closest (largest cosine similarity) to $\z$,  $\cos \theta_{\z,1} \coloneqq \max_i \cos \theta_{\z,i}$, $\w_2$ the second closest and so on. We also have,
 \begin{align}
	\max_i \sigma(\z)_i &= \ \frac{ \exp ||\w_1|| \, ||\z|| \cos \theta_{\z,1} }
	    {\sum_{j} \exp ||\w_j|| \, ||\z|| \cos \theta_{\z,j}} \\
	&= \frac{ \exp c \cos \theta_{\z,1} }
	    {\exp c \cos \theta_{\z,1} + \sum_{j =2}^K \exp c \cos \theta_{\z,j}} .
 \end{align}
for some constant, $||\w_i|| \, ||\z|| = c >0$. The specific value is unimportant here, it's effect being evaluated independently in lemma \ref{prop_optimal_zmag}.

\textbf{Two classes}

For two classes we have, $ \cos \theta_{\z,1} + \cos \theta_{\z,2} = 0$. 

Any decrease in $\cos \theta_{\z,1}$ therefore leads to an increase in $ \cos \theta_{\z,2}$.
Inspecting the softmax,
 \begin{align}
	-U_\text{max}(\z) = \max_i \sigma(\z)_i
	&= \frac{ \exp c  \cos \theta_{\z,1} }
	    { \exp c \cos \theta_{\z,1} +  \exp c \cos \theta_{\z,2}} \\
	&= \frac{ 1 }
	    { 1 +  \exp c (\cos \theta_{\z,2} - \cos \theta_{\z,1} )} \\
	&= \frac{ 1 }
	    { 1 +  \exp -2c \cos \theta_{\z,1} } 
 \end{align}
reveals that $\max_i \sigma(\z)_i$ decreases with decreasing $\cos \theta_{\z,1}$, so $U_\text{max}(\z)$ increases.

\textbf{Three classes}

For three classes, $ \cos \theta_{\z,1} + \cos \theta_{\z,2} + \cos \theta_{\z,3} = 0$.
The subspace covered by the three weight vectors is actually a 2D plane (in fact the space covered by evenly distributed vectors is always a subspace, $\mathbb{R}^{K-1}$ \citep{Bagchi1997} hence our requirement for $H \geq K-1$).  On this 2D plane, $\cos \theta_{i,j} = \frac{-1}{K-1} = \frac{-1}{2} \;\; \forall i \neq j \; \implies \theta_{i,j} = \frac{2}{3}\pi$.

First consider the case when $\z$ is also on this plane.
We can capture all possible movement by considering the change of a single parameter, $\theta_{\z,1} \in [0, \pi/3]$ (noting beyond $\pi/3$ the maximum class changes, and all other possibilities are symmetries of this). 

In this 2D planar case, we can write, $\theta_{\z,2} = \frac{2}{3}\pi - \theta_{\z,1}$ and $\theta_{\z,3} = \frac{2}{3}\pi + \theta_{\z,1}$.
We now express the softmax in these terms,
 \begin{align}
	-U_\text{max}(\z) = \max_i \sigma(\z)_i
	&= \frac{ \exp c  \cos \theta_{\z,1} }
	    { \exp c \cos \theta_{\z,1} +  \exp c \cos \theta_{\z,2} +\exp c \cos \theta_{\z,3} } \\
	&= \frac{ \exp c  \cos \theta_{\z,1} }
	    { \exp c \cos \theta_{\z,1} +  \exp c \cos( \frac{2}{3}\pi - \theta_{\z,1}) +\exp c \cos ( \frac{2}{3}\pi + \theta_{\z,1})} \label{eq_k3_intermsof_theta1}
 \end{align}
 This can be plugged into a derivative calculator, as visualised in fig. \ref{fig_derviative_k3}, to verify that, $\frac{\partial \max_i \sigma(\z)_i}{\partial \cos \theta_{\z,1} }<0$ across the evaluated range, hence $\max_i \sigma(\z)_i$ is decreasing in $\cos \theta_{\z,1}$, or equivalently, $U_\text{max}(\z)$ is increasing in $\cos \theta_{\z,1}$.
 \begin{figure}[h]
  \vskip 0.05in
 \begin{center}
 \includegraphics[width=0.4\columnwidth]{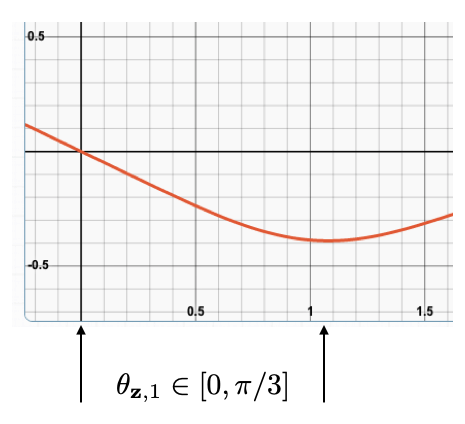} 
 \caption{Plot of first derivative of eq. \ref{eq_k3_intermsof_theta1}, $\frac{\partial \max_i \sigma(\z)_i}{\partial \cos \theta_{\z,1} }$.}
 \label{fig_derviative_k3}
 \end{center}
  \vskip -0.2in
 \end{figure}

It remains to describe what happens when $\cos \theta_{\z,1}$ is decreased by lifting the $\z$ vector away from the 2D subspace, rather than moved within the plane as above. Consider some modified vector rotated off the 2D plane by some angle $\theta_r$, which can be created by, $\z_r = \z \cos \theta_r + \mathbf{n} \sin \theta_r $ (here $\z$ lives on the 2D plane), for some vector normal to the 2D plane, $\mathbf{n}$. Now consider passing this vector through the softmax, 
 \begin{align}
	-U_\text{max}(\z) = \max_i \sigma(\z_r)_i
	&= \frac{ \exp \w_1^\intercal \z_r }
	    { \exp \w_1^\intercal \z_r  +  \exp \w_2^\intercal \z_r  +\exp \w_3^\intercal \z_r  } \\
&= \frac{ \exp \w_1^\intercal (\z \cos \theta_r + \mathbf{n} \sin \theta_r ) }
	    { \exp \w_1^\intercal (\z \cos \theta_r + \mathbf{n} \sin \theta_r )  +  \exp \w_2^\intercal (\z \cos \theta_r + \mathbf{n} \sin \theta_r ) + \cdots  } \\
&= \frac{ \exp \w_1^\intercal (\z \cos \theta_r  ) }
	    { \exp \w_1^\intercal (\z \cos \theta_r  )  +  \exp \w_2^\intercal (\z \cos \theta_r ) +\exp \w_3^\intercal (\z \cos \theta_r )  } 
 \end{align}
where the last line follows since $\mathbf{n}$ is normal to \textit{all} weight vectors, so $\w_i^\intercal \mathbf{n} = 0$. Here $\cos \theta_r$ has a similar cooling effect to lowering the magnitude as described in lemma \ref{prop_optimal_zmag}. Hence, decreasing $\cos \theta_{\z,1}$ by rotating it off the plane also reduces softmax confidence, and hence increases $U_\text{max}(\z)$.

\end{proof}

\begin{theorem}
\label{theorem_optimal_costheta}
Consider a softmax with weight vectors, $\w_i$, and some large number of $K$ classes, so it holds that $\cos \theta_{i,j} = \frac{-1}{K-1} \approx 0$.
For the optimal structure (def. \ref{def_optimal_structure}), $U_\text{max}(\z)$ increases with decreasing $\max_i \cos \theta_{\z,i}$. 
\end{theorem}

\begin{proof}

We borrow notation and several results from theorem \ref{theorem_optimal_costheta_k2}; $\cos \theta_{\z,1}$ will refer to the largest cosine angle between any weight vector $\w_i$ and $\z$.
We also use the conservation of $\cos \theta_{\z,i}$'s,
\begin{align}
    \sum_{i=1}^K \cos \theta_{\z,i} = 0 .
\end{align}

We will consider the change in $U_\text{max}$ for two vectors which are modified in specific ways to decrease $\max \cos \theta_{\z,1}$. For $\z_\text{orig.}$ and $\z_\text{mod.}$, we study the change in uncertainty estimator,
$U_\text{max}(\z_\text{orig.}) - U_\text{max}(\z_\text{mod.})$. Where obvious from context we drop the orig. and mod. subscript. Note that the changes in $\cos \theta_{\z,i}$'s are also conserved,
\begin{align}
    \sum_{i=1}^K \cos \theta_{\z_\text{orig.},i} - \cos \theta_{\z_\text{mod.},i} = \sum_{i=1}^K  \vartriangle \cos \theta_{\z,i} = 0.
\end{align}
To lighten notation, we assume (without loss of generality) that, $||\w_i|| = ||\z_\text{orig.}||=||\z_\text{mod.}|| = 1$ (this can be achieved with a scaling of the space).

Let us begin by writing the softmax for the original and modified vector, with $\delta_i \coloneqq \vartriangle \cos \theta_{\z,i}$,
\begin{align}
- U_\text{max}(\z_\text{orig.}) = \max_i \sigma(\z_\text{orig.})_i &=
\frac{\exp^{\cos \theta_{\z,1}}}
{\exp^{\cos \theta_{\z,1}} + \sum_{j=2}^K \exp^{ \cos \theta_{\z,j}}} \\
- U_\text{max}(\z_\text{mod.}) = \max_i \sigma(\z_\text{mod.})_i &=
\frac{\exp^{\cos \theta_{\z,1} + \delta_1}}
{\exp^{\cos \theta_{\z,1} + \delta_1} + \sum_{j=2}^K \exp^{\cos \theta_{\z,j} + \delta_j}} \\
&=
\frac{\exp^{\cos \theta_{\z,1} }}
{\exp^{\cos \theta_{\z,1}} + \sum_{j=2}^K \exp^{\cos \theta_{\z,j} + \delta_j - \delta_1}} .
\end{align}
For an increase in $U_\text{max}$, or a decrease in $\max_i \sigma$, it must hold that,
\begin{align}
\label{eq_sorig_smod_compare}
    \sum_{j=2}^K \exp^{\cos \theta_{\z,j}} <
    \sum_{j=2}^K \exp^{\cos \theta_{\z,j} + \delta_j - \delta_1} .
\end{align}

\textbf{Moving between two weight vectors}

We first consider a rotation such that $\cos \theta_{\z,1}$ decreases by moving directly toward $\cos \theta_{\z,2}$. Hence, $\delta_1<0$ and $\delta_2>0$. Because rotation is only in this 2D plane, $\delta_{j}$ will all be equal for $j \geq 3$, say, $\delta_c$, and must conserve the change in $\cos \theta_{\z,i}$.
\begin{align}
 \delta_1 +\delta_2 + \sum_{j=3}^K \delta_j &= 0 \\
 \delta_1 +\delta_2 + (K-2) \delta_c &= 0 \\
 \frac{-(\delta_1 +\delta_2)}{K-2} &= \delta_c
\end{align}


We turn to showing that eq. \ref{eq_sorig_smod_compare} is true for our case. Note that since we have assumed a large number of classes, $\cos \theta_{\z,j} \approx 0 \;\; \forall j \geq 3$.
\begin{align}
    \exp^{\cos \theta_{\z,2}}+ \sum_{j=3}^K \exp^{\cos \theta_{\z,j}} &<
    \exp^{\cos \theta_{\z,2} + \delta_2 - \delta_1} +
    \sum_{j=3}^K \exp^{\cos \theta_{\z,j} + \delta_j - \delta_1} \\
    \exp^{\cos \theta_{\z,2}}+ (K-2) &<
    \exp^{\cos \theta_{\z,2} + \delta_2 - \delta_1} +
    (K-2) \exp^{ \delta_c - \delta_1} \\
     (K-2)(1- \exp^{ \delta_c - \delta_1} ) &<
    \exp^{\cos \theta_{\z,2} + \delta_2 - \delta_1}- \exp^{\cos \theta_{\z,2}}. \label{eq_88}
\end{align}
It's a little difficult to directly see why this should be true. Formally it follows from the exponential function being strictly increasing and convex.
The quantity, $1- \exp^{ \delta_c - \delta_1} $, will be small but it's multiplied by $K-2$, while, $\exp^{\cos \theta_{\z,2} + \delta_2 - \delta_1}- \exp^{\cos \theta_{\z,2}}$, will be large.
Fig. \ref{fig_costheta_diag_w2} articulates the quantities geometrically. 
Roughly speaking, each side of eq. \ref{eq_88} has a similar total budget in terms of change in $\cos \theta$, but the budget used on the RHS is applied higher up the exponential curve, so is amplified (since it is strictly increasing and convex), resulting in a greater change than on the LHS. Hence, the inequality must be true.

%

\begin{figure}[h!]
\vspace{-0.1in}
\centering
\resizebox{0.9\textwidth}{!}{
  \includegraphics[width=1\textwidth]{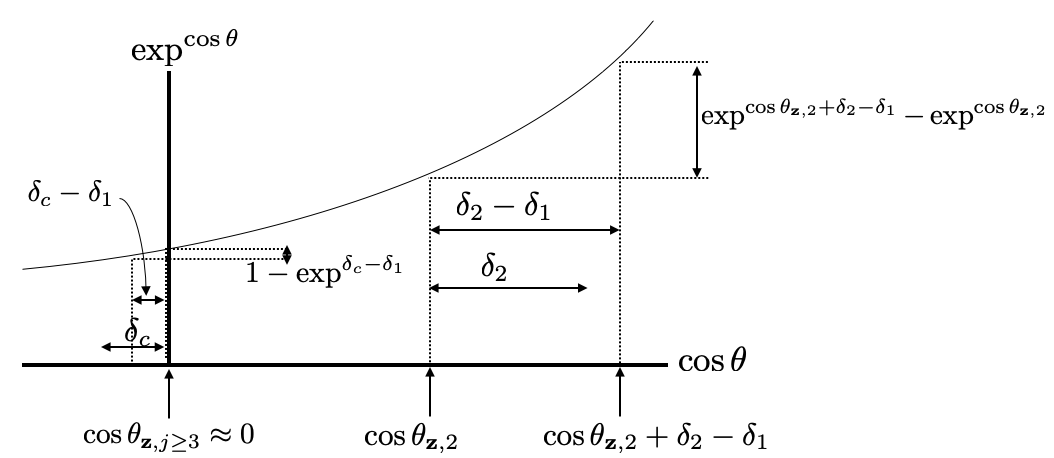} 
}
\vskip -0.05in
\caption{Geometry of eq. \ref{eq_88}.}
\label{fig_costheta_diag_w2}
\end{figure}

\textbf{Moving between multiple weight vectors}

Having shown $U_\text{max}$ always increases with a decrease in $\cos \theta_{\z,1}$ when moving on the 2D space between two vectors, we now repeat a similar exercise, but allow movement in a 3D space between three vectors, so while $\delta_1<0$, both $\delta_2>0$ and $\delta_3>0$.
\begin{align}
    \exp^{\cos \theta_{\z,2}} +
    \exp^{\cos \theta_{\z,3}} +
    \sum_{j=4}^K \exp^{\cos \theta_{\z,j}} 
    &<
    \exp^{\cos \theta_{\z,2} +\delta_2 -\delta_1} +
    \exp^{\cos \theta_{\z,3} +\delta_3 -\delta_1}  +
    \sum_{j=4}^K \exp^{\cos \theta_{\z,j}+ \epsilon_j -\delta_1}\\
    \exp^{\cos \theta_{\z,2}} +
    \exp^{\cos \theta_{\z,3}} +
    (K-3)
    &<
    \exp^{\cos \theta_{\z,2} +\delta_2 -\delta_1} +
    \exp^{\cos \theta_{\z,3} +\delta_3 -\delta_1}  +
    (K-3) \exp^{\delta_c -\delta_1}\\
    (K-3)(1-\exp^{\delta_c -\delta_1})
    &<
    (\exp^{\cos \theta_{\z,2} +\delta_2 -\delta_1} -\exp^{\cos \theta_{\z,2}}) +
    (\exp^{\cos \theta_{\z,3} +\delta_3 -\delta_1}  -\exp^{\cos \theta_{\z,3}}) 
\end{align}
Precisely the same arguments apply in this case as before; since $\cos \theta_{\z,2}>\cos \theta_{\z,j}\approx 0$ and $\cos \theta_{\z,3}>\cos \theta_{\z,j}\approx 0$, the exponential amplifies the budget of change in cosine angle on the RHS of the equation, making it true. 

Clearly this reasoning can be extended to further weight vectors.

\textbf{Moving outside of the subspace covered by all vectors}

The other way to decrease $\max_i \cos \theta_{\z,i}$ is to rotate a vector outside of the ($K-1$ dimensional) subspace covered by all weight vectors. It should be straightforward to recognise that this decreases the cosine angle of all weight vectors, increasing $U_\text{max}$ -- we discuss this in more detail in theorem \ref{theorem_optimal_costheta_k2} for $K=3$.


\end{proof}

\begin{corollary}
\label{corollary_optimal_maximises}
Generalising from theorem \ref{theorem_optimal_costheta_k2} \& \ref{theorem_optimal_costheta}, assume that for the optimal structure, for any $K \geq 2$, it holds that $U_\text{max}$ always increases for a decrease in $\max_i \cos \theta_{\z,i}$. 
Further assume fixed cluster magnitudes, $|| \pmb{\mu}_i || = c_2 \; \forall i$. The optimal structure (def. \ref{def_optimal_structure}) maximises the linear approximate valid OOD region for $U_\text{max}$. 

\end{corollary}

\begin{proof}
This result can be seen by observing fig. \ref{fig_optimal}.
First note that an increase in cluster variance will decrease the valid OOD region. Since the optimal structure assumes, $\pmb{\Sigma} \approx \mathbf{0}$, the valid OOD region is maximised from this perspective.
Secondly, if $U_\text{max}$ always increases for a decrease in $\max_i \cos \theta_{\z,i}$, the valid OOD region will lie between all data clusters.




\end{proof}

\begin{proposition}
\label{prop_optimal_notalways}
$U_\text{max}(\z)$ is not guaranteed to increase with decreasing $\max_i \cos \theta_{\z,i}$ for general (non-optimal) decision boundary structures. 
\end{proposition}
\begin{proof}
We provide one numerical counterexample to verify. Consider a `sandwich' weight structuring as in fig. \ref{fig_counterfact_structs}, with $K=3$ and $H=2$. We choose, 
\begin{align*}
    \w_1 &= [0,1]\\
    \w_2 &= [-1,0]\\
    \w_3 &= [1,0] \\
    \z_1 &= [1,0] \\
    \z_2 &= [0.9,-0.44],
\end{align*}
so that $||\z_1||=||\z_2||$, and, $\max_i \cos \theta_{\z_1, i} > \max_i \cos \theta_{\z_2, i}$ -- specifically $\max_i \cos \theta_{\z_1, i} = 1$, $\max_i \cos \theta_{\z_2, i} = 0.8$.

Passing both $\z$ vectors through a softmax gives, $\max \sigma(\z_1)_i = 0.665$ and $\max \sigma(\z_2)_i = 0.700$, hence $U_\text{max}(\z_1)>U_\text{max}(\z_2)$ despite $\z_1$ having a larger $\max \cos \theta$.

\end{proof}

\newpage
\subsection{Mental Model}
\label{sec_app_proof_mentalmodel}

\begin{proposition}
\label{prop_mentalmodel}
Assume a softmax layer separating $K$ classes has an optimal decision boundary structure (def. \ref{def_optimal_structure}), and also that $\cos\theta_{i,\z} = \frac{-1}{K-1}$ for all angles except the maximum class. An uncertainty estimator that maintains the ordering of the resulting estimator (but not absolute values) is, 
\begin{align}
U_\text{max mental}(\z) &\coloneqq
\frac{-1}
{1 + (K-1)\exp{- ||\z||( \frac{1}{K-1} + \max \cos \theta_{i,\z} ) }}.
\end{align}
\end{proposition}

\begin{proof}
Begin by writing $U_\text{max}$ in terms of magnitudes and angles. Then observe that for the optimal structure, weight magnitudes are constant, $||\w_i|| = c \;\; \forall i$. Next substitute $\cos\theta_{i,\z} = \frac{-1}{K-1}$ for non-maximum class angles,
\begin{align}
- U_\text{max}(\z) &= \frac{\exp{\lvert \lvert \w_i \rvert \rvert \, \lvert \lvert \z \rvert \rvert \cos \theta_{i,\z}}}{\sum_j \exp{\lvert \lvert \w_j \rvert \rvert \, \lvert \lvert \z \rvert \rvert \cos \theta_{j,\z}}} \\ 
&= \frac{\exp{c \, \lvert \lvert \z \rvert \rvert \cos \theta_{i,\z}}}{\sum_j \exp{c \, \lvert \lvert \z \rvert \rvert \cos \theta_{j,\z}}} \\
&= \frac{\exp{c \, \lvert \lvert \z \rvert \rvert \max \cos \theta_{i,\z} }}
{\exp{c \, \lvert \lvert \z \rvert \rvert \max \cos \theta_{i,\z} }
+ \sum_{j\neq i} \exp{c \, \lvert \lvert \z \rvert \rvert \cos \theta_{j,\z} }} \\
&= \frac{\exp{c \, \lvert \lvert \z \rvert \rvert \max \cos \theta_{i,\z} }}
{\exp{c \, \lvert \lvert \z \rvert \rvert \max \cos \theta_{i,\z} }
+ (K-1) \exp{c \, \lvert \lvert \z \rvert \rvert \frac{-1}{K-1} }} \\
&= \frac{1}
{1
+ (K-1) \exp{\left(c \, \lvert \lvert \z \rvert \rvert \frac{-1}{K-1} 
- c \, \lvert \lvert \z \rvert \rvert \max \cos \theta_{i,\z}\right)}
} \\
&= \frac{1}
{1
+ (K-1) \exp{-c \, \lvert \lvert \z \rvert \rvert \left(  \frac{1}{K-1} 
+  \max \cos \theta_{i,\z}\right)}
}
\end{align}
If we are only concerned by the relative ordering produced by our uncertainty estimator, as is the case in OOD detection, we can ignore 
the constant $c$, to produce our mental model,
\begin{align}
U_\text{max mental}(\z) &\coloneqq
\frac{-1}
{1 + (K-1)\exp{- ||\z||( \frac{1}{K-1} + \max \cos \theta_{i,\z} ) }}.
\end{align}

\end{proof}

\begin{remark}
We are careful to point out that the model we propose in proposition \ref{prop_mentalmodel} is intended to formalise findings from this paper in a single interpretable equation, where we have traded off accuracy for readability. Nevertheless, we discuss the validity of the assumptions made, arguing that, whilst imperfect, they are not entirely unreasonable. 

Assumption 1: A trained softmax layer has optimal decision boundary structure. 
Section \ref{sec_decision_analysis} evaluated how well the optimal structure matched that found in neural networks of various architectures trained on various datasets. Whilst there were differences, we argued it was a reasonable match.

Assumption 2: $\cos\theta_{i,\z} = \frac{-1}{K-1}$ for all angles except the maximum class. Under the optimal structure, data from the training distribution should point in the same direction as their corresponding weight vector. In this case, we'd expect. $\max_i \cos\theta_{i,\z} \approx 1$, and for all other $i$'s, $\cos\theta_{i,\z} = \frac{-1}{K-1} \; \forall i \neq \text{argmax}_i \cos \theta_{i,\z}$. It's less clear how well this holds for OOD data. Fig. \ref{fig_cifar_etc} (middle column) plots this empirically for trained ResNet18's, which shows similar patterns for both the OOD and training histograms, suggesting this assumption holds no worse than the optimal structure assumption.

\end{remark}

\newpage
\section{Extended Related Work}
\label{sec_app_relatedwork}

\subsection{Comparison Studies in the Literature}
\label{sec_app_baseline_table}
In table \ref{table_reported_results} we list a selection of previous studies reporting AUROC for OOD detection, both for softmax confidence, and other reference methods.

Note softmax confidence AUROC  never drops below 78\% (Shafaei et al. \citeyearpar{Shafaei2018} report accuracy), and achieves up to 98.8\%, hence we report 75\% to 99\% in the main text, also in line with our own results in fig. \ref{fig_grid_auroc}.
Modifications typically improve over softmax by 1-5\% with two exceptions, one of which used an unconventional `Dirty-MNIST' training dataset (testing regular MNIST vs. F.MNIST in our own experiments produces around 94\% AUROC).



\begin{table}[h!]
  \caption{AUROC scores for OOD detection reported in the literature. Various architectures, methods and datasets.}
  \vspace{-0.2in}
  \label{table_reported_results}
  \begin{center}
  \resizebox{\textwidth}{!}{
  \begin{tabular}{llll}
    \toprule
    Citation & $\mathcal{D}_\text{in}$ & $\mathcal{D}_\text{out}$ & Reported OOD score in format, Method:AUROC \\
    \midrule
     \textbf{Softmax only} \\
    \citep{Hendrycks2017} & CIFAR10 & Various & Softmax:96.0     \\
    \citep{Hendrycks2017} & CIFAR100 & Various & Softmax:90.0     \\
    \citep{Hendrycks2017} & MNIST & Various & Softmax:91.0     \\
    \citep{Hendrycks2017} & IMDB & Various & Softmax:94.0     \\
    \citep{Hendrycks2017} & TIMIT & Various & Softmax:97.0     \\
    \textbf{Various methods} \\
    \citep{Malinin2018} & MNIST  & Omniglot & Softmax:98.8, MC Dropout:99.2, PriorNetwork:100.0     \\
    \citep{Lee2018a} & CIFAR10 & SVHN & Softmax:89.9, Mahalanobis(vanilla):93.9     \\
    \citep{Shafaei2018} & Various & Various & Softmax:72.6, MCDropout:73.5, Ensemble:72.8 \\
    \citep{VanAmersfoort2020} & CIFAR10 & SVHN & Softmax:90.6, Ensemble:93.3, DUQ:92.7     \\
    \citep{VanAmersfoort2020} & F.MNIST & MNIST & Softmax:84.3, Ensemble:86.1, DUQ:95.5     \\
    \citep{Mukhoti2021} & CIFAR10 & SVHN & Softmax:93.0, Ensemble:97.7, DUQ:93.7, DDU:97.9     \\
    \citep{Mukhoti2021} & CIFAR10 & CIFAR100 & Softmax:88.1, Ensemble:91.4, DUQ:85.9, DDU:91.8     \\
    \citep{Mukhoti2021} & `Dirty MNIST' & F.MNIST & Softmax:84.0, Ensemble:97.7, DDU:99.8     \\
    \textbf{Pre-training} \\
    \citep{Hendrycks2019a} & CIFAR10 & SVHN & Softmax:91.8, Softmax(pre-train):95.7     \\
    \citep{Hendrycks2019a} & CIFAR10 & Various & Softmax:77.6, Softmax(pre-train):83.8     \\
    \bottomrule
  \end{tabular}
  }
\end{center}

\end{table}

\subsection{Softmax Criticisms}
\label{sec_app_softmax_critiques}

Below we quote a selection of academic papers and talks on the topic of uncertainty in deep learning, evidencing a general belief that the softmax confidence of deep neural networks is unsuitable for capturing epistemic/model/distributional uncertainty.

\textbf{Academic Papers}

\textit{"[The softmax output] is often erroneously interpreted as model confidence."} 
\newline 
\citep{Gal2015}

\textit{"Deterministic models can capture aleatoric uncertainty but cannot capture epistemic uncertainty."} 
\newline  
\citep{Gal2017}

\textit{"NNs ... until recently have been unable to provide measures of uncertainty in their predictions."} 
\newline 
\citep{Malinin2018}

\textit{"When asked to predict on a data point unlike the training data, the NN should increase its uncertainty. There is no mechanism built into standard NNs to do this ... standard NNs cannot estimate epistemic uncertainty."} 
\newline \citep{pearce2021thesis}

\textit{"NNs are poor at quantifying predictive uncertainty."} \newline  \citep{Lakshminarayanan2017}

\textit{"Deep neural networks with the softmax classifier are known to produce highly overconfident posterior distributions even for such abnormal samples."} \newline  \citep{Lee2018a}

\textit{"The output of the [softmax] classifier cannot identify these [far from the training data] inputs as out-of-distribution."} 
\newline \citep{Hein}

\textit{"The only uncertainty that can reliably be captured by looking at the softmax distribution is aleatoric uncertainty."} 
\newline \citep{VanAmersfoort2020}

\textit{"Softmax entropy is inherently inappropriate to capture epistemic uncertainty."} 
\newline  \citep{Mukhoti2021}

\textit{"For [softmax] classifiers ... misclassification will occur with high confidence if the unknown is far from any known data."} \newline  \citep{Boult2019}


\textbf{Academic Talks}

\textit{"Softmax is not telling you anything about ... model uncertainty."} 
\newline  Elise Jennings, Training Program on Extreme-Scale Computing 2019 
\newline \textcolor{blue}{\url{https://youtu.be/Puc_ujh5QZs?t=1323}} 

\textit{"The [softmax] network has no way of telling you `I'm completely uncertain about the outcome and don't rely on my prediction'."}
\newline  Florian Wilhelm, PyData Berlin 2019
\newline  \textcolor{blue}{\url{https://youtu.be/LCDIqL-8bHs?t=262}} 


%
%
%
%


\subsection{Comprehensive Literature Review}
\label{sec_app_comprehensive}



This section provides a full literature review (summarised in section \ref{sec_bg}).

\textbf{Softmax origins:}
Neural networks have long combined the cross-entropy loss with a sigmoid or softmax output for classification problems \citep{Baum, Bridle1990,Bridle1990b}. The original attraction was compatibility with maximum likelihood and information theoretic objectives and the interpretation of outputs as posterior predictive probabilities \citep{Richard1991}.

\textbf{Softmax criticism:}
Recent work has shown that interpreting softmax confidence as predictive probabilities can have several pitfalls; they are poorly calibrated and generally overconfident (i.e. 90\% confidence does not mean correct 90\% of the time) \citep{Guo2017}, and can be easily manipulated by adversarial examples \citep{Nguyen2015, Ozbulak2018a}. 
It's also been claimed there is no reason to trust them outside of the training distribution -- \textit{"[the softmax output] is often erroneously interpreted as model confidence."} \citep{Gal2015}.

This paper investigates the last point. Whilst most work on uncertainty and neural networks criticises softmax uncertainties only informally, several works more rigorously demonstrate weaknesses.
Hein et al. \citeyearpar{Hein} prove that any input can be magnified, $\tilde{\x} = \alpha \x$, with large $\alpha>1$, to produce an arbitrarily high confidence prediction by ReLU networks. Whilst they define $\alpha \x$ as `far from the training data', our work uses `far' to mean unfamiliar image datasets with bounded pixel intensity (something trivial to check). 


\textbf{More on Mukhoti et al.:} Mukhoti et al. \citeyearpar{Mukhoti2021} released work concurrently with our own. 
Though they make claims that contradict our findings, these contradictions mostly arise from overly-strong wording and differences in task set up, and in our view the two papers are actually complimentary.

Their study is motivated by settings where there is significant overlap between classes, and hence high aleatoric uncertainty in the training data. Their experiments use a custom-created `Ambiguous-MNIST' dataset, where a VAE generates inputs in-between MNIST classes, with multiple conflicting labels assigned to each input. `Dirty-MNIST' refers to a mixture of Ambiguous-MNIST and standard MNIST data. This leads to justifiably low confidence softmax estimates for much of the training dataset, and hence softmax entropy is weakened when distinguishing training data from OOD data. However, even for this dataset, softmax confidence achieves a modest 84\% AUROC in OOD detection.

There are two main arguments supporting their claim to \textit{``prove we cannot infer epistemic uncertainty from a deterministic model’s softmax entropy''}. The first hinges on the observation that by looking at softmax entropy, one cannot disentangle aleatoric from epistemic uncertainty. While we do not dispute this (section \ref{sec_failure_measure}, cause 2b), according to our paper's findings, the same logic could be used to claim the reverse; that we cannot infer aleatoric uncertainty from softmax entropy. Secondly, they argue that in an ensemble with epistemic uncertainty, some of the individual models must have lower softmax entropy than others, making them inherently unreliable. We do not dispute that there is increased variance in the softmax entropies of an ensemble for OOD data, but our findings would predict the \textit{average} softmax entropy of individual models would also increase. 


In our view Mukhoti et al. highlight an important failure of softmax conflating the two uncertainty types. The majority of benchmarks in the literature, e.g. those in table \ref{table_reported_results},  use datasets for which high accuracy is achievable, and which do not have heavy aleatoric uncertainty -- our analysis focuses on this set up, which has allowed us to investigate the relationship of softmax confidence and epistemic uncertainty in isolation. We do not comment on which of these set ups are more reflective to the real-world problems. In practise we note that it is simple for practitioners to know if there is high aleatoric uncertainty in their dataset by examining test and training accuracies (if high, there will be low aleatoric uncertainty).

\textbf{Softmax uses:} 
Two independent studies empirically assessed the ability of softmax confidence to detect OOD inputs and misclassified test inputs \citep{Hendrycks2016a,Shafaei2018}. They found softmax confidence performs reasonably well on both tasks, though has room for improvement. This modest capability has been confirmed through work using softmax as a baseline (table \ref{table_reported_results}).  





\textbf{Softmax manipulation:}
Several prominent works have shown that OOD inputs can be intentionally crafted to create high softmax confidence predictions, for example through gradient-based perturbations to the input \citep{Szegedy2014}, pattern generators optimised via evolutionary algorithms \citep{Nguyen2015}, or synthetic objects rendered in unusual poses \citep{Alcorn}. We test a more benign setup in our work, commonly used in OOD benchmarking, where OOD data is selected arbitrarily as some unrelated dataset.


\textbf{Improvements:}
Various approaches improve over uncertainties obtained directly from the softmax, e.g. ensembling \citep{Heskes1996, Tibshirani1996, Lakshminarayanan2017} and Bayesian neural networks \citep{MacKay1992}. But these always come with an increased implementation and computational burden, and have so far failed to be widely adopted \citep{Wenzel2020}. As a result softmax confidence is the sole measure of uncertainty available in most of today's neural networks.




 \textbf{Pre-training:}
 Zeiler et al. \citeyearpar{Zeiler2014} showed networks pre-trained on large-scale generic vision tasks can be fine-tuned on other vision tasks with improved accuracy and reduced training times. Recent work has shown this also leads to improved robustness, tested on OOD detection in vision and NLP tasks \citep{Orhan2019, Hendrycks2019a, Hendrycks2020}.

\newpage
\section{Experimental Details and Further Results}
\label{sec_app_experiment_details}

This section lists the main hyperparameters for each experiment. For full details, we refer readers to the code -- for each experiment and figure we point to the relevant script. These are hosted at 
\textcolor{blue}{\url{https://github.com/******}}. 

Smaller experiments in the paper were run on a machine with a single GPU (GTX 1080), while the benchmark results in table \ref{table_AUROC_bench} were run on an internal cluster of GPUs (four Titan X's).

\subsection{Misc. Figures}
\label{sec_app_exp_miscfigures}

\begin{list}{•}{}
\item Fig. \ref{fig_2d}a script: \verb|toy_classification_latent_05.py|

Architecture:
\newline
$\x \in \mathbb{R}^2$ $\to$ 16 fc $\to$ $H=16$ $\to$ softmax $K=3$,  ReLU activations, trained for 50 epochs

\item Fig. \ref{fig_2d}b script: \verb|bottleneck_mnist_08_intro.py|

Architecture:
\newline
$\x \in \mathbb{R}^{28,28,1}$ $\to$ 32$\times$(5,5) conv $\to$ 64$\times$(3,3) conv $\to$ (2,2) pooling $\to$ 128$\times$(3,3) conv $\to$ (2,2) pooling $\to$ $H=16$ fc $\to$ softmax $K=3$,  ReLU activations, trained for 2 epochs

PCA visualisation:
\newline
We select the first two principle components over the final-layer space (on a mixture of training and OOD data points). We then create a grid on this 2D plane, plotting $U_\text{entropy}$ over the grid. 100 data points from the test set, and 100 data points from the OOD dataset are plotted.
\item Fig. \ref{fig_valid_ood_region} script: \verb|region_03_2class.py|, \verb|region_04_3class.py|
\item Fig. \ref{fig_vecfield} script: \verb|vector_field_03.py|
\end{list}

\subsection{Decision Boundary Analysis}
\label{sec_app_exp_traineddecisionboundary}

Relating to experiment: Section \ref{sec_decision_analysis}, fig. \ref{fig_weight_analysis}

Run using script: \verb|symmetry_search_04.py|

LeNet models are trained from scratch with following architectures.

For $K=2$:
\newline
$\x \in \mathbb{R}^{28,28,1}$ $\to$ 32$\times$(5,5) conv $\to$ 64$\times$(3,3) conv $\to$ (2,2) pooling $\to$ 128$\times$(3,3) conv $\to$ (2,2) pooling $\to$ $H=4$ fc $\to$ softmax $K=2$,  Tanh activations

For $K=5$:
\newline
$\x \in \mathbb{R}^{28,28,1}$ $\to$ 32$\times$(5,5) conv $\to$ 64$\times$(3,3) conv $\to$ (2,2) pooling $\to$ 128$\times$(3,3) conv $\to$ (2,2) pooling $\to$ $H=32$ fc $\to$ softmax $K=5$,  Tanh activations

Trained ResNet18's (C10 and C100) are taken from experiments in section \ref{sec_failure_measure}.

Trained EfficientNetB0 is taken from tensorflow's public checkpoint.

\subsection{Counterfactual Experiments}
\label{sec_app_exp_counterfactual_structs}

Relating to experiment: Section \ref{sec_optimal_boundary_importance}, fig. \ref{fig_counterfact_structs}

Run using script: \verb|counterfact_mnist_02.py| 

Architecture:
\newline
$\x \in \mathbb{R}^{28,28,1}$ $\to$ 32$\times$(5,5) conv $\to$ 64$\times$(3,3) conv $\to$ (2,2) pooling $\to$ 128$\times$(3,3) conv $\to$ (2,2) pooling $\to$ $H=16$ fc $\to$ softmax $K=3$,  ReLU activations, trained for 2 epochs

For the stack structure, we deviated slightly from this. To avoid the optimisation failing, we needed to evolve the weight vectors during training -- see code for details. All other structures were fixed from initialisation.

Fig. \ref{fig_weird_structures_pca} shows PCA visualisations (similar to those described in section \ref{sec_app_exp_miscfigures}) relating to fig. \ref{fig_counterfact_structs}.

\begin{figure}[h!]
\begin{center}
\includegraphics[width=0.3\columnwidth]{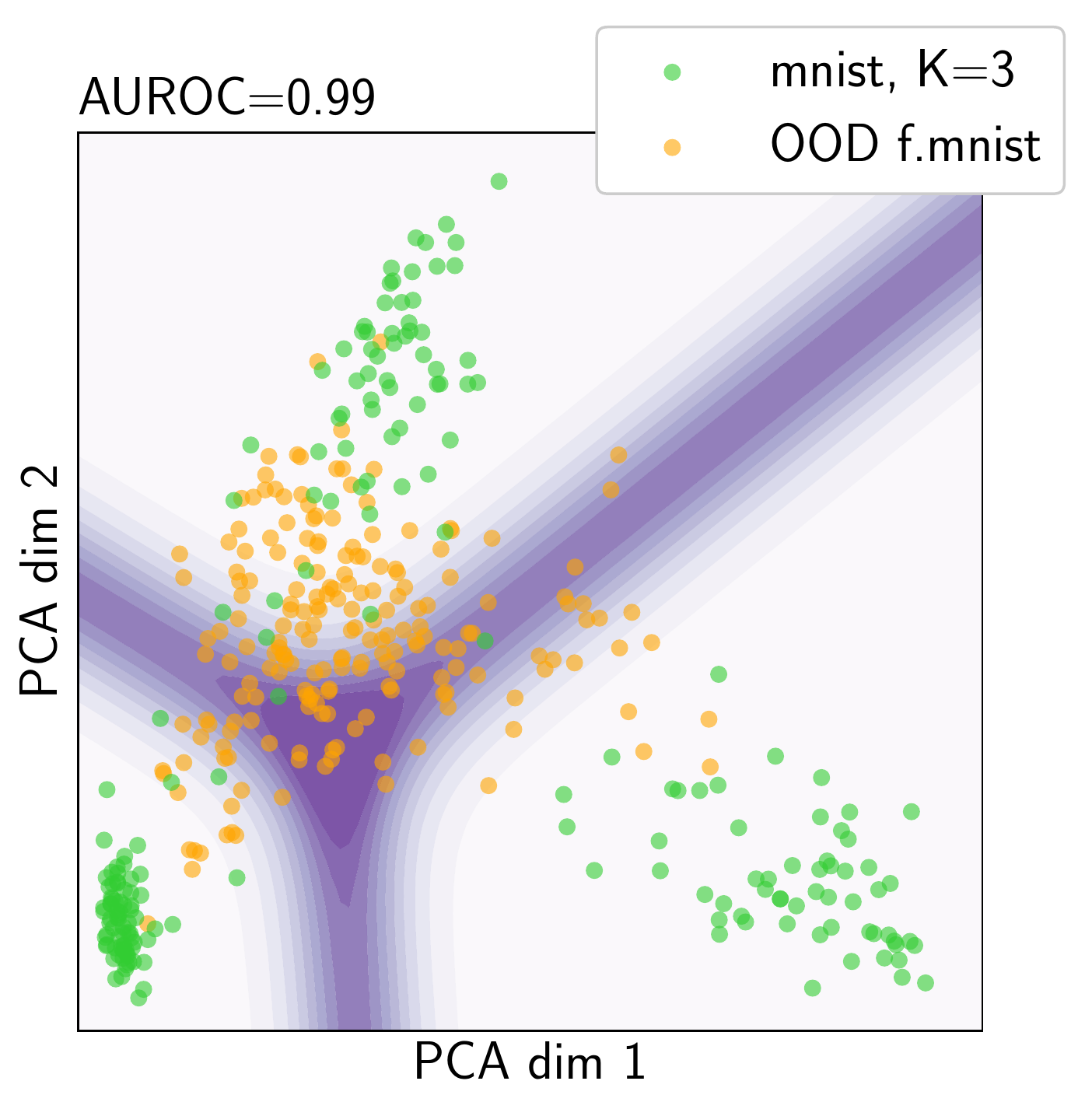}
\put(-80,-10){\small 1) Trainable}
\includegraphics[width=0.3\columnwidth]{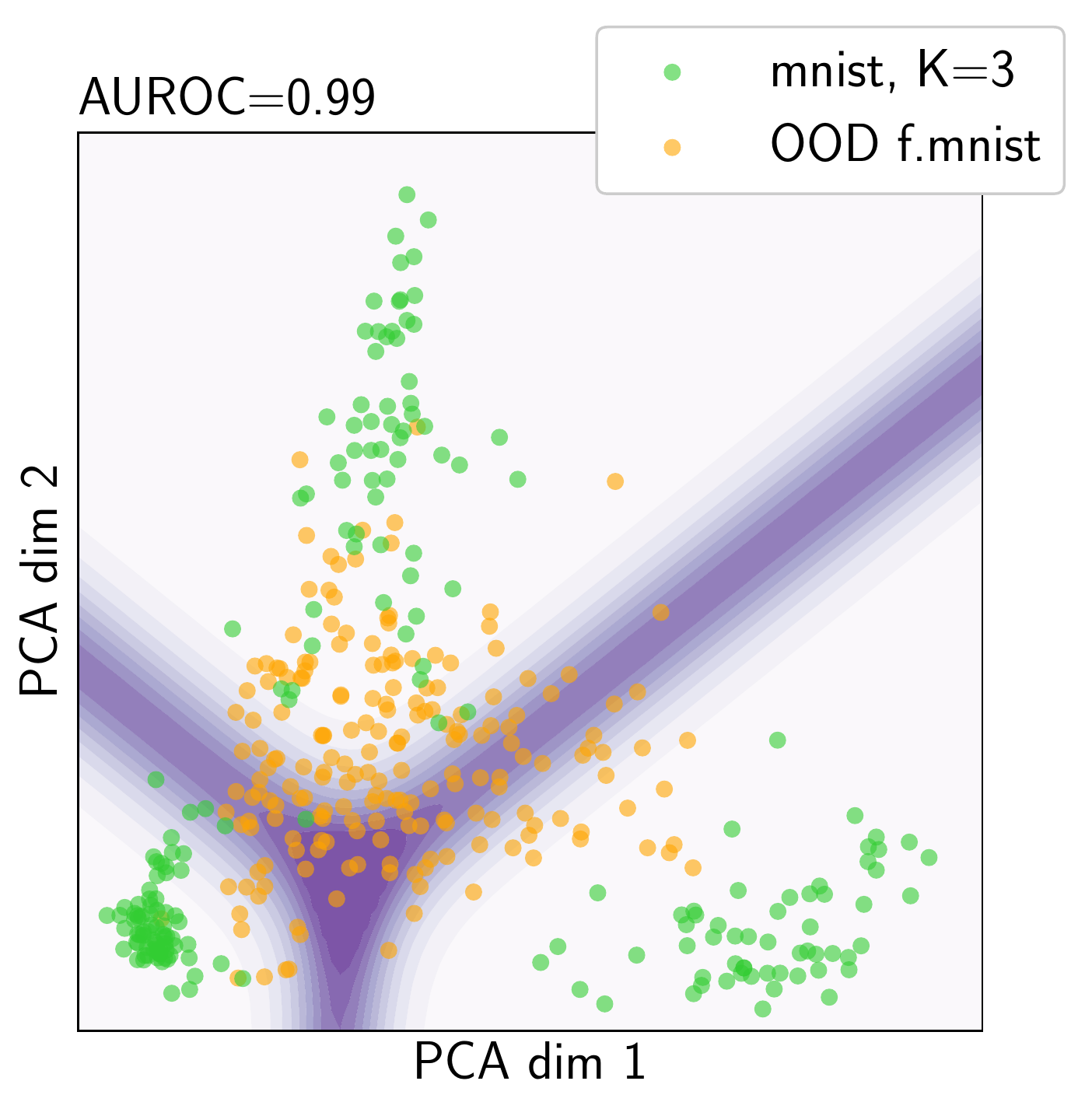}
\put(-80,-10){\small 2) Optimal}
\includegraphics[width=0.3\columnwidth]{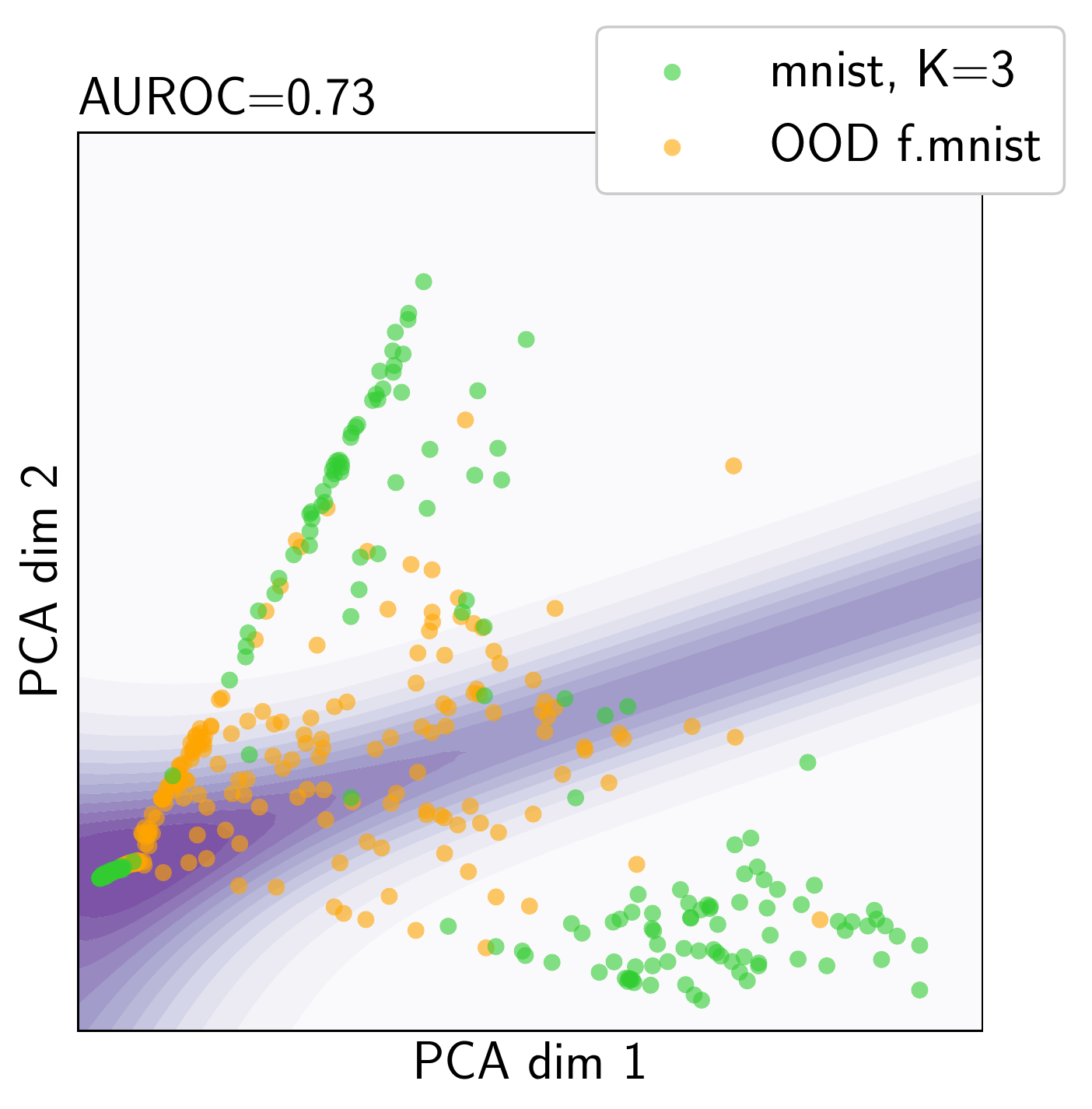}
\put(-80,-10){\small 3) Sandwich}

\includegraphics[width=0.3\columnwidth]{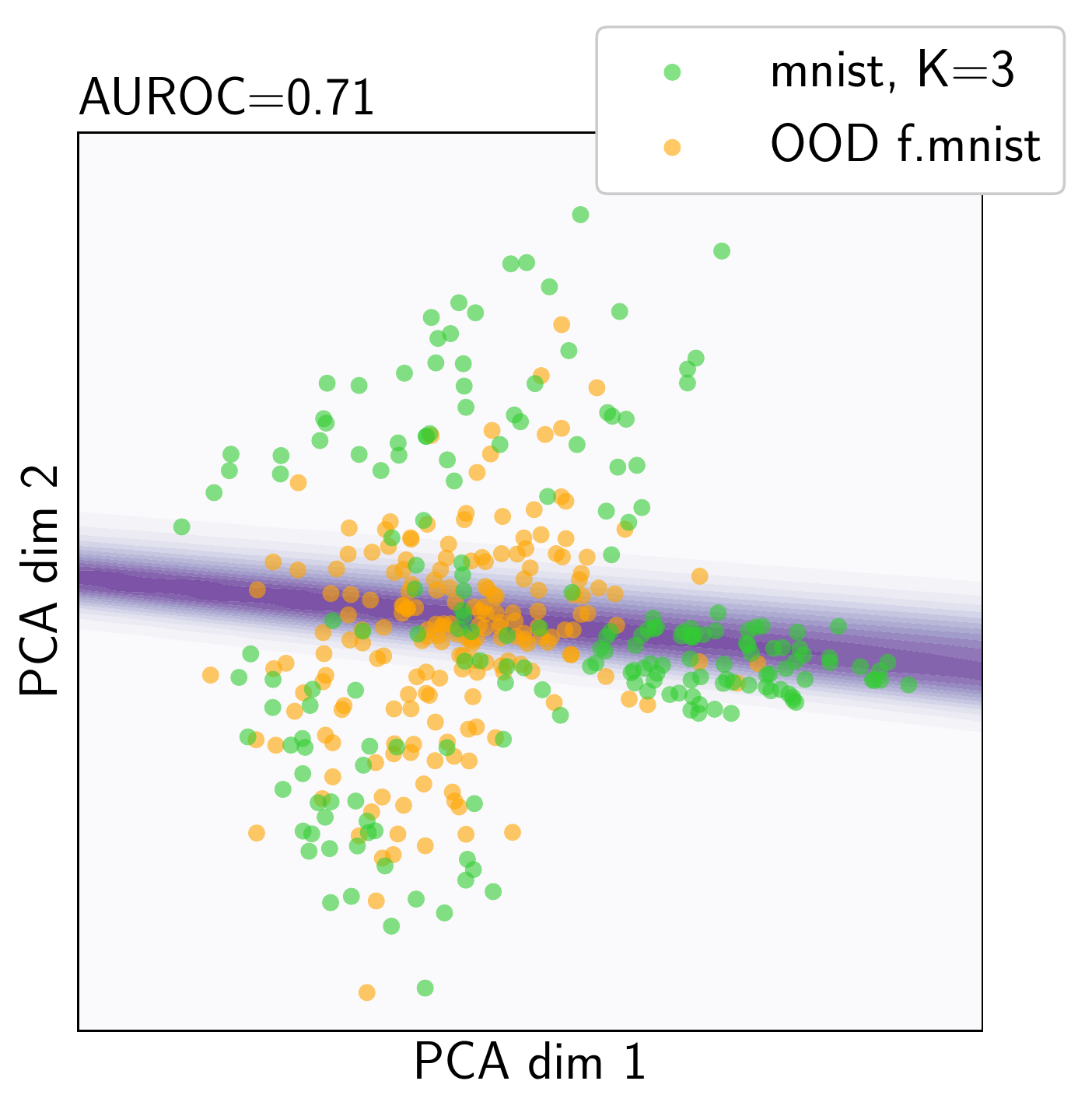}
\put(-80,-10){\small 4) Stack}
\includegraphics[width=0.3\columnwidth]{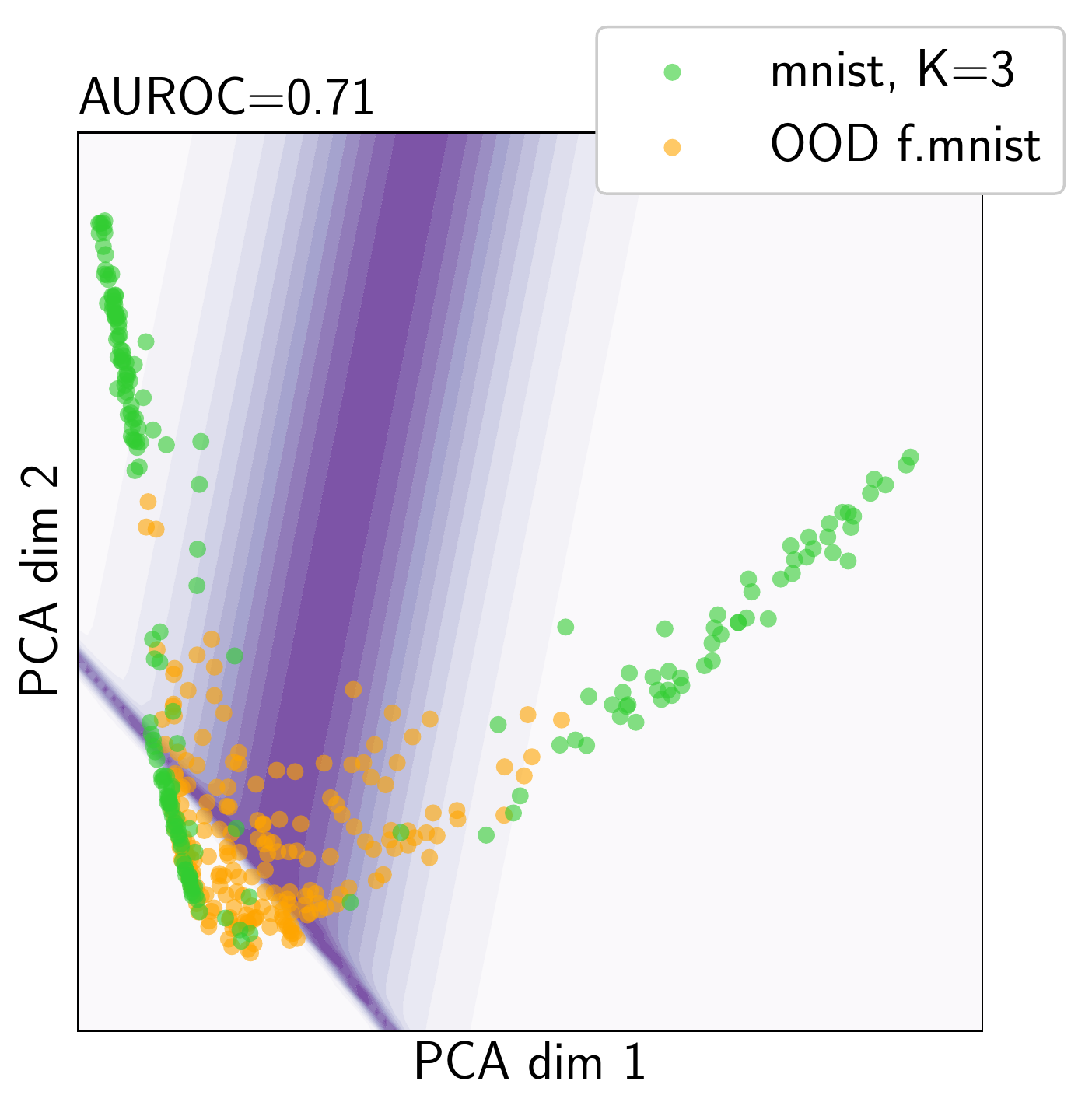}
\put(-80,-10){\small 5) Lopsided}
\vspace{0.2in}
\end{center}
\vskip -0.1in
\caption{Relating to fig. \ref{fig_counterfact_structs}, we provide PCA visualisations in the final-hidden layer for each structure tested.}
\label{fig_weird_structures_pca}
\vskip -0.1in
\end{figure}

\subsection{Filtering Experiments}
\label{sec_app_exp_filter_magnitude_egs}

\subsubsection{Visualising Activations}

Relating to experiment: Section \ref{sec_filter_empirical}, fig. \ref{fig_activation_mnsit_fashion}

Run using script: \verb|bottleneck_mnist_10_interp_04.py| 

Architecture:
\newline
$\x \in \mathbb{R}^{28,28,1}$ $\to$ 32$\times$(5,5) conv $\to$ 64$\times$(3,3) conv $\to$ (2,2) pooling $\to$ 128$\times$(3,3) conv $\to$ (2,2) pooling $\to$ $H=64$ fc $\to$ softmax $K=3$,  ReLU activations, trained for 40 epochs

We used a simple algorithm to order the activations before displaying, grouping by which class each final-layer neurons was most activated for.

\subsubsection{Literal Filter Interpretation}

Relating to experiment: Section \ref{sec_filter_empirical}, fig. \ref{fig_noise}

Run using script: \verb|permute_test_04_mac_02.py| 

Architecture:
\newline
$\x \in \mathbb{R}^{9,9}$ $\to$ 16$\times$(3,3) conv $\to$ 32$\times$(3,3) conv $\to$ (2,2) pooling $\to$ 64$\times$(2,2) conv $\to$ softmax $K=10$,  ReLU activations, trained for 15 epochs

\subsubsection{Effect of Depth on Filtering}

Relating to experiment: Section \ref{sec_filter_depth_exp}, fig. \ref{fig_depth_experiment}

Run using script: \verb|bottleneck_mnist_10_windows.py| 

Example architecture for 8 convolutional layers (4x VGG blocks):
\newline
$\x \in \mathbb{R}^{28,28,1}$ 
$\to$ 32$\times$(3,3) conv $\to$ 32$\times$(3,3) conv $\to$ (2,2) pooling $\to$ batchnorm
$\to$ 64$\times$(3,3) conv $\to$ 64$\times$(3,3) conv $\to$ (2,2) pooling $\to$ batchnorm
$\to$ 128$\times$(3,3) conv $\to$ 128$\times$(3,3) conv $\to$ (2,2) pooling $\to$ batchnorm
$\to$ 128$\times$(3,3) conv $\to$ 128$\times$(3,3) conv $\to$ (2,2) pooling $\to$ batchnorm
$\to$ dropout $\to$ 512 fc $\to$ softmax $K=10$
,  ReLU activations, trained for 12 epochs

\subsection{Standard OOD Detection Benchmarking}
\label{sec_app_exp_standard_bench}

Relating to experiment: Section \ref{sec_failure_measure}, table \ref{table_AUROC_bench}

Run using script:
Models were trained for 100 epochs using script from
\textcolor{blue}{\url{https://github.com/huyvnphan/PyTorch_CIFAR10}}, which uses the official pytorch ResNet definitions 
\textcolor{blue}{\url{https://github.com/pytorch/vision/blob/master/torchvision/models/resnet.py}}. We made only minor modifications to take different image datasets. 

Comparison of GMM vs Softmax: \verb|grid_run_05.py| 
 
Results display: \verb|grid_display_03.py|

Fig. \ref{fig_grid_auroc} provides a breakdown of AUROC over individual datasets, the averages are given in table \ref{table_AUROC_bench}. We balance all datasets when computing AUROC.


We experimented extensively to tune the fitting of the GMM as far as possible, optimising for AUROC results using $U_\text{density}$. We performed random search over the first four variables below, and a full grid search over the last three.
\begin{list}{•}{}
\item When appropriate, constrict one GMM component to model each class, or allow flexible fitting via expectation maximisation (EM)
\item Number of final-layer neurons used (randomly selected) $\in [100,256,512]$ -- note for ResNet18, $H=512$
\item Diagonal covariance regularisation $\in [1e-1,1e-3,1e-5,1e-7]$
\item Fit on $\z$ before or after non-linearity is applied
\item Covariance type $\in[ $full , diagonal , isotropic$]$
\item Fit on $\z$ or $\log \z$
\item Number components $\in [1,K,10,100]$
\end{list}
The final parameters providing the best fit are as follows.
\begin{list}{•}{}
\item GMM fit with EM
\item Number of final-layer neurons used = $512$ 
\item Diagonal covariance regularisation = $1e-5$
\item Fit on $\z$ after non-linearity is applied
\item Covariance type = full
\item Fit on $\z$
\item Number components = $K$
\end{list}
We further experimented using kernel density estimation (KDE) instead of GMM, but this gave consistently worse performance even after similar tuning efforts.

\begin{figure}[t]
\vskip 0.05in
\begin{center}

AUROC 
\vspace{0.1in}

\includegraphics[width=0.2\columnwidth]{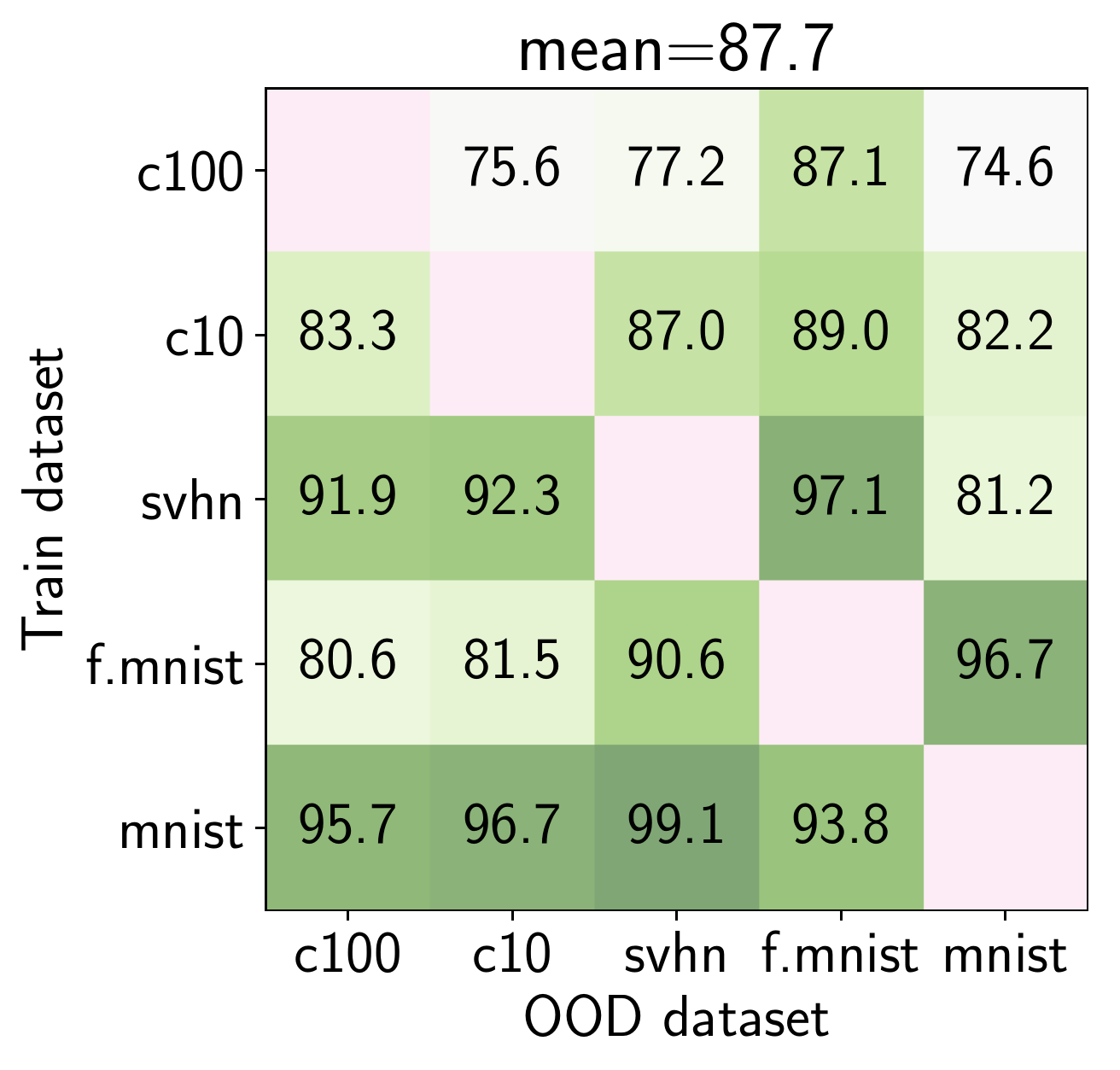}
\put(-60,80){\small Softmax, $U_\text{max}$}
\includegraphics[width=0.2\columnwidth]{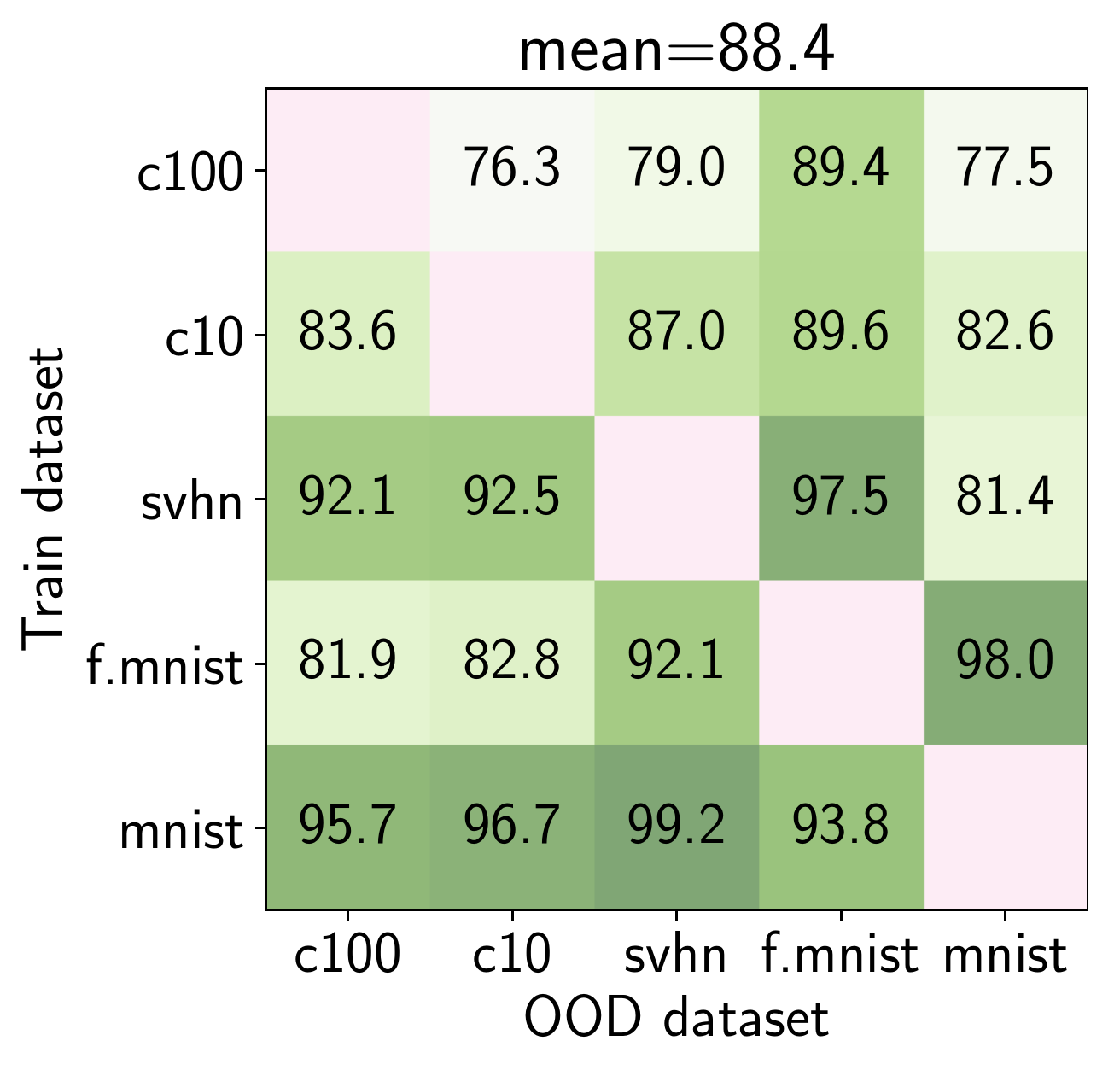}
\put(-60,80){\small Softmax, $U_\text{entropy}$}
\includegraphics[width=0.2\columnwidth]{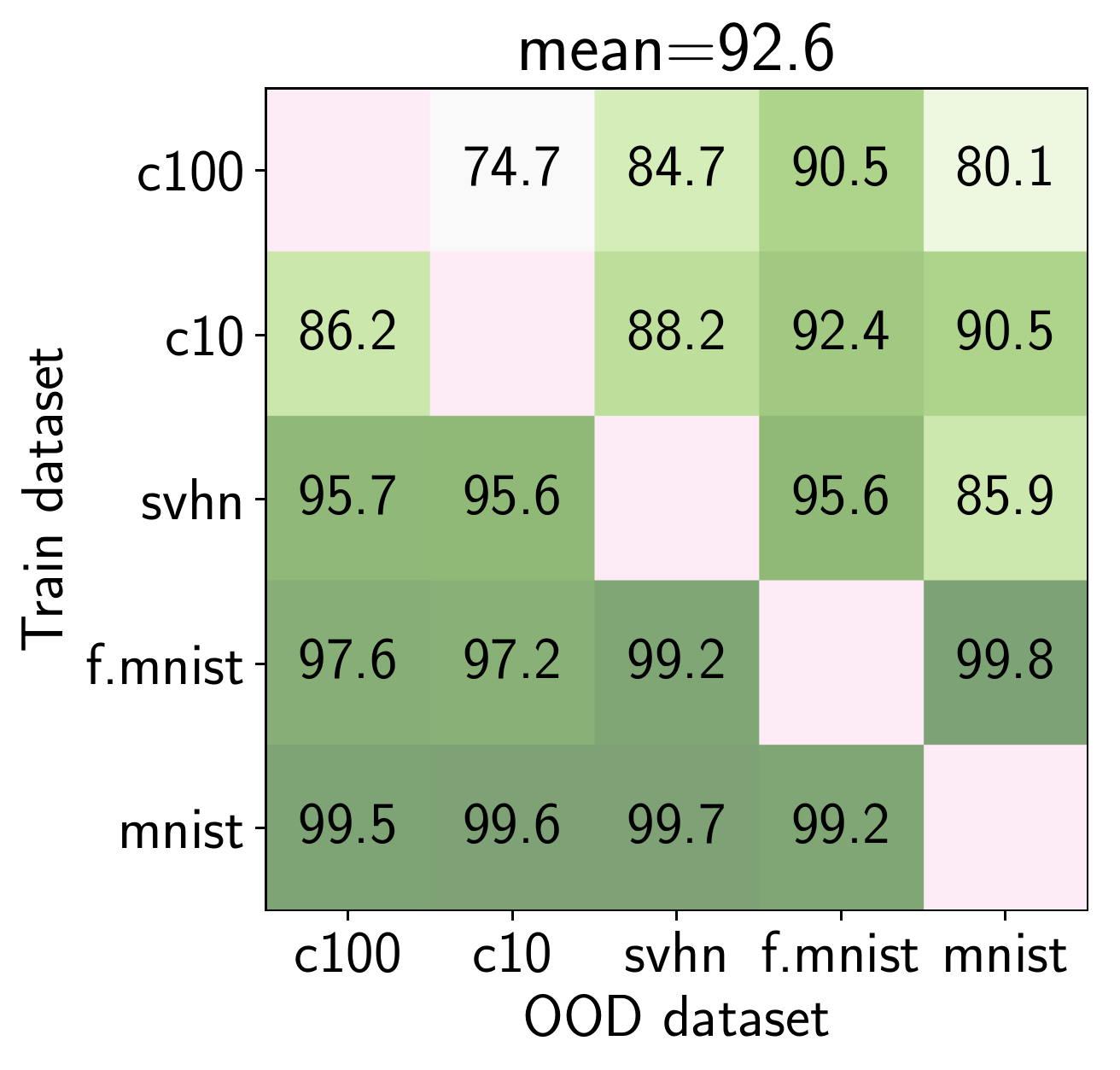}
\put(-60,80){\small Density, $U_\text{density}$}
\includegraphics[width=0.2\columnwidth]{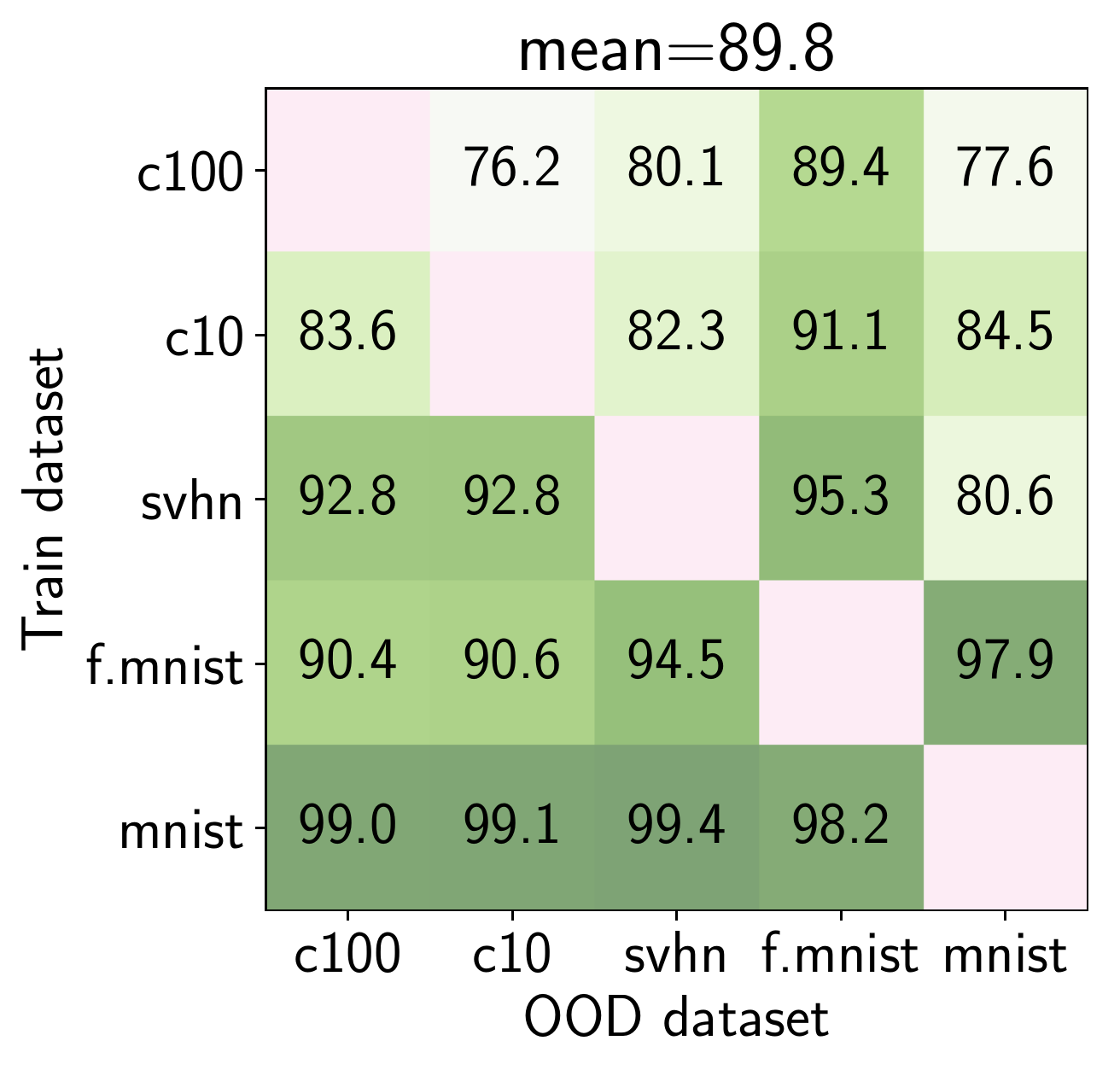}\put(-70,80){\small Softmax, $\max \cos \theta_{i,\z}$}
\includegraphics[width=0.2\columnwidth]{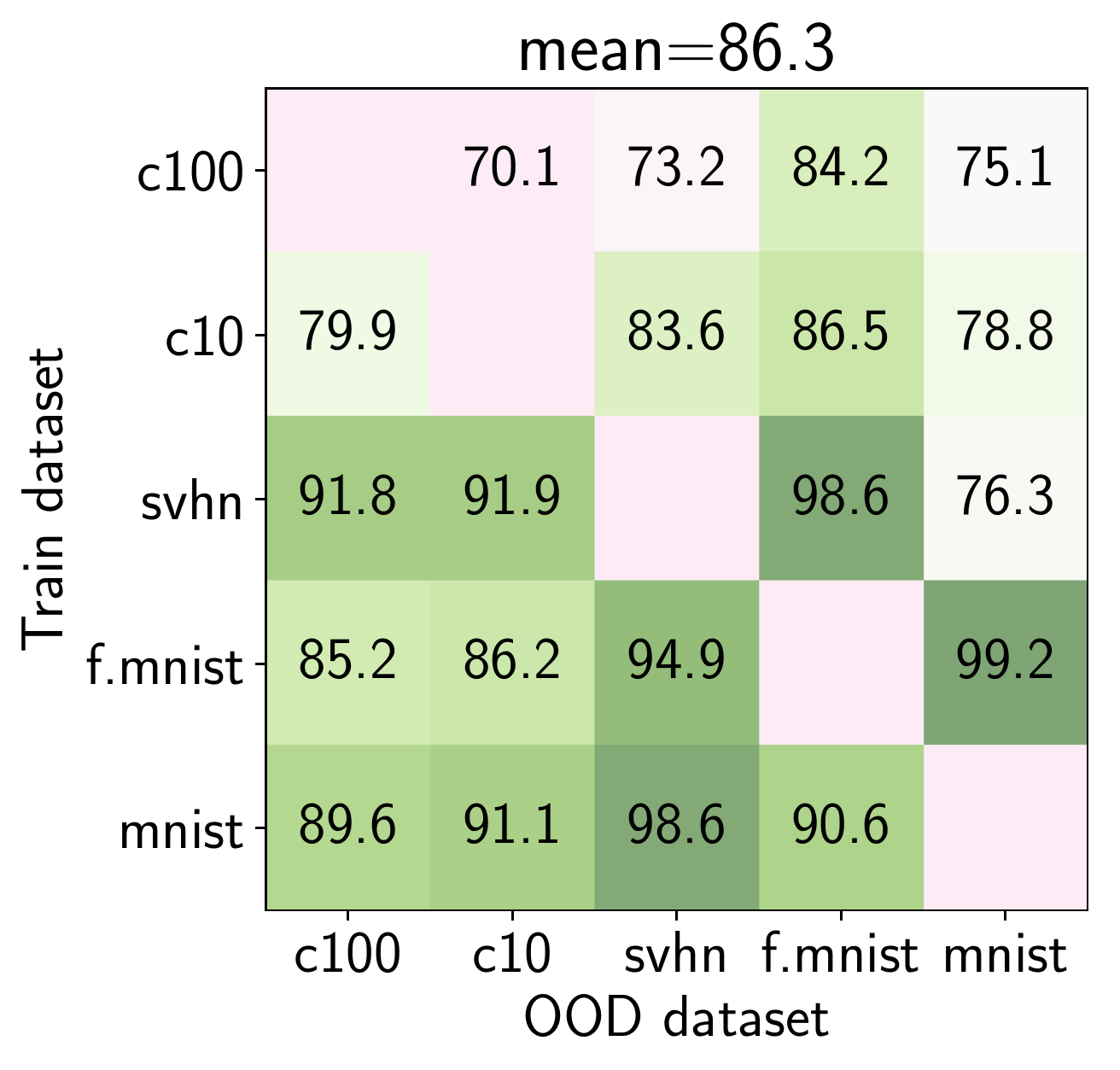}\put(-60,80){\small Softmax, $||\z||$}

%

\caption{Full permutations of training datasets, with each dataset used as OOD. Numbers are AUROC for an evenly balanced combination of train and OOD datasets, averaged over five runs. E.g. for a ResNet18 trained on CIFAR 10, the mean AUROC against MNIST using $U_\text{max}$ is 82.2\%.}
\label{fig_grid_auroc}
\end{center}
\vskip -0.2in
\end{figure}

\subsection{Pre-training/Fine-tuning OOD Detection}
\label{sec_app_exp_pretrain_finetune}

Relating to experiment: Section \ref{sec_failure_measure}, table \ref{table_AUROC_bench}

Run using script: \verb|transfer_wind_12.py|, \verb|read_results_transfer_02.py|

Details of datasets used are given in table \ref{table_transferdatasets}.

We conducted a tuning process similar as in section \ref{sec_app_exp_standard_bench}. 
The final parameters providing the best fit are as follows.
\begin{list}{•}{}
\item GMM fit with EM
\item Number of final-layer neurons used = $1280$ 
\item Diagonal covariance regularisation = $1e-3$
\item Fit on $\z$ after non-linearity is applied
\item Covariance type = full
\item Fit on $\log \z$
\item Number components = $K$
\end{list}

\begin{table}[h!]
  \caption{Transfer dataset information.}
  \vspace{-0.1in}
  \label{table_transferdatasets}
  \begin{center}
  \resizebox{0.8\textwidth}{!}{
  \begin{tabular}{l c ccccc l}
    \toprule
   Dataset & Crop size & $N$ train data & $N$ test data & $K$ classes & Last-layer eps & Full fine-tune eps & Hyperlink \\
    \midrule
Satellite & (224, 224) & 1600 & 250 & 21 & 16 & 32 & \textcolor{blue}{\href{https://www.tensorflow.org/datasets/catalog/uc_merced}{uc merced}} \\ 
Cancer & (150, 150) & 4000 & 1000 & 8 & 16 & 32 & \textcolor{blue}{\href{https://www.tensorflow.org/datasets/catalog/colorectal_histology}{colorectal histology}} \\ 
Pets & (224, 224) & 3680 & 1000 & 37 & 16 & 32 & \textcolor{blue}{\href{https://www.tensorflow.org/datasets/catalog/oxford_iiit_pet}{oxford iiit pet}} \\ 
Flowers & (224, 224) & 2040 & 1000 & 102 & 16 & 64 & \textcolor{blue}{\href{https://www.tensorflow.org/datasets/catalog/oxford_flowers102}{oxford flowers102}} \\ 
Beans & (224, 224) & 534 & 500 & 3 & 16 & 32 & \textcolor{blue}{\href{https://www.tensorflow.org/datasets/catalog/beans}{beans}} \\ 
   \bottomrule
  \end{tabular}
  }
\end{center}
\end{table}

\newpage
\section{Plots}

\begin{figure}[h]
\begin{center}
\includegraphics[width=0.3\columnwidth]{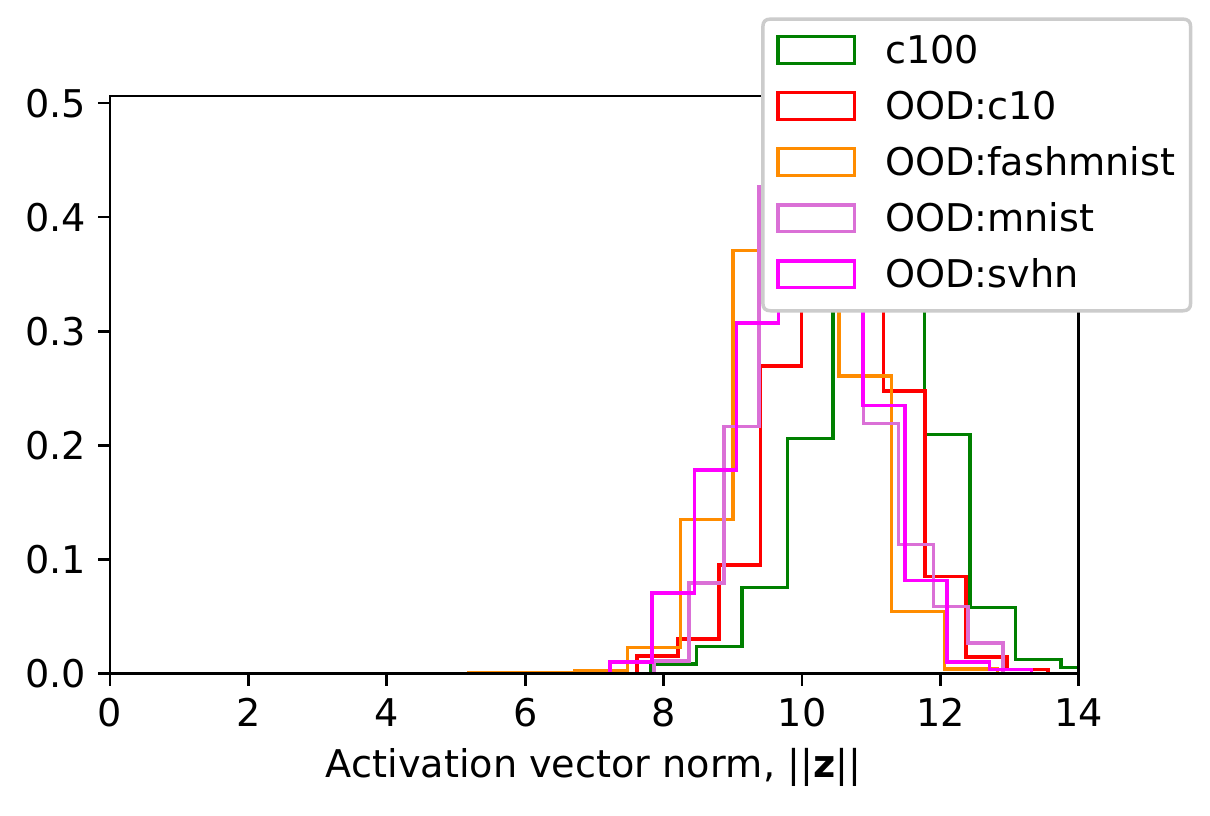}
\includegraphics[width=0.3\columnwidth]{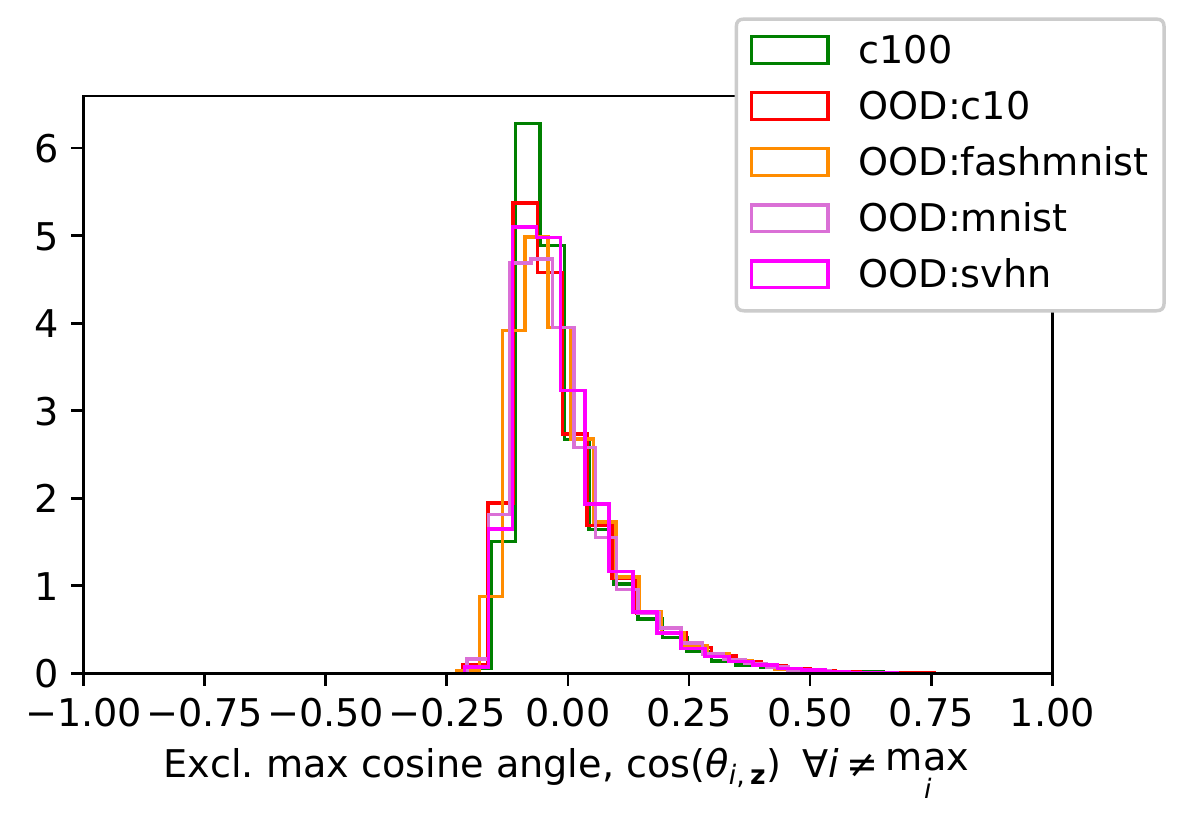}
\includegraphics[width=0.3\columnwidth]{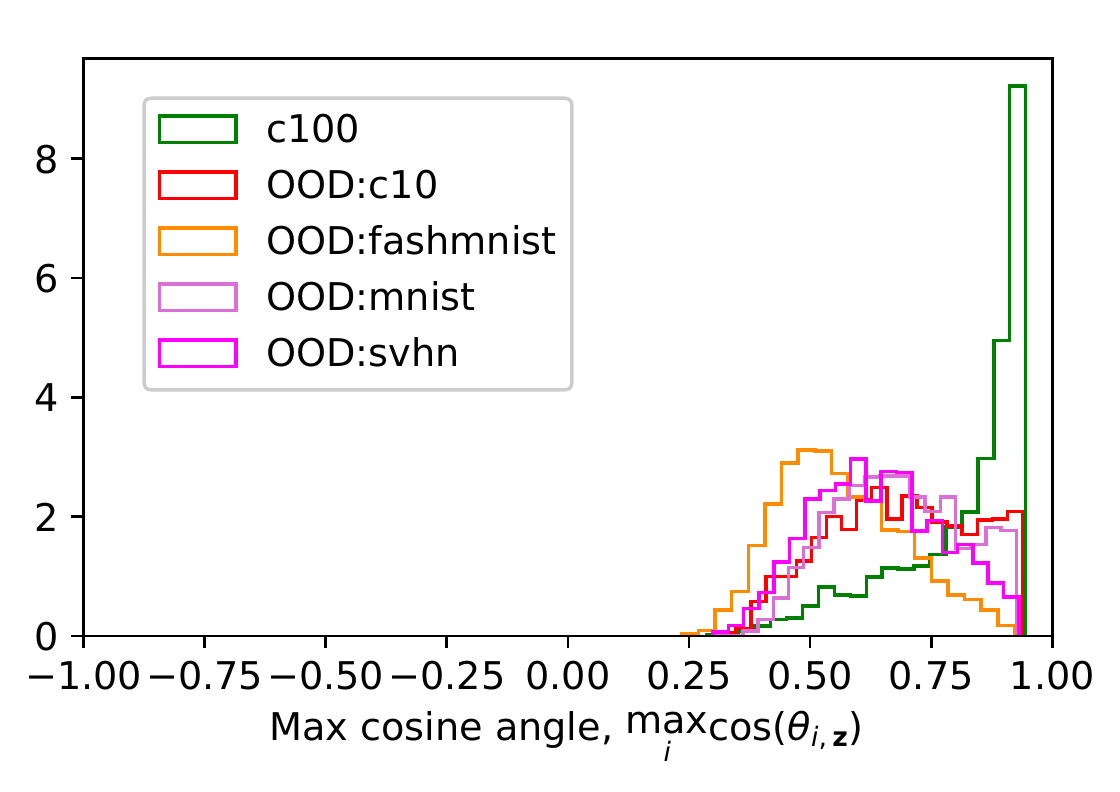}
\put(-204,83){\small $\mathcal{D}_\text{in} = $ Cifar 100}
\vspace{0.2in}

\includegraphics[width=0.3\columnwidth]{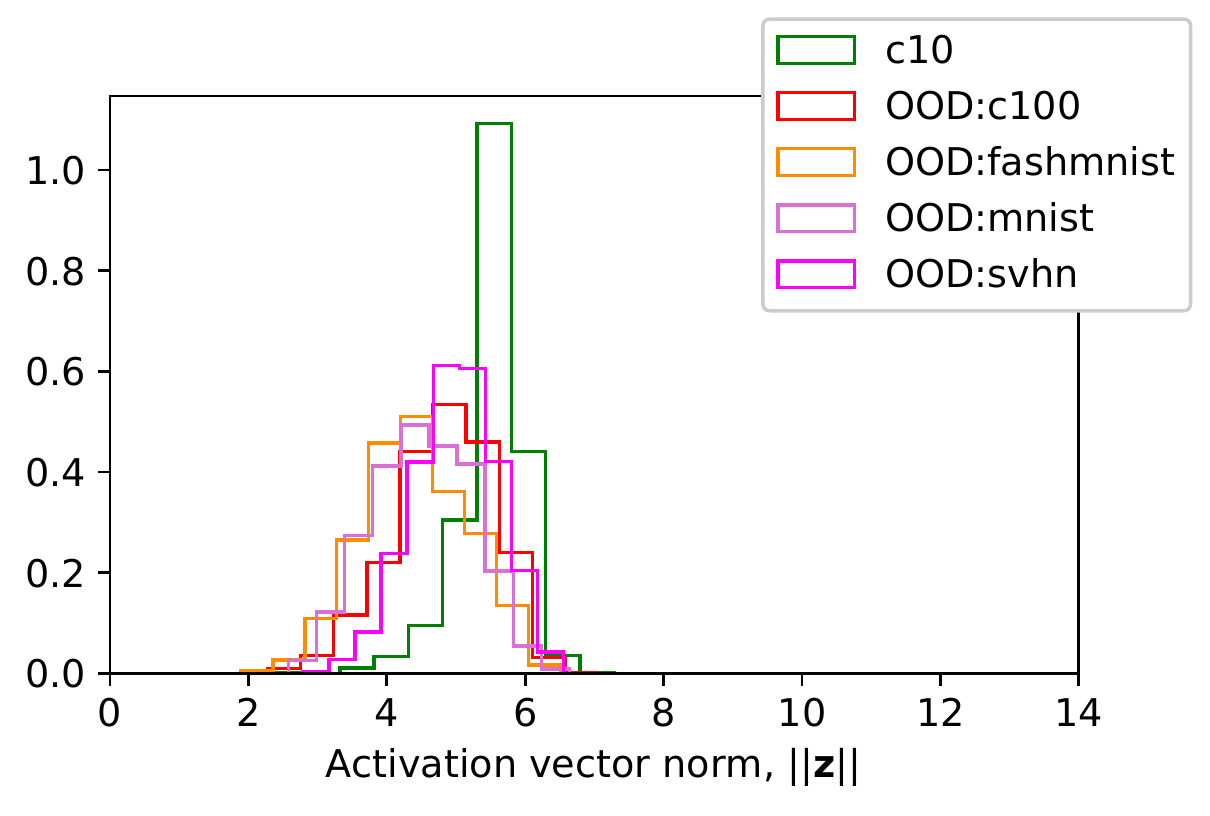}
\includegraphics[width=0.3\columnwidth]{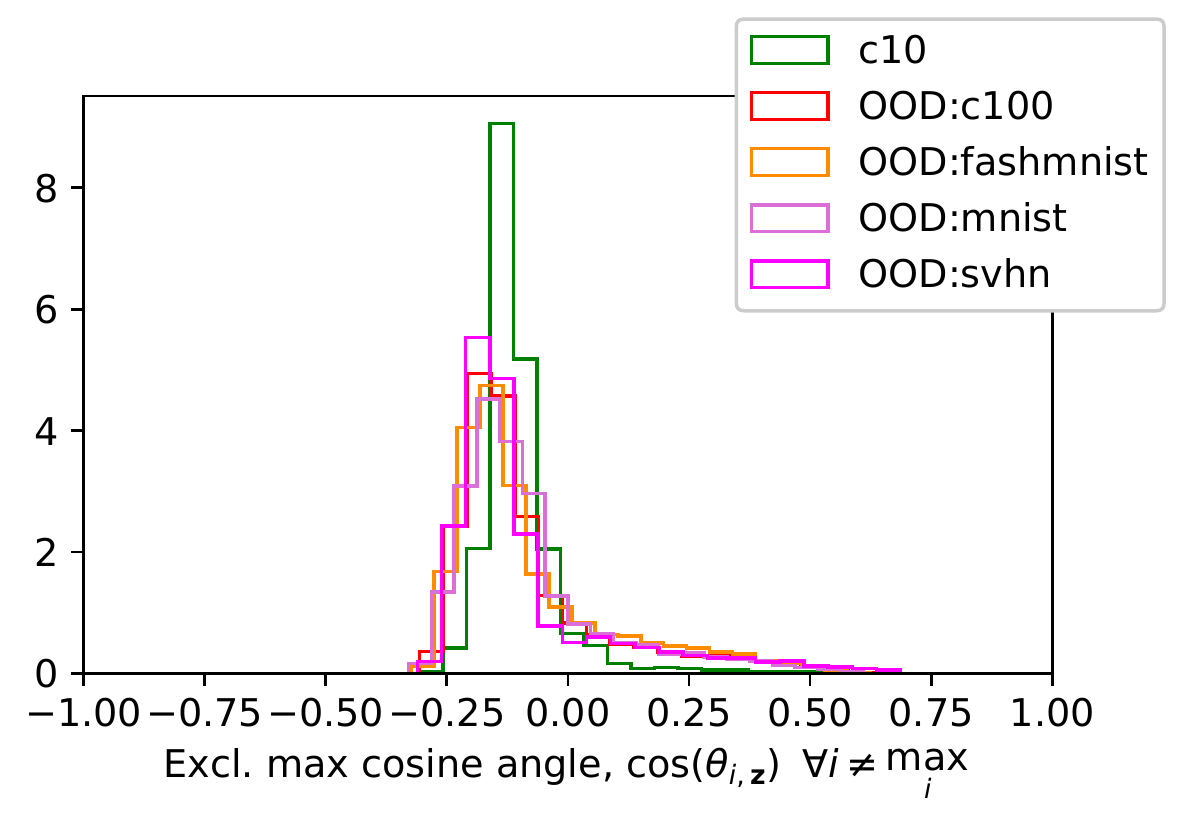}
\includegraphics[width=0.3\columnwidth]{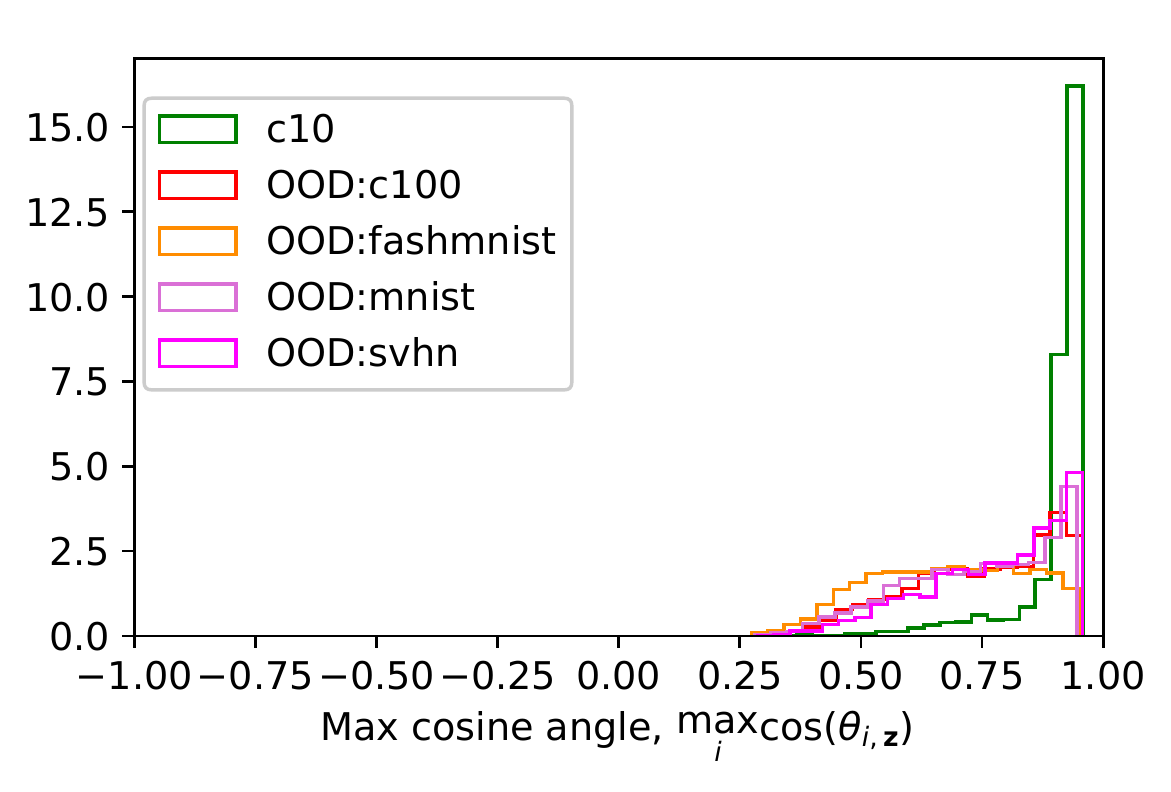}
\put(-204,83){\small $\mathcal{D}_\text{in} = $ Cifar 10}
\vspace{0.2in}

\includegraphics[width=0.3\columnwidth]{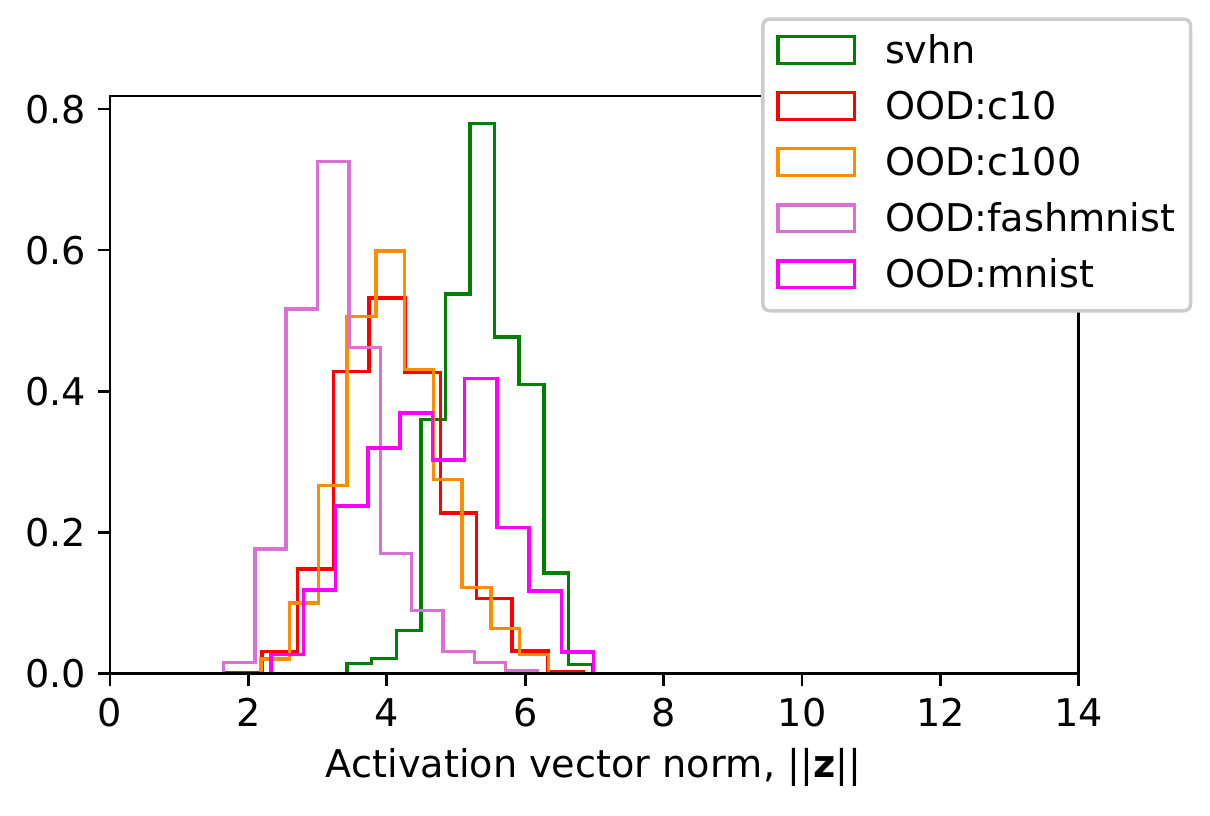}
\includegraphics[width=0.3\columnwidth]{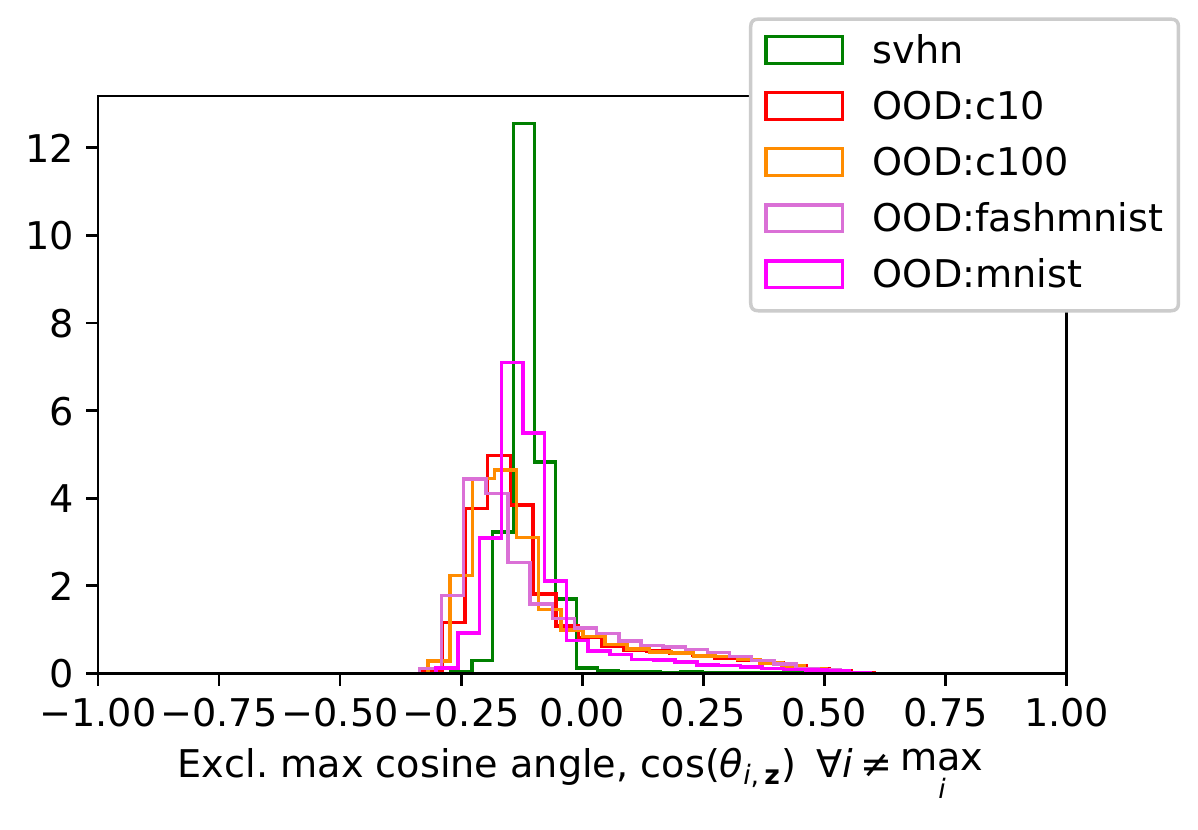}
\includegraphics[width=0.3\columnwidth]{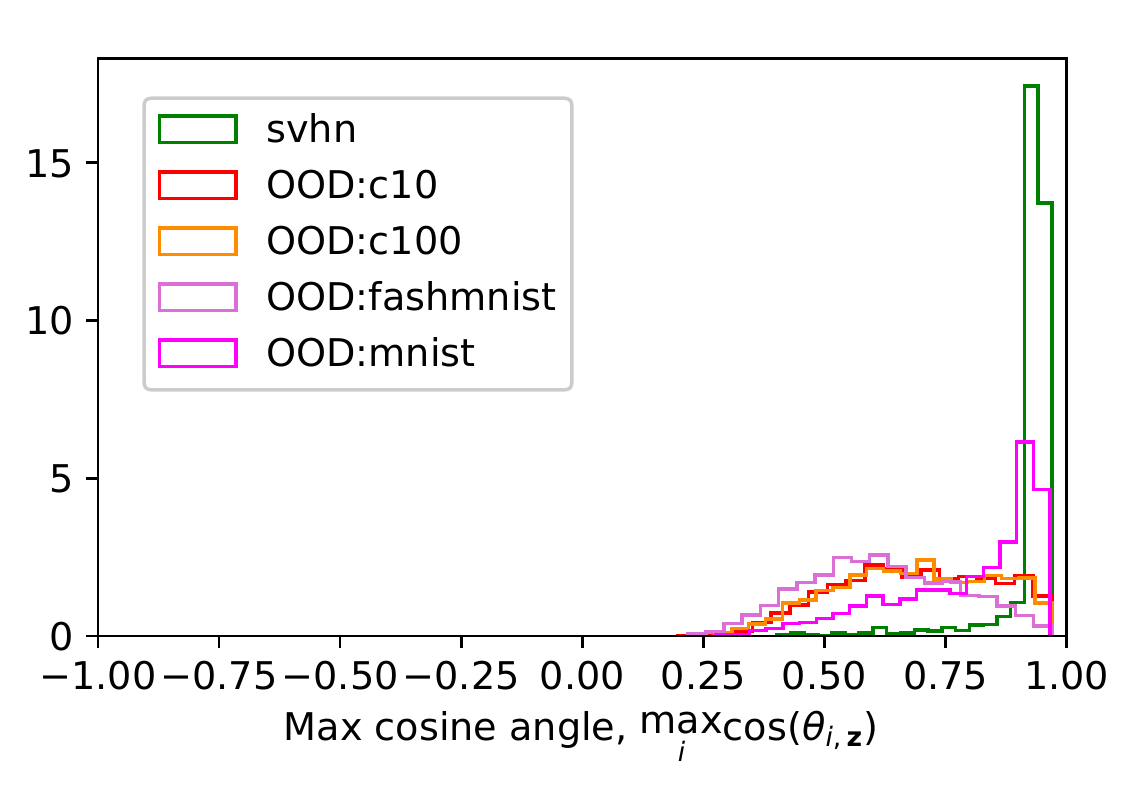}
\put(-204,83){\small $\mathcal{D}_\text{in} = $ SVHN} 

\includegraphics[width=0.3\columnwidth]{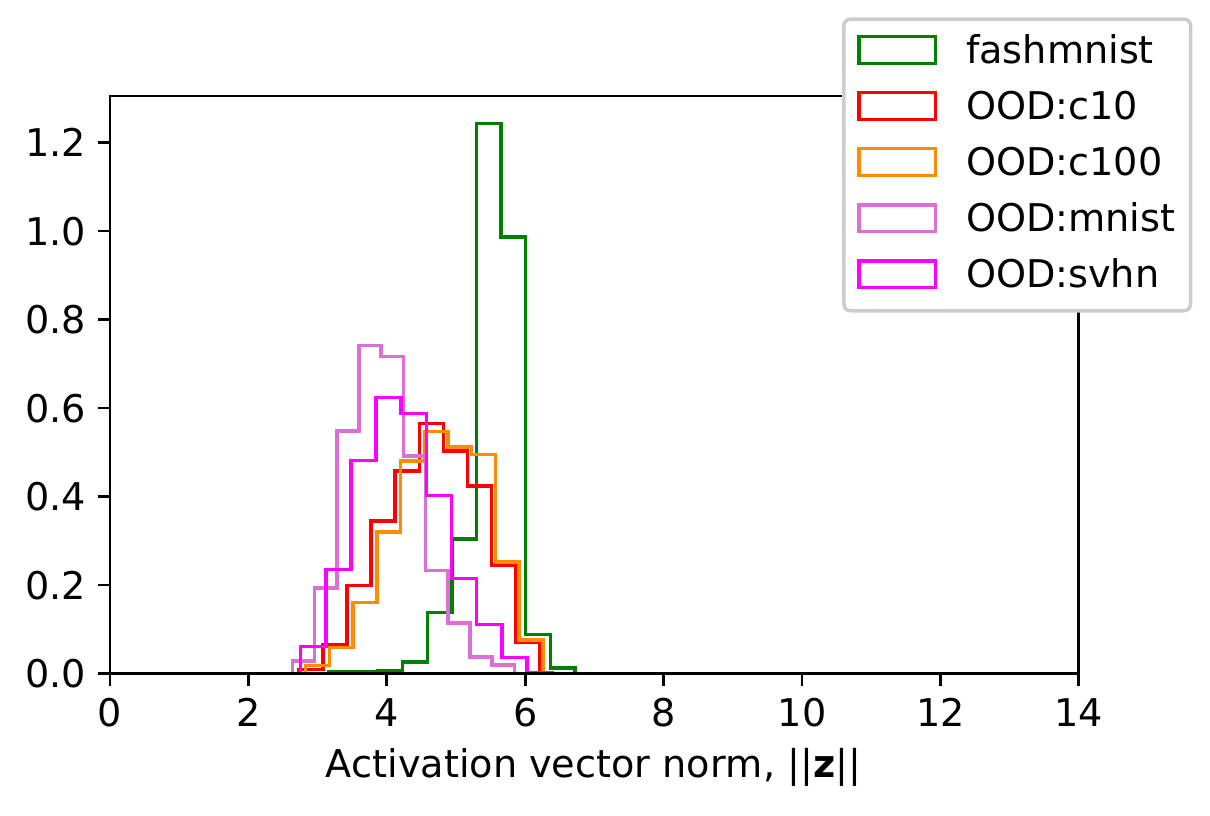}
\includegraphics[width=0.3\columnwidth]{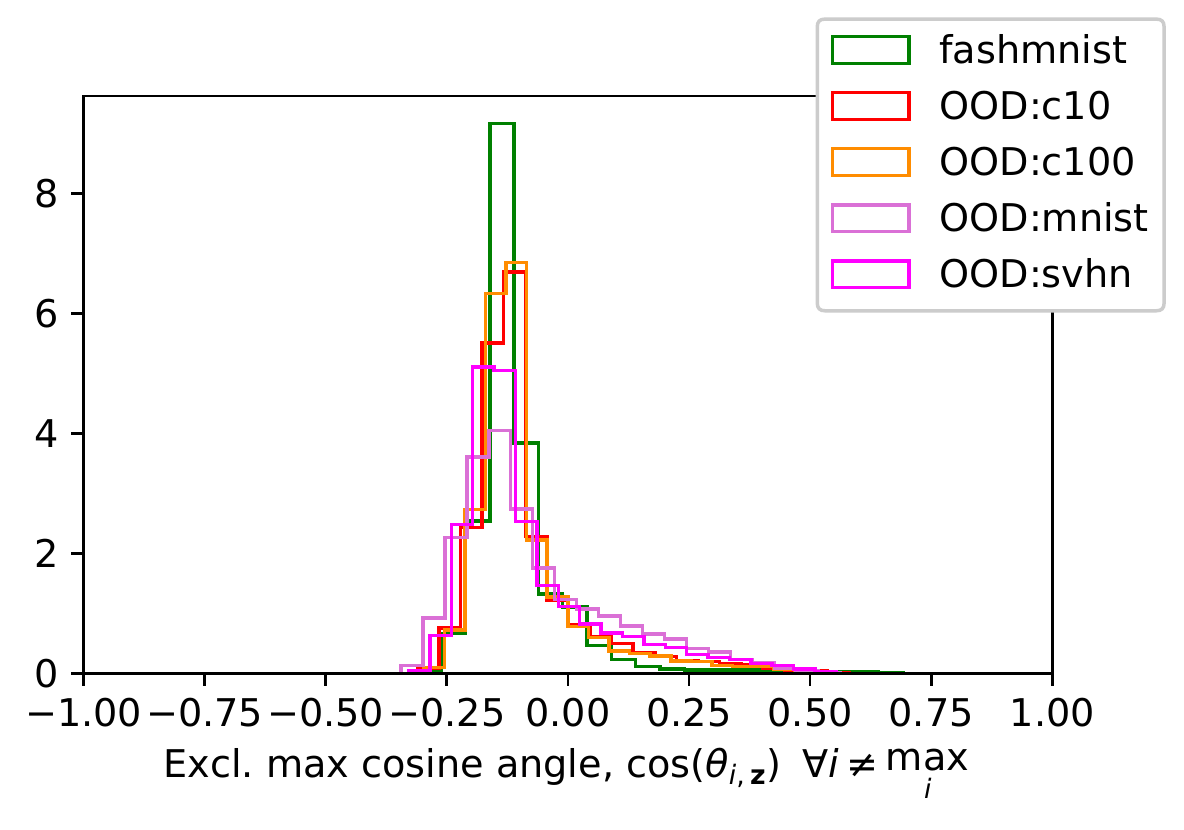}
\includegraphics[width=0.3\columnwidth]{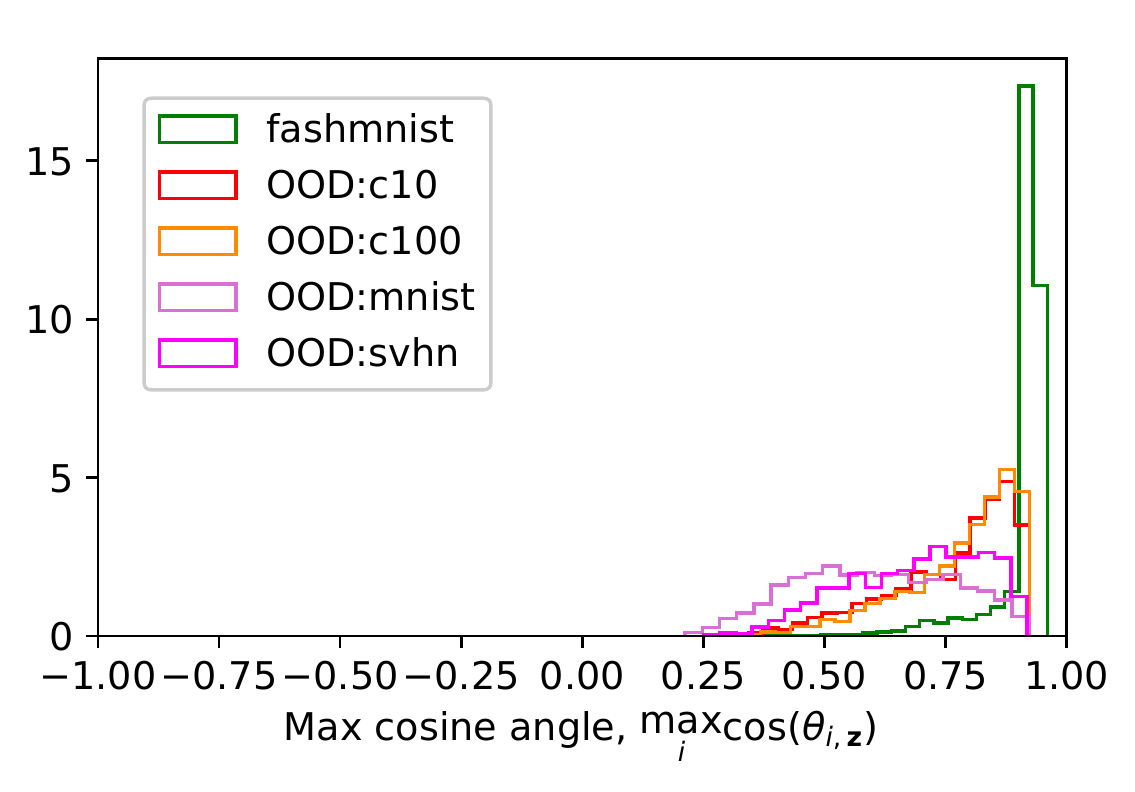}
\put(-204,83){\small $\mathcal{D}_\text{in} = $ Fashion MNIST}
\vspace{0.2in}

\includegraphics[width=0.3\columnwidth]{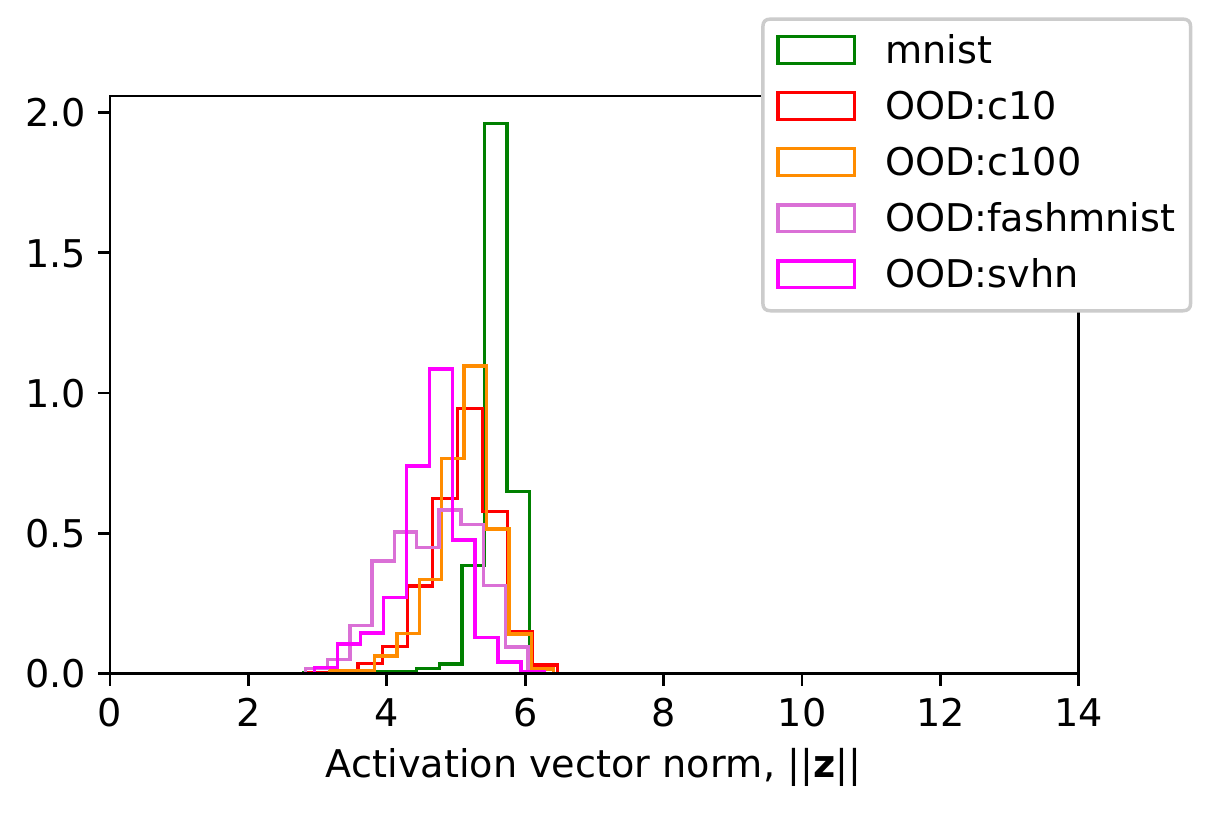}
\includegraphics[width=0.3\columnwidth]{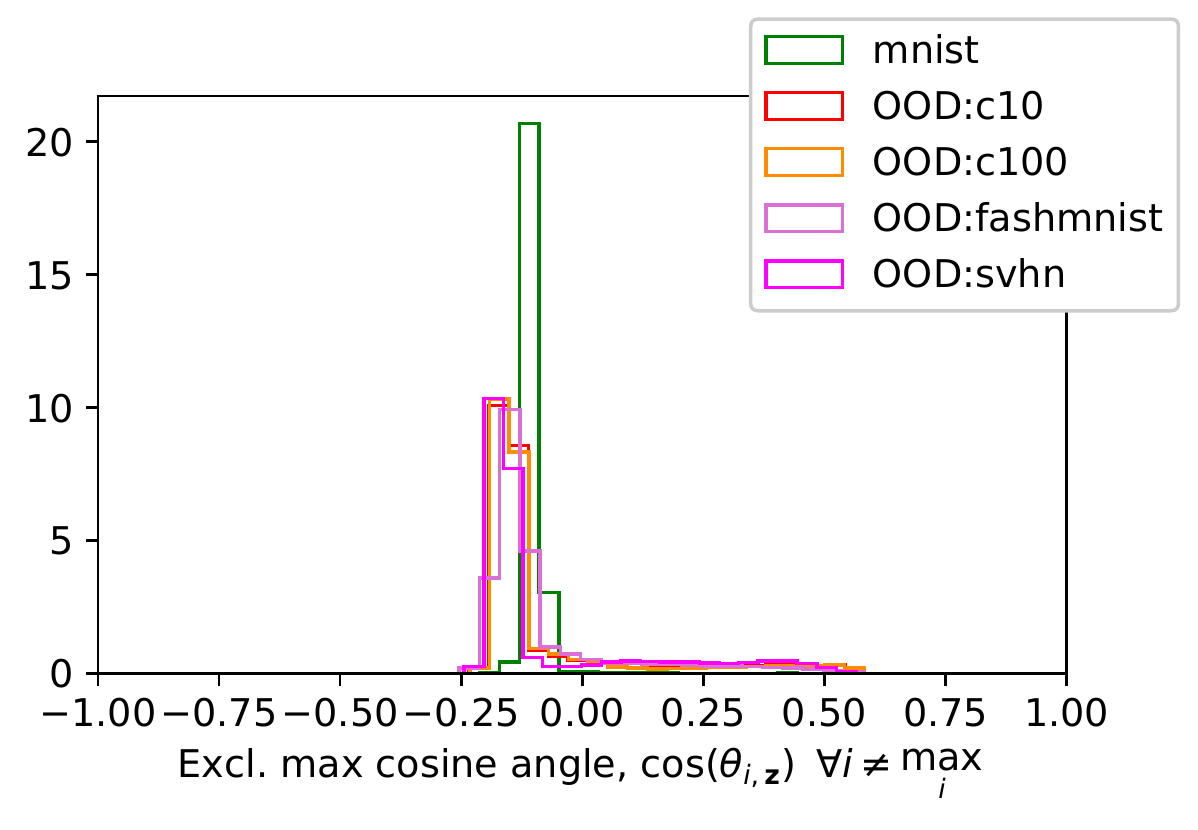}
\includegraphics[width=0.3\columnwidth]{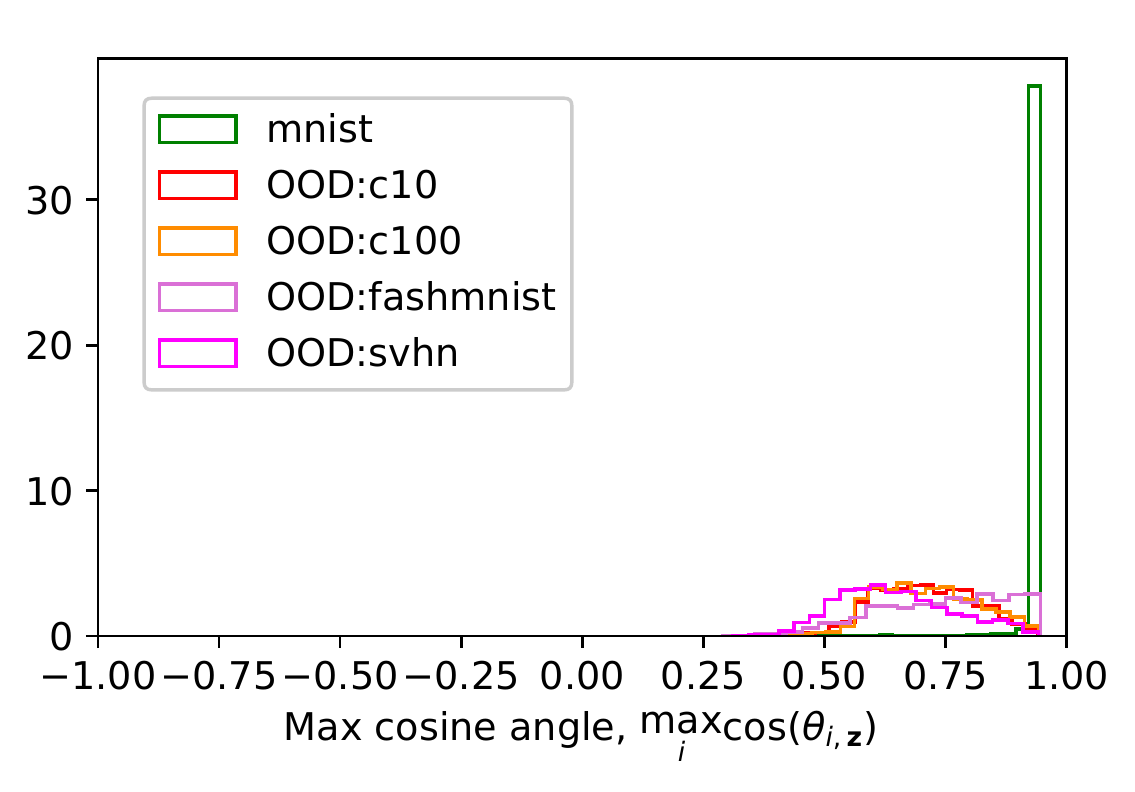}
\put(-204,83){\small $\mathcal{D}_\text{in} = $ MNIST} 
\vspace{0.2in}

\end{center}
\vskip -0.1in
\caption{For ResNet18's trained on $\mathcal{D}_\text{in}$ from experiments in section \ref{sec_failure_measure}, we plot distributions of $||\z||$ (left column), and cosine angles -- both the maximum cosine angle (right column), and all except the maximum cosine angle (middle column).}
\label{fig_cifar_etc}
\vskip -0.1in
\end{figure}

\newpage
\section{Miscellaneous}

\subsection{Code Snippet}
\label{sec_app_codesnippet}
We detail the procedure for computing $U_\text{density}$ through a code snippet.

\begin{lstlisting}[float=tp, language=Python, caption=Code sketch of the method.,captionpos=h!]
from tensorflow import keras
from sklearn.mixture import GaussianMixture

# train model in usual way
model.fit(x_train, y_train)

# extract train data features from final hidden layer
model_z = keras.Model(inputs=model.input, 
                      outputs=model.layers[-2]) 
z_train = model_z.predict(x_train)

# fit density model
gmm = GaussianMixture(n_components).fit(z_train)

# extract test features
z_test = model_z.predict(x_test)

# compute log probability score
uncertainty_estimate = -gmm.score_samples(z_test)
\end{lstlisting}
\label{python_algo}

\subsection{Illustrative Example of Softmax Failure}
Fig. \ref{fig_ankle_boots} plots final-layer activations as in fig. \ref{fig_activation_mnsit_fashion}, but here ankle boots are included in the subset of three Fashion MNIST training classes. MNIST digits are consistently mapped to the same feature space. Note this is not a failure of softmax extrapolations, but due to `feature overlap'.

\begin{figure}[h!]
\begin{center}
\includegraphics[width=0.8\columnwidth]{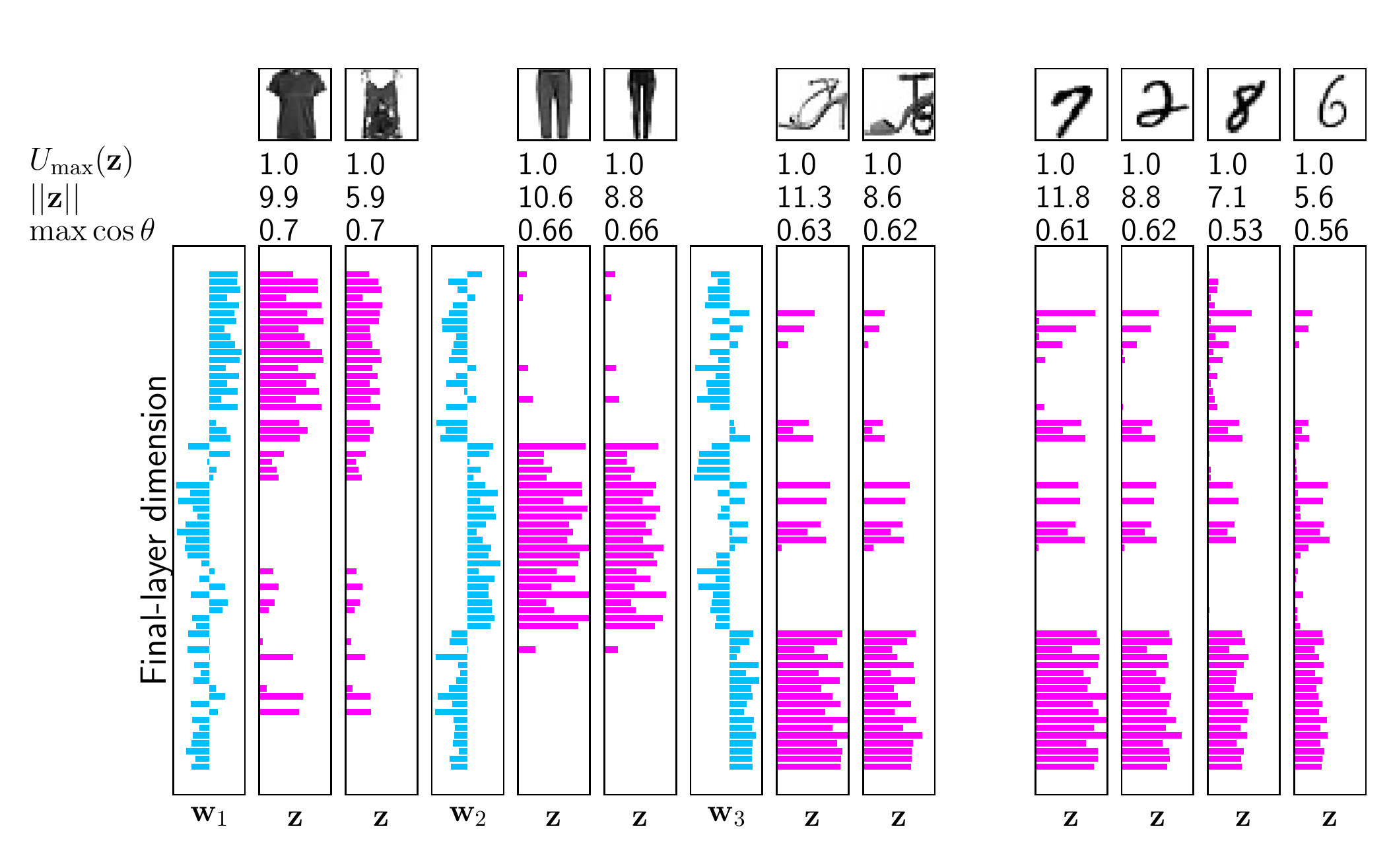}
\end{center}
\vskip -0.1in
\caption{When ankle boots are included in the subset of classes trained on, OOD MNIST digits are consistently mapped into that class with high confidence, leading to poor OOD detection.}
\label{fig_ankle_boots}
\end{figure}

\subsection{The Importance of Depth on the Filtering Effect}
\label{sec_filter_depth_exp}

We hypothesised that the feature filtering (section \ref{sec_misleading_intuition}) effect might be stronger in deeper convolutional networks -- the receptive field of convolutional filters increases with depth, meaning more holistic features are represented \citep{Zeiler2014}. OOD data may be less likely to contain these fuller features. To test this, we present a brief experiment; VGG-style networks with varying numbers of convolutional blocks are trained on all classes of either MNIST or Fashion MNIST, with the other dataset used as OOD data.
Accuracy and AUROC are plotted for each dataset pair and network depth in fig. \ref{fig_depth_experiment}. Whilst accuracy quickly plateaus in these simple tasks, AUROC signifcantly improves in the deeper networks, suggesting depth strenghthens the filtering effect.
\begin{figure}[h!]
\begin{center}
\includegraphics[width=0.3\columnwidth]{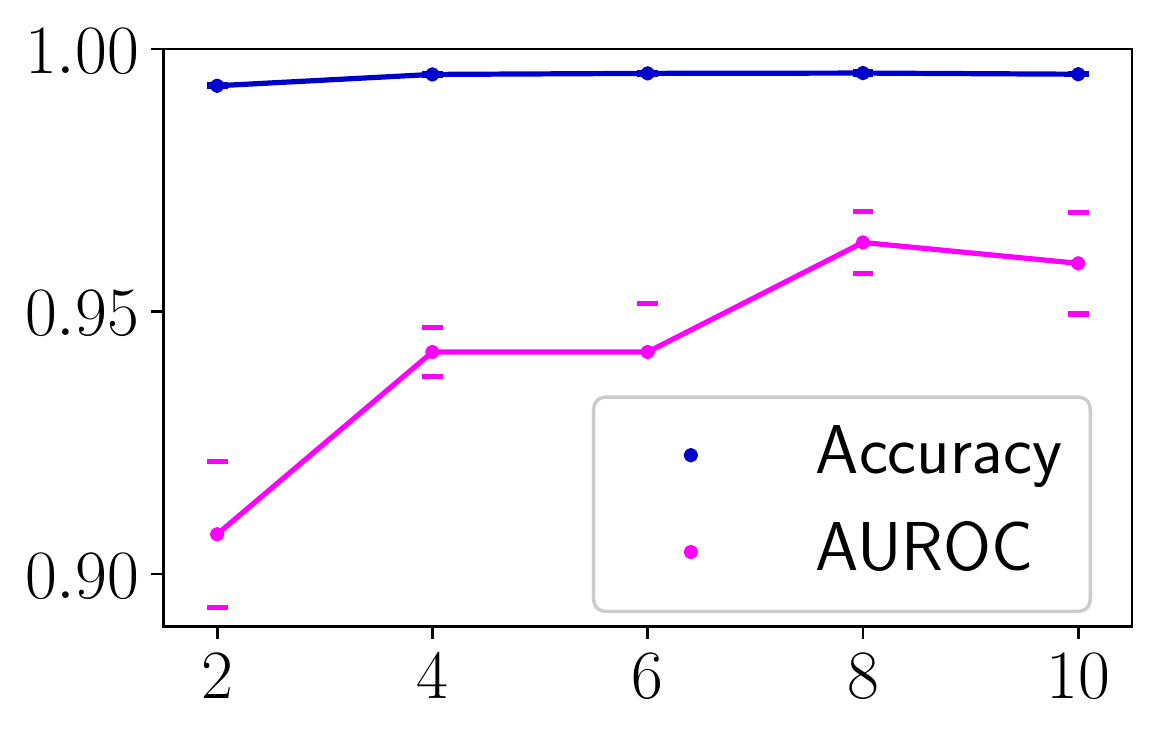}
\put(-80,78){\footnotesize $\mathcal{D}_\text{in}=\text{MNIST}$}
\includegraphics[width=0.3\columnwidth]{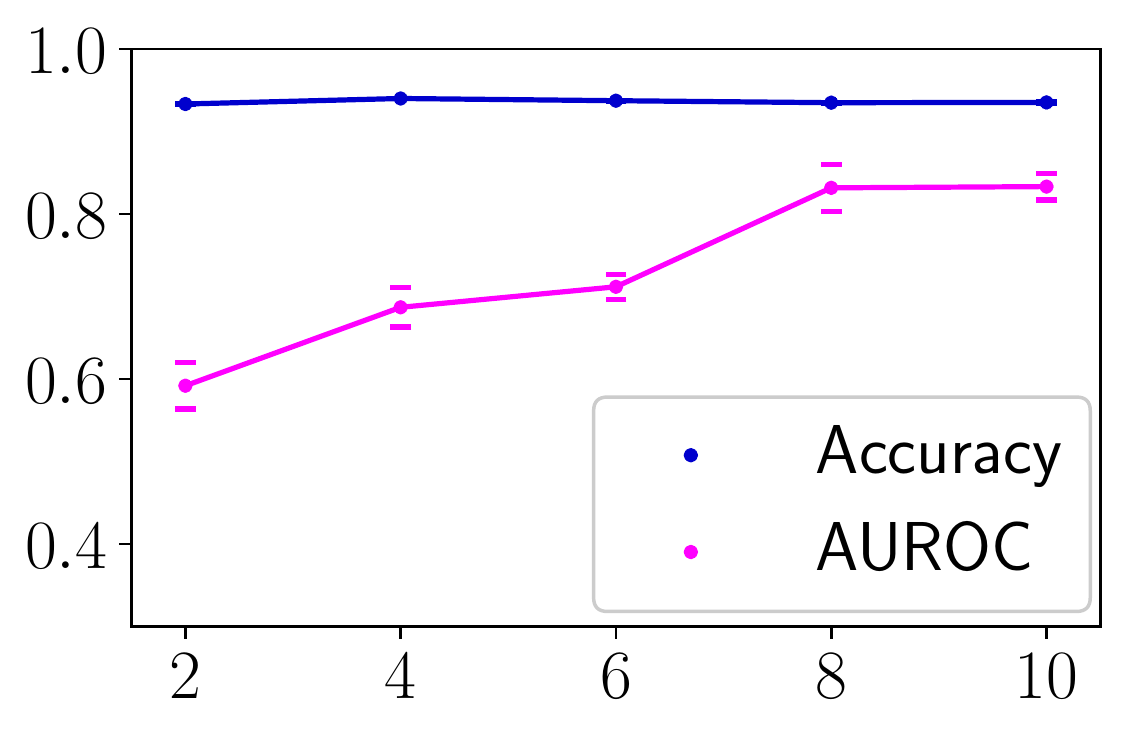}
\put(-90,78){\small $\mathcal{D}_\text{in}=\text{fashion MNIST}$}
\put(-160,-5){\small No. convolutional layers}
\caption{As depth increases, OOD detection improves. Mean$\pm$1 std. error over five runs.}
\label{fig_depth_experiment}
\end{center}
\end{figure}

\end{document}